\documentclass{article}

\usepackage{hyperref}

\usepackage[accepted]{icml2021}

\usepackage{microtype}
\usepackage{graphicx}
\usepackage{subfigure}
\usepackage{booktabs} 
\usepackage[utf8]{inputenc}
\usepackage{algorithm,algorithmic}
\usepackage{amsthm}
\usepackage{amsmath}
\usepackage{amssymb}
\usepackage{amsfonts}       
\usepackage{nicefrac}       
\usepackage{mathtools}
\usepackage{thmtools,thm-restate}
\usepackage[title]{appendix}
\usepackage{dsfont}
\usepackage{cleveref}
\usepackage{xcolor}

\def\eqdef{\stackrel{\text{def}}{=}}
\usepackage{bbm}

\def\Regret{\mathrm{Reg}}
\newcommand{\Ocal}{\mathcal{O}}
\newcommand{\Olog}{\tilde{\mathcal{O}}}

\usepackage{xspace}
\DeclareRobustCommand{\eg}{e.g.,\@\xspace}

\usepackage{mathtools}
\DeclarePairedDelimiter\br{(}{)}
\DeclarePairedDelimiter\brs{[}{]}
\DeclarePairedDelimiter\brc{\{}{\}}
\DeclarePairedDelimiter\abs{\lvert}{\rvert}
\DeclarePairedDelimiter\norm{\lVert}{\rVert}
\DeclarePairedDelimiter\inner{\langle}{\rangle}

\DeclarePairedDelimiter\floor{\lfloor}{\rfloor}
\DeclarePairedDelimiter\ceil{\lceil}{\rceil}

\newtheorem{theorem}{Theorem}
\newtheorem{lemma}{Lemma}

\newtheorem{remark}{Remark}

\newtheorem{example}{Example}

\newenvironment{proofsketch}{%
  \proof}{\endproof}
\newtheorem*{remark*}{Remark}
\newtheorem{theorem-rst}[theorem]{Theorem}
\newtheorem{lemma-rst}[lemma]{Lemma}
\newtheorem{proposition-rst}[lemma]{Proposition}
\newtheorem{assumption-rst}[lemma]{Assumption}
\newtheorem{claim-rst}[lemma]{Claim}
\newtheorem{corollary-rst}[lemma]{Corollary}

\usepackage{accents}
\newcommand{\ubar}[1]{\underaccent{\bar}{#1}}

\newcommand\trace[1]{\text{trace}\br*{#1}}
\renewcommand{\det}[1]{\text{det}\br*{#1}}

\newcommand{\indicator}[1]{1{\brc*{#1}}}
\newcommand{\E}{\mathbb{E}}
\newcommand{\R}{\mathbb{R}}
\newcommand{\N}{\mathbb{N}}
\newcommand{\G}{\mathbb{G}}
\newcommand{\Pb}{\mathbb{P}}

\newcommand{\F}{\mathcal{F}}

\newcommand{\VAR}{\mathrm{Var}}

\newcommand{\unu}{{\underline{\nu}}} 
\newcommand{\Narms}{A}

\newcommand{\kl}{\mathrm{KL}}
\newcommand{\klBin}{\mathrm{kl}}

\newcommand{\Bq}{B^q}
\newcommand{\LR}{\mathcal{L}_R} 
\newcommand{\Scal}{\mathcal{S}}
\newcommand{\Acal}{\mathcal{A}}

\newcommand{\sign}{\text{sign}}

\icmltitlerunning{Confidence-Budget Matching for Sequential Budgeted Learning}

\begin{document}

\twocolumn[
\icmltitle{Confidence-Budget Matching for Sequential Budgeted Learning}



\icmlsetsymbol{equal}{*}

\begin{icmlauthorlist}
\icmlauthor{Yonathan Efroni$^*$}{msr}
\icmlauthor{Nadav Merlis$^*$}{tech}
\icmlauthor{Aadirupa Saha}{msr}
\icmlauthor{Shie Mannor}{tech,nvidia}
\end{icmlauthorlist}

\icmlaffiliation{tech}{Technion, Israel}
\icmlaffiliation{msr}{Microsoft Research, New York}
\icmlaffiliation{nvidia}{Nvidia Research, Israel}
\icmlcorrespondingauthor{Yonathan Efroni}{jonathan.efroni@gmail.com}
\icmlcorrespondingauthor{Nadav Merlis}{merlis.nadav@gmail.com}
\icmlkeywords{Machine Learning, ICML}

\vskip 0.3in
]



\printAffiliationsAndNotice{\icmlEqualContribution} 

\begin{abstract}
A core element in decision-making under uncertainty is the feedback on the quality of the performed actions. However, in many applications, such feedback is restricted. For example, in recommendation systems, repeatedly asking the user to provide feedback on the quality of recommendations will annoy them. In this work, we formalize decision-making problems with querying budget, where there is a (possibly time-dependent) hard limit on the number of reward queries allowed. Specifically, we consider multi-armed bandits, linear bandits, and reinforcement learning problems. We start by analyzing the performance of `greedy' algorithms that query a reward whenever they can. We show that in fully stochastic settings, doing so performs surprisingly well, but in the presence of any adversity, this might lead to linear regret. To overcome this issue, we propose the Confidence-Budget Matching (CBM) principle that queries rewards when the confidence intervals are wider than the inverse square root of the available budget. We analyze the performance of CBM based algorithms in different settings and show that they perform well in the presence of adversity in the contexts, initial states, and budgets.
\end{abstract}

\section{Introduction}
In the past few decades, there have been great advances in the field  of sequential decision making under uncertainty. From a practical perspective, recent algorithms achieve superhuman performance in problems that had been considered unsolvable \citep{mnih2015human,silver2017mastering}. From a theoretical perspective, algorithms with order-optimal performance were presented to various important settings \citep[][and others]{garivier2011kl,azar2017minimax}. 

To solve such problems, most works share the same abstract interaction model. At each round, an agent (i)~observes some information on the state of the environment, (ii)~decides how to act, based on previous interactions, and, (iii)~observes new feedback on the effect of its action. Finally, the environment changes its state based on the agent's action, and the cycle begins anew. Much effort had been devoted to study specific instances of this abstract model, e.g., multi-armed bandits (MABs)~\citep{auer2002finite,garivier2011kl,kaufmann2012thompson,agrawal2012analysis}, linear bandits~\citep{dani2008stochastic,abbasi2011improved,agrawal2013thompson,abeille2017linear} and reinforcement learning (RL) settings~\citep{azar2017minimax,jin2018q,dann2019policy,zanette2019tighter,efroni2019tight,simchowitz2019non,tarbouriech2020no,cohen2020near,zhang2020reinforcement}. However, there are (still) several gaps between theory and practice that hinder the application of these models in real-world problems.

One such evident gap is the need to act under a budget constraint that limits the amount of feedback from the environment. That is, receiving feedback on the quality of the agent's actions has an inherent cost.  Consider, for example, an online recommendation system. There, asking for feedback from users negatively affects their experience, and feedback should be requested sparingly. Another example can be found in most large-scale RL domains, including autonomous driving. In many such cases, the reward should be labeled manually, and the resources for doing so are limited. Motivated by these problems, in this work, we aim to tackle the following question:

\begin{center}
    \emph{How should an agent trade-off exploration and exploitation  when the feedback is limited by a budget?}
\end{center}

In our efforts to answer this question, we study the effect of time-varying observation budget in various decision-making problems. Formally, we assume that at each round, the agent observes a non-decreasing, possibly adversarial, budget $B(t)$, which limits the number of queries for the reward of the problem. We first show that when the problem is stochastic and the budget is oblivious, greedily using any available budget leads to good performance. However, as soon as adversarial elements appear in the problem, or when the budget is controlled by an adaptive adversary, such an algorithm miserably fails. To tackle this problem, we suggest a simple, generic, scheme, that only samples rewards for actions with high uncertainty, in comparison to the budget. We call such a mechanism \emph{confidence-budget matching} (CBM). We show how to apply CBM to MAB, linear bandit and RL problems. In all cases, the mechanism can be applied in the presence of adaptive adversarial budgets. For linear bandits and RL, we show that CBM can be applied even when the contexts and initial states are adversarial. Finally, we present lower bounds for MABs and linear bandits, which show that CBM leads to order-optimal regret bounds. 

\section{Preliminaries}
\label{section: prelimineries}
We start by defining a general model for sequential decision-making under uncertainty. Then, we will explain its realization in each individual model. In the most general model, at each round~$t$, the environment supplies the agent with a context $u_t$ that may either be stochastic or adversarially chosen. Then, the agent selects a policy $\pi_t\in\Pi(u_t)$ that can depend on $u_t$ and past observations. Finally, the environment generates two stochastic feedback variables, from fixed distributions conditioned on $u_t$ and $\pi_t$: feedback on the interaction with the environment $Z_t$ and reward feedback $R_t$. In RL, for example, $Z_t$ is the visited state-actions while $R_t$ is their respective rewards. We also assume that there exists a reward function $f$ such that the agent aims to maximize $f(R_t)$ throughout the interaction. Alternatively, algorithms aim to minimize its \emph{pseudo-regret} (or regret), which is defined as 
\begin{align*}
    \Regret(T) \!=\!\!\sum_{t=1}^T\br*{\max_{\pi\in\Pi(u_t)}\E\brs*{f(R_t)\vert u_t, \pi} \!-\! \E\brs*{f(R_t)\vert u_t, \pi_t}\!}.
\end{align*}
Note that the pseudo-regret is random, as the policy depend on random feedback from the environment and contexts might be stochastic. Thus, regret bounds for different algorithms hold either with expectation or with high probability.

To illustrate the generality of this model, we explain how it encompasses both MAB, linear bandit and RL problems:

\textbf{Contextual Multi-Armed Bandits (CMABs).} At the beginning of each round, a context $u_t\in\brc*{1,\dots,S}$ is chosen, either stochastically or adversarially. Then, the agent chooses an action (arm) from a finite set of cardinality $A$, $\pi_t\triangleq a_t\in\Acal$ and the environment generates a reward $R_t\in\brs*{0,1}$ with an expectation $\E\brs{R_t \vert u_t=u, a_t=a}=r(u,a)$. An optimal arm is denoted by $a^*(u)\in\arg\max_a r(u,a)$ and its value by $r^*(u)=\max_a r(u,a)$. The reward function is $f(R_t)=R_t$ and there is no additional feedback ($Z_t=\phi$). A specific case of interest is where a single context exists, which is the well known MAB problem. Then, we denote $r(a)\triangleq r(1,a)$.

\textbf{Linear Contextual Bandits.} In the stochastic setting, $u_t$ contains a set of $\Narms$ vectors in $\R^d$, generated independently from a fixed distrubution. In the adversarial case, $u_t$ is an arbitrary set of vectors in $\R^d$. At each round $t$, the agent selects a single vector $\pi_t\triangleq x_t\in u_t$. Then, the environment generates a reward $R_t=\inner{x_t,\theta}+\eta_t$, where $\eta_t$ is zero-mean subgaussian noise and $\theta\in\R^d$ is unknown. As in the CMAB problem, the reward function is $f(R_t)=R_t$ and there are no additional observations ($Z_t=\phi$).

\textbf{Episodic Reinforcement Learning.} Let $\Scal,\Acal$ be finite state and action sets with cardinalities of $S,A$, respectively. Before each episode $t$, an initial state $s_{t,1}$ is generated either stochastically or adversarially (and serves as a context $u_t$). Then, an agent selects a nonstationary policy $\pi_t:\Scal\times \brs*{H}\to \Acal$, for some $H\in\N$. The policy is evaluated for $H$ steps, and states are generated according to a transition kernel $P$; namely, for any $s'\in \Scal$ and $h\in\brc*{1,\dots,H}$, $\Pr\br*{s_{t,h+1}=s'\vert s_{t,h},\pi_{t,h}}=P_h\br*{s'\vert s_{t,h},\pi_{t,h}(s_{t,h})}$. For brevity, we denote $a_{t,h}=\pi_{t,h}(s_{t,h})$. The agent observes the trajectory $Z_t=\brc*{(s_{t,h},a_{t,h})}_{h=1}^H$ and for each visited state, a reward $R_t = \brc*{R_{t,h}}_{h=1}^H\in \brs*{0,1}^H$ is generated such that $\E\brs*{R_{t,h}\vert s_{t,h}=s,a_{t,h}=a}=r(s,a)$. The reward function is then $f(R_t)=\sum_{h=1}^{H-1}R_{t,h}$.

\paragraph{Sequential Budgeted Learning.} In most cases, it is natural to observe the effect of the policy on the environment; for example, it is reasonable to assume that the agent observes the visited states in RL, as it acts according to them. Thus, we assume that the agent always observes $Z_t$. On the other hand, many applications require specifically querying or labeling the reward. Then, oftentimes, such feedback is limited. Formally, let $\brc{B(t)}_{t\ge1}$ be a non-negative budget sequence that might be adversarially chosen. We also assume that the budget is non-decreasing, that is, a budget that is given cannot be taken. At each round $t$, the agent observes $B(t)$ and selects whether to query $R_t$ or not, which we denote by $q_t=1$ and $q_t=0$, respectively. However, the agent can choose $q_t=1$ only if its budget was not exhausted. Throughout most of the paper, we assume that querying a reward incurs unit cost. Then, an agent can select $q_t=1$ only if $n^q_{t-1}\triangleq\sum_{k=1}^{t-1}\indicator{q_k=1}\le B(t)-1.$
In some cases, we extend the cost to be action-dependent. Then, a reward can only be queried if $${\Bq(t-1)\triangleq\sum_{k=1}^{t-1}c({\pi_k})\indicator{q_k=1}\le B(t)-c(\pi_t).}$$
Notice that when queries have unit costs, then $n^q_t=\Bq(t)$. For the RL setting, we give access to more refined feedback from specific time steps, to avoid confusion we only discuss it in \Cref{section: RL}. In all cases, we allow $q_t$ to also depend on $Z_t$. Finally, and for ease of notations, we assume that the agent always observes $Y_t=R_t\cdot q_t$.

\paragraph{General Notations}
We let $\brc{F_t}_{t\geq 0}$ be a filtration, where $F_t$ is the $\sigma$-algebra that contains the random variables $\brc{(u_k,\pi_k, Z_k,q_k, Y_k, B(k))_{k=0}^t, B(t+1), u_{t+1}}$. In words, it contains the information on all \emph{observed} rewards, actions, budget until the $t^{th}$ episode, the budget at the $(t+1)^{th}$ episode, and the context at the $(t+1)^{th}$ episode. We denote $\brs*{n}=\brc*{1,\dots,n}$ for $n\in\N$ and also $x\vee1=\max\brc{x,1}$ for any $x\in\R$. We use $\mathcal{O}(X)$ and $\Olog(X)$ to refer to a quantity that depends on $X$ up to constants and poly-log and constant expressions in problem parameters, respectively. Lastly, $\lesssim,\gtrsim$ denote inequalities that hold up to poly-log and constant expressions in problem parameters.

\section{Lower Bounds for Budgeted Problems}
Before suggesting algorithms to the budgeted setting, it is of importance to understand how the new constraint affects the best-achievable regret. 
To this end, we study problem-independent lower bounds for budgeted MAB. By the end of the section, we also shortly discuss lower bounds for budgeted linear bandits. To derive the lower bounds, we require a more detailed description of the MAB model and additional notations. Moreover, we need to adapt the fundamental inequality of \citet{garivier2019explore} to the case where the agent does not query all samples (\Cref{lemma: kl counts inequality}). We refer the reader to \Cref{appendix: lower bounds model notations} for more details on the model and to \Cref{appendix: lower bound basic inequality} for \Cref{lemma: kl counts inequality}. Other proofs for this section can be found at \Cref{appendix: lower bound proofs}. Using \Cref{lemma: kl counts inequality}, we can prove a lower bound for the following scenario in which $(i)$ sampling an arm requires a unit cost, $(ii)$ the budget constraint holds in expectation, and, $(iii)$ the budget is given to the learner at the initial interaction, i.e. $\forall t\in [T],\ B(t)=B$:
\begin{restatable}{proposition-rst}{mabLowerBoundEqual}
\label{proposition: lower bound mab unit costs}
Let $T$ be some time horizon and let $\pi$ be some bandit strategy such that for any bandit instance, it holds that $\E\brs*{n^q_T}\le B$. Then, for $\Narms\ge2$, there exists a bandit instance for which
\begin{align*}
    \E\brs*{\Regret(T)} \ge \frac{1}{140}\min\brc*{T\sqrt\frac{\Narms}{B},T}\enspace.
\end{align*}
\end{restatable}
As expected, when the budget is linear ($B=T$), we get the standard $\Omega\br*{\sqrt{\Narms T}}$ lower bound. However, as we decrease the budget, the lower bound increases, up to the point of linear regret when the budget is not time-dependent. We also remark that the lower bound holds even if the budget constraint is only met \emph{in expectation}. We will later present algorithms whose regret match these bounds, without \emph{ever} violating the budget constraint. This implies that relaxing the budget requirement to hold in expectation cannot improve performance, from a worst-case perspective. Finally, note that when the budget is polynomial in $T$, .e.g., $B=T^\beta$, we get a lower bound of $\Omega\br*{\sqrt{\Narms }T^{1-\beta}}$. We will later prove upper bounds that match this rate. 

Next, it is of interest to generalize the bound to the case of arm-dependent costs. In this case, we require a more subtle analysis that also costs in a $\sqrt{\log\Narms}$ factor:
\begin{restatable}{proposition-rst}{mabLowerBoundVaryingArmCosts}
\label{proposition: lower bound mab varying arm costs}
Let $T$ be the time horizon, and let $c(1),\dots,c(\Narms)\ge0$ be arm-dependent querying costs. Also, let $\pi$ be some bandit strategy such that for any bandit instance, it holds that $\E\brs*{\Bq(T)}\le B$. Then, for $\Narms\ge2$, there exists a bandit instance for which
\begin{align*}
    \E\brs*{\Regret(T)} \ge \frac{1}{140}\min\brc*{T\sqrt\frac{\sum_{a=1}^\Narms c(a)}{B(1+\log \Narms)},T}\enspace.
\end{align*}
\end{restatable}
While both bounds deal with fixed budget, $B(t)=B$ for all rounds, one can easily reduce them to lower bounds for time-dependent budgets, by reducing the lower bound only at a logarithmic factor. This is done by lower bounding the regret by the bound of the `worst-case' time horizon $\Omega\br*{\max_{t\in\brs*{T}}\brc*{\frac{t}{\sqrt{B(t)}}}}$. We demonstrate how to do so for the case of arm-dependent costs in the following corollary:
\begin{restatable}{corollary-rst}{mabLowerBoundVaryingBudget}
\label{corollary: lower bound mab varying budget}
Let $c(1),\dots,c(\Narms)\ge0$ be arm-dependent querying costs and let $B(1),\dots, B(T)>0$ be an arbitrary non-decreasing budget sequence. Also, let $\pi$ be some bandit strategy such that for any bandit instance and any time index $t\in\brs*{T}$, it holds that $\E\brs*{\Bq(t)}\le B(t)$. Then, for any $\Narms\ge2$, there exists a bandit instance for which
\begin{align*}
    &\E\brs*{\Regret(T)} \\
    &\ge \frac{1}{140(1+\log T)}\sum_{t=1}^T\min\brc*{\sqrt\frac{\sum_{a=1}^\Narms c(a)}{B(t)(1+\log \Narms)},1}\enspace.
\end{align*}
\end{restatable}
\begin{proof}
By \Cref{proposition: lower bound mab varying arm costs}, for any $t\in\brs*{T}$, there exists an instance such that 
\begin{align*}
    \E\brs*{\Regret(t)} \ge \frac{t}{140}\min\brc*{\sqrt\frac{\sum_{a=1}^\Narms c(a)}{B(t)(1+\log \Narms)},1}\enspace.
\end{align*}
Let $t_m$ be the time index in which the r.h.s. is maximized and fix the bandit problem to the corresponding instance that leads to its lower bound. Then.
\begin{align*}
    \E\brs*{\Regret(T)} 
    &\!\ge\! \E\brs*{\Regret(t_m)} \\
    &\!=\!\max_{t\in\brs*{T}}\brc*{ \frac{t}{140}\min\brc*{\!\sqrt\frac{\sum_{a=1}^\Narms c(a)}{B(t)(1+\log \Narms)},1}}.
\end{align*}
Finally, by H\"older's inequality, if $x,y\in\R^T$ are such that $x_t,y_t\ge0$ for all $t\in\brs*{T}$, then
\begin{align*}
    \max_t x_t =\norm*{x}_\infty\ge \frac{\sum_{t=1}^T x_ty_t}{\norm*{y}_1} = \frac{\sum_{t=1}^T x_ty_t}{\sum_{t=1}^T y_t} \enspace.
\end{align*}
Taking $x_t \!=\! \frac{t}{140}\min\brc*{\sqrt\frac{\sum_{a=1}^\Narms c(a)}{B(t)(1+\log \Narms)},1}$ and $y_t \!=\! \frac{1}{t}$ and recalling that $\sum_{t=1}^T \frac{1}{t}
\le1+\log T$ concludes the proof.
\end{proof}

In the following sections, we derive regret upper bounds of similar budget-dependence, e.g., $\Olog\br*{\sum_{t=1}^T \sqrt{A/B(t)} }$, if $c(a)=1$ for all $a\in \Acal$ . It is therefore of interest to observe the behavior of such bounds as a function of different budget profiles.
\begin{example}\emph{(Budget Profiles and Consequences).}
\begin{itemize} 
    \item \textbf{Linear Budget:} if $B(t)=\epsilon t$ for some $\epsilon>0$, then $\Regret(T) \le 2\sqrt{AT/\epsilon}$. Specifically, if $\epsilon=\Omega(1)$, then we get the standard rates of $\Regret(T) = \Theta(\sqrt{AT})$.
    \item \textbf{Polynomial Budget:} if $B(t)=t^{c}$ for some $c\in(0,1]$, then the regret is also polynomial, i.e., $\Regret(T) =\Theta(\sqrt{A}T^{1-c/2})$.
    \item \textbf{Fixed Budget:} if $B(t)=B_0>0$ is an initial budget, then $\Regret(T) =\Theta(\sqrt{A}T/\sqrt{B_0})$. However, if the budget is given at the end of the game, namely, $B(t)=0$ for $t\le T-B_0$ and $B(t)=B_0$ for $t> T-B_0$, then $\Regret(T) =\Omega(T)$ for any $B_0=o(T)$.
    \item \textbf{Periodically-replenished budget} if the budget is replenished by $B_0>0$ every $N\in\mathbb{N}$ steps, namely, $B(t)=B_0\cdot\br*{1 + \floor{\frac{t}{N}}}$, then
    \begin{align*}
        \Regret(T) = \Ocal\br*{\sum_{s=1}^{\ceil{T/N}}\frac{\sqrt{A}N}{\sqrt{B_0s}}} = \Ocal\br*{\sqrt{\frac{ATN}{B_0}}}\enspace.
    \end{align*}
\end{itemize}
\end{example}

\subsection{Lower Bounds for Linear Contextual Bandits}
We end this section by presenting a lower bound for linear bandits. Here, we assume that the budget constraint is never violated (as we assume in the upper bounds). Then, for fixed budget and context space, we derive the following bound:
\begin{restatable}{proposition-rst}{lbLowerBoundEqual}
\label{proposition: lower bound lb unit costs}
Let $T\in\N$ be some time horizon and let $\pi$ be a linear bandit policy such that $n^q_T\le B$ a.s. for some fixed $B\le T$. Then, there exists a $d$-dimensional linear bandit instance with arm set $\brs*{-1,1}^d$ for which the expected regret of $\pi$ is lower bounded by $\frac{dT}{80\sqrt B}$.
\end{restatable}
See~\Cref{appendix: lower bound linear bandits} for a proof. Importantly, this bound can be generalized to time-varying budgets, as in~\Cref{corollary: lower bound mab varying budget}.
\section{The Greedy Reduction:  Gap Between Adversarial and Stochastic Contexts}
\label{section:general reduction}
\begin{algorithm}[t]
\caption{Greedy Reduction} \label{alg: greedy reduction}
\begin{algorithmic}[1]
\STATE {\bf Require:} Algorithm $\mathbb{A}$, initial budget $B(1)\geq 1$ 
\STATE {\bf Initialize:} $l=0$

\FOR{$t=1,...,T$}
\STATE Observe context $u_t\sim \mathcal{P}_u$, and current budget $B(t)$
\IF{$B(t)\geq \Bq(t) + 1$} 
    \STATE \textcolor{gray}{// Query reward feedback, act with $\mathbb{A}$}
    \STATE  Advance $l\gets l+1$ and calculate $\pi_t\gets\mathbb{A}_l(u_t)$
    \STATE  Act with $\pi_t, q_t=1$; observe $Z_t$ and $R_t$
\ELSE
    \STATE \textcolor{gray}{//  Don't query feedback, act with `average' policy}
    \STATE Sample $j\sim Uniform(\brc*{1,\dots, l})$
    \STATE Act with $\pi_t\gets \mathbb{A}_j(u_t)$and $q_t=0$; ignore $Z_t$
\ENDIF
\ENDFOR
\end{algorithmic}
\end{algorithm}

We start by tackling the simpler case where the contexts are stochastic and the budget is oblivious. Formally, before the game starts, a sequence of budgets $\brc*{B(t)}_{t>1}$ is chosen, possibly adversarially. Later, at the beginning of each round $t$, a context $u_t$ is generated from a distribution $\mathcal{P}_u$, independently at random of other rounds. Then, the model continues as in \Cref{section: prelimineries}. For this section, we also assume that queries have unit costs.

For this model, we suggest a simple greedy reduction (see~\Cref{alg: greedy reduction}). Take an algorithm $\mathbb{A}$. If there is enough budget, query reward feedback and ask $\mathbb{A}$ for a policy $\pi_{t}$ to act with. Otherwise, when there is no available budget, pick uniformly at random a policy from past policies returned by $\mathbb{A}$, $\brc*{\pi_t}$, and act with it. We remark that $\mathbb{A}_k(u)$ denotes an output-policy of the algorithm at its $k^{th}$ iteration, with $u$ as the input context. Albeit simple, this algorithm performs surprisingly well, as we show in the following theorem:

\begin{restatable}[Black Box Reduction for Stochastic Contexts]{theorem-rst}{BlackBoxReductionTheorem}\label{theorem: black box reduction for stochastic context}
Let $\mathbb{A}$ be an anytime algorithm with bounded regret 
$
\E[\Regret(T)]\leq \alpha T^\beta + C
$
for some $\alpha,C\in \R_+,\beta\in\brs*{0,1}$ and any $T\in \mathbb{N}$. Moreover, assume the budget is chosen by an oblivious adversary such that it is non-decreasing,  $B(1)\ge1$ and $B(t)\in\N$ for all $t\ge1$. Then, the expected regret of~\Cref{alg: greedy reduction} with base algorithm $\mathbb{A}$ and budget sequence $\brc{B(t)}_{t\geq 1}$ is upper bounded by $\alpha T^\beta + C + \sum_{t=1}^T \frac{\alpha}{B^{1-\beta}(t)} + \frac{C}{B(t)}.$
\end{restatable}
The proof of the theorem (and all other results in the section) can be found at \Cref{appendix: gap between adversarial and stochastic}. One possible application of the theorem is in the MAB setting, combined with MOSS-anytime \citep{degenne2016anytime}. This would result in a regret bound of $\Ocal\br*{\sum_{t=1}^T \sqrt{\Narms/B(t)}}$, which matches the lower bound of \Cref{corollary: lower bound mab varying budget} up to log-factors. For linear bandits, using OFUL~\cite{abbasi2011improved} as the base algorithm implies a regret of $\Olog\br*{\sum_{t=1}^T d/\sqrt{B(t)}}$, which matches the lower bound of Proposition~\ref{proposition: lower bound lb unit costs} up to log-factors. In general, we believe that this reduction is tight in many interesting settings. One possible intuitive explanation for this can be found at the following proposition. In it, we prove that in non-contextual problems, for any fixed horizon, any general algorithm can be converted to one that uses the budget at the beginning of the game. A reasonable adaptation for time-varying budget and anytime algorithm would be to use the budget whenever possible.

\begin{restatable}{proposition-rst}{GreedyEquivalent}\label{proposition: greedy algorithm equivalent}
Assume that the decision-making problem is non-contextual ($u_t=\phi, \forall t$) with no environment feedback ($Z_t=\phi, \forall t$)  and unit-querying costs. Then, for any $T,B\in\N$ such that $B\le T$ and any policy $\pi$ under which $n^q_T\le B$, there exists a policy $\pi'$ such that $q_t=1$ for all $t\in\brs*{B}$ (and zero otherwise) and $\E\brs*{\Regret(T)\vert \pi'}=\E\brs*{\Regret(T)\vert \pi}$.
\end{restatable}

We end this section by returning to its basic assumptions - stochastic contexts and oblivious budget. We show that when at least one of these assumptions do not hold, the greedy reduction suffers a linear regret in a very simple CMAB problem, even if the budget is linear in expectation:

\begin{restatable}[Greedy Reduction Degrades in the Presence of Adversary]{proposition-rst}{CounterExamplesPropositions}\label{proposition: counterexample black box reduction}
If an adaptive adversary controls either (i)~the contexts, or (ii)~the budget, then for any base algorithm $\mathbb{A}$ used in Algorithm~\ref{alg: greedy reduction}, there exists a contextual MAB problem with two contexts and two arms such that $\E[\Regret(T)] \geq \frac{T}{4}$, even if $\E\brs*{B(t)} = \frac{t}{2}$ for all $t\in\brs*{T}$. 
\end{restatable}
\begin{proofsketch}
Consider a contextual multi-armed bandit instance with two contexts, $u\in \{1,2\}$. Assume that querying a reward feedback costs $1$ for all contexts and all arms. Furthermore, assume the budget increases in each episode by $1$ with probability $1/2$. 

If the adversary is adaptive to the history, it can choose $u=1$ every round the budget increases and otherwise  choose $u=2$. The greedy reduction then only queries for feedback for $u=1$. Thus, the regret for $u=2$ is linear in $T$, since no information is gathered for this context, and the number of rounds $u=2$ is $\Omega(T)$. Lastly, it can be shown that $\mathbb{E}[B(t)]=t/2$ in this construction. Equivalently, the same result holds if the contexts are uniformly distributed and an adaptive adversarial budget increases by a single unit only when $u=1$.
\end{proofsketch}
This emphasizes the need for developing non-greedy algorithms that store budget to face adversities in the problem.

\section{The Confidence-Budget Matching Principle}
\begin{algorithm}[t]
\caption{Confidence-Budget Matching (CBM) Scheme} \label{alg: CBM}
\begin{algorithmic}[1]
\STATE {\bf Require:} Optimistic algorithm~$\mathbb{A}$, $\brc*{\alpha_t}_{t\geq 1}$
\FOR{$t=1,...,T$}
\STATE Observe context $u_t$
\STATE Act with $\pi_t$, acquired from $\mathbb{A}(F_{t-1})$ and observe $Z_t$
\STATE Observe current budget $B(t)$
\IF{$CI_t(u_t,\pi_t) \geq \alpha_t \sqrt{1/B(t)}$} 
    \STATE Ask for feedback ($q_t=1$) and observe $R_t$
\ENDIF
\ENDFOR
\end{algorithmic}
\end{algorithm}
In the previous section, we showed that a simple greedy query rule performs well for sequential budgeted learning with stochastic contexts and oblivious budget. That is, querying for feedback as long as a spare budget exists results in a well-performing approach. However, this `greedy' approach can miserably fail in the presence of adversarial contexts or budget. In this section, we introduce an alternative approach we refer to as the Confidence-Budget Matching (CBM) principle. Unlike the greedy approach, CBM works well in the presence of adversities as it adequately preserves budget. 

CBM is a generic algorithmic scheme that converts an unbudgeted optimistic algorithm to an algorithm that can be utilized in sequential budgeted learning. As evident in~\Cref{alg: CBM}, the agent follows a policy calculated by the baseline algorithm $\mathbb{A}$. Then, feedback on the reward of $\pi_t$ is queried if the confidence interval (CI) of the policy, given current context, $CI_t(u_t,\pi_t)$ is larger than $\alpha\sqrt{1/B(t)}$ for some $\alpha>0$. As querying rewards decreases the CI, $CI_t(u_t,\pi_t)$ will gradually decrease. Then, if a policy is chosen frequently enough, reward querying will stop once its confidence matches $\alpha\sqrt{1/B(t)}$.

Unlike the greedy reduction, the performance of CBM does not degrade in the presence of adversarial contexts or budget, as we demonstrate later in this section. A crucial reason for this is that CBM stops querying rewards of policies with small CI. This somewhat conservative behavior leads to a more robust algorithm. To better understand the robustness of this querying rule, we consider the MAB problem. For this problem, we set $\alpha_t\sim\sqrt{\Narms}$, thus, for the MAB problem, CBM queries reward feedback if $CI_t(a_t)\geq \Olog(\sqrt{A/B(t)})$. Denoting the number of queries from action $a$ before the $t^{th}$ episode by $n_{t-1}^q(a)$ and setting $CI_t(a_t)\sim 1/\sqrt{n^q_{t-1}(a_t)}$ (Hoeffding-based CI) leads to the following equivalent condition to CBM query rule for MAB: \emph{ask for reward feedback if $n_{t-1}(a_t)\lesssim B(t)/\Narms.$} Namely, query for feedback if $a_t$ was queried less than $B(t)/\Narms$ times so far. Thus, this rule implicitly allocates $1/\Narms$ of the current budget to each of the arms for possible use. This immediately implies the budget constraint is never violated, since there are $\Narms$ arms in total.

\begin{remark}
Notice that the CBM scheme plays actions selected by the optimistic baseline algorithm $\mathbb{A}$, which do not depend on the current budget $B(t)$. In particular, all our results also hold even if the budget is revealed after the agent selects an action, as depicted in \Cref{alg: CBM}.
\end{remark}

Next, we study the performance of the CBM principle applied to MAB, linear bandits and RL problems. Importantly, we show that for all these settings, it matches the performance of the greedy reduction for stochastic environments, while being able to face adversarial contexts and budgets.
\begin{remark}[Sufficient Initial Budget] For simplicity, we assume the initial budget $B(1)$ is large enough such that Algorithm~\ref{alg: CBM} queries at the first round, that is $CI_1(u_1,\pi_1)\geq \alpha_1 \sqrt{1/B(1)}$. If this condition does not hold, an extra term of $T_I$ should be added to the regret bounds where $T_I$ is the first time in which $CI_{T_I}(u_{T_I},\pi_{T_I})\geq \alpha_{T_I} \sqrt{1/B(T_I)}$.
\end{remark}
\subsection{Multi-Armed Bandits}
We start by studying the performance of CBM for the MAB problem, where the base algorithm is UCB1 \citep{auer2002finite}. We call the resulting algorithm CBM-UCB, which follows \Cref{alg: CBM} with $\alpha_t=4\sqrt{6\sum_{a}c(a)\log(\Narms t)}$. Although this setting is extremely simple, it highlights the central analysis technique, which is extended in the rest of this section to more challenging decision-making problems.
\begin{restatable}[Confidence Budget Matching for Multi Armed Bandits]{theorem-rst}{CBMforBandits}\label{theorem: CBM Bandits}
For any querying costs $c(1),\dots,c(\Narms)\ge0$, any adaptive non-decreasing adversarially chosen sequence $\brc{B(t)}_{t\geq 1}$ and for any $T\ge1$, the expected regret of CBM-UCB is upper bounded by 
$
\Olog\br*{\! \sqrt{\Narms T} +\!\! \sqrt{\sum_{a}\!c(a)}\sum_{t=1}^T \E\brs*{\sqrt{\frac{1}{B(t)}}}}.
$
\end{restatable}

Full description of the algorithm, alongside the proof of \Cref{theorem: CBM Bandits}, is supplied at~\Cref{appendix: cbm for bandits}. We now present a proof sketch that highlights how the CBM principle affects the regret bounds.
\begin{proofsketch}
We use UCB bonus of $b_{t}^r(a)\!\triangleq\! \sqrt{\frac{3\log(\Narms t)}{2n^q_{t-1}(a)\vee 1}}$, where $n^q_t(a)$ is the number of times arm $a$ was queried up to round $t$; namely, if $\bar{r}_t(a)$ is the empirical mean of $a$ then,
$
UCB_t(a)= \bar{r}_{t-1}(a)+b_{t}^r(a),\ LCB_t(a) = \bar{r}_{t-1}(a)-b_{t}^r(a)
$
and $CI_t(a) = UCB_t(a)-LCB_t(a) = 2b_{t}^r(a)$. 

\textbf{Budget analysis.} We start the proof by establishing that the budget constraint is never violated, $\Bq(T)\leq B(T)$ for all $T\geq 1$. For simplicity, we do so for unit querying costs (where $\Bq(t)=n^q(t)$). By the CBM condition, if $q_t=1$, then $CI_t(a_t)\ge \alpha_t/\sqrt{B(t)}$. Then, for any $T\geq 1$
\begin{align*}
    n^q(T)
    &=\sum_{t=1}^T \indicator{q_t=1}
    \le \sum_{t=1}^T \frac{CI_t(a_t)}{\alpha_t/\sqrt{B(t)}}\indicator{q_t=1} \\
    & \lesssim \sqrt{B(T)}\sum_{t=1}^T \frac{1}{\sqrt{n^q_{t-1}(a)\vee 1}}\indicator{q_t=1},
\end{align*}
where in the last relation we substituted all parameters and used the fact that the budget is non-decreasing. Importantly, notice that when the reward of an arm is queried, its count increases, up to $n^q_t(a)$. Therefore, for any $T\geq 1$
\begin{align*}
    n^q(T)
     \!\lesssim\! \sqrt{B(T)}\sum_{a=1}^\Narms \sum_{i=0}^{n^q_T(a)}\!\frac{1}{\sqrt{i\vee1}}
   \!  \lesssim \sqrt{B(T)}\sqrt{n^q(T)}.
\end{align*}
Reorganizing and choosing the right constants leads to the relation $n^q(T)\le B(T)$, which deterministically holds. Importantly, this implies that CBM-UCB never tries to query reward without sufficient budget, so $q_t=1$ if and only if the CBM condition holds, or, equivalently, $q_t=0$ if and only if the CBM condition does not hold.

\textbf{Regret analysis}. Using standard concentration arguments, the expected regret $\E\brs*{\Regret(T)}$ is bounded by
\begin{align}
    &\sum_{t=1}^T \E\brs*{(UCB_t(a_t)-LCB_t(a_t))\indicator{q_t=1}} \label{eq: mab queried} \\
    & \quad+ \sum_{t=1}^T \E\brs*{(UCB_t(a_t)-LCB_t(a_t))\indicator{q_t=0}}.\label{eq: mab not queried}
\end{align}
For term \eqref{eq: mab queried}, reward is always queried; therefore, the analysis closely follows standard analysis for UCB, which results with a bound of $\Ocal\br*{ \sqrt{\Narms T\log(\Narms T)}}$. 
For \eqref{eq: mab not queried}, we know that reward was not queried, i.e, $q_t=0$. Since $q_t=0$ if and only if the CBM condition is not met, it implies that the CI is lower than the CBM-threshold, namely
\begin{align*}
    & \sum_{t=1}^T \E\brs*{(UCB_t(a_t)-LCB_t(a_t))\indicator{q_t=0}} \\
    &\lesssim \sum_{t=1}^T\E\brs*{\frac{\alpha_t}{\sqrt{B(t)}}} = \Olog\br*{\sum_{a}c(a)}\sum_{t=1}^T \E\brs*{\sqrt{\frac{1}{B(t)}}}.
\end{align*}
Combining both bounds leads to the desired regret bound. 
\end{proofsketch}

\subsection{Linear Bandits}
Next, we focus on applying the CBM principle, i.e., Algorithm~\ref{alg: CBM}, for linear bandits. The base algorithm that we rely on is OFUL~\citep{abbasi2011improved}, and we set $\alpha_t=\Olog(d)$ (see Appendix~\ref{appendix: cbm for linear bandits} for the full description of the algorithm). We call the resulting algorithm CBM-OFUL. Importantly, and in contrast to the greedy reduction of~\Cref{section:general reduction}, we allow both the contexts and the budget to be chosen by an adaptive adversary. Nonetheless, CBM-OFUL still achieve the same performance as the greedy reduction~\Cref{theorem: black box reduction for stochastic context}, while not suffering of performance degradation in the presence of adaptive adversary (for a complete  proof see \Cref{appendix: cbm for linear bandits}):

\begin{restatable}[Confidence Budget Matching for Linear Bandits]{theorem-rst}{CBMforLinearBandits}\label{theorem: CBM Linear Bandits}
For any adaptive adversarially chosen sequence of non-decreasing budget and context sets $\brc{B(t), u_t}_{t\geq 1}$ the regret of CBM-OFUL is upper bounded by 
$
\Olog\br*{d\br*{ \sqrt{T}+ \sum_{t=1}^T \frac{1}{\sqrt{B(t)}}}}
$
for any $T\geq 1$ with probability greater than $1-\delta$.
\end{restatable}
Notice that this matches the lower bound of \Cref{proposition: lower bound lb unit costs}. Notably, the examples of \Cref{proposition: counterexample black box reduction} can be represented as a linear bandit problem with $d=4$. Thus, in contrast to the greedy reduction, which suffers linear regret, the regret of CBM-OFUL is $\Olog\br*{\sqrt{T}}$.

\subsection{Reinforcement Learning}
\label{section: RL}

In this section we apply the CBM principle to RL. For this setting, we relax the budget model presented in \Cref{section: prelimineries} and allow agents to query specific state action pairs along the trajectory observed at the $t^{th}$ episode $\brc*{(s_{t,h},a_{t,h})}_{h\in[H]}$. Namely, at the $t^{th}$ episode, the agent acts with $\pi_t$, observes a trajectory $\brc*{(s_{t,h},a_{t,h})}_{h\in[H]}$ and is allowed to query for reward feedback from any state-action pair along the trajectory. If the agent queries reward feedback in the $t^{th}$ episode at the $h^{th}$ time step it receives $R_{t,h}(s_{t,h},a_{t,h})$. We denote this event as choosing $q_{t,h}=1$. For simplicity, we work with unit-budget costs, i.e., the total budget used by the agent is $\Bq(t)=\sum_{k=1}^t\sum_{h=1}^H \indicator{q_{k,h}=1}$ and must be smaller than $B(t)$. Observe that in the standard RL setting, the reward budget is $B(t)=Ht$ for all $t\geq 1$. 

Notably, querying reward feedback from specific time steps allows us to derive regret bounds that depend on the \emph{sparsity} of the reward function. Formally, let $\LR$ be the set of tuples $(s,a,h)$ with $r_h(s,a)\ne0$. Then, for any $(s,a,h)\notin\LR$, $r_h(s,a)=0$, and since $R_t\in\brs*{0,1}$, it also implies that $R_{t,h}\equiv0$. Assume that the algorithm knows the cardinality of this set $\abs*{\LR}$ (or an upper bound on $\abs*{\LR}$). Leveraging this knowledge, we set the CBM feedback query rule in Algorithm~\ref{alg: CBM}, line 6, as follows,
\begin{center}
   Ask for reward feedback on $(s_{t,h},a_{t,h})$ if $CI^R_{t,h}(s_{t,h},a_{t,h})\gtrsim\ \sqrt{\frac{|\LR|}{B(t)}} + \frac{SAH}{B(t)}$  ($q_{t,h}=1$),
\end{center}
where $CI^R_{t,h}(s_{t,h},a_{t,h})$ is the CI of the reward estimation of $s_{t,h},a_{t,h}$ in the $h^{th}$ time step at the $t^{th}$ episode. Setting the reward bonus of the `optimistic' model as in UCBVI-CH~\citep{azar2017minimax} leads to the following bound (see~\Cref{appendix: cbm-ucbvi RL} for more details on the algorithm and proofs).
\begin{restatable}[CBM-UCBVI]{theorem-rst}{CBMforRL}\label{theorem: CBM RL ULVI}
For any adaptive adversarially chosen sequence of non-decreasing budget and initial state, $\brc*{B(t),s_{t,1}}_{t\geq 1}$, the regret of CBM-UCBVI is upper bounded by 
$$
\Olog\br*{ \sqrt{SAH^4 T} + H^3S^2A + \sum_{t=1}^T\sqrt{\frac{|\LR|H^2}{B(t)}} + \frac{SAH^2}{B(t)} } 
$$
for any $T\geq 1$ with probability greater than $1-\delta$.
\end{restatable}
Notice that the last term of the regret is dominated by its first term when $B(t)=\Omega(\sqrt{T})$ and the remaining budget-dependent term only scales with the sparsity-level of the reward $\abs*{\LR}$. Notably, this implies that when ${B(t)\!\sim\! t\ceil{\abs*{\LR}/SAH}}$, the third term is of the same order as the first term. Differently put, if the query budget $B(t)$ increases by a single unit every $SAH/\abs*{\LR}$ episodes, the worst case performance of CBM-UCBVI remains the same, while reducing the amount of reward feedback.

While CBM-UCBVI clearly demonstrates the analysis techniques and insights from applying the CBM principle to RL, it is of interest to combine it with an algorithm with order-optimal regret bounds of $\sqrt{SAH^3T}$ (e.g.,~\citep{jin2018q}) when $B(t)=Ht$, that is, in the standard RL setting (notice that $T$ is the number of \emph{episodes} and not the total number of time steps). We achieve this goal by performing a more refined analysis that uses tighter concentration results based on~\citep{azar2017minimax,dann2019policy,zanette2019tighter}. Indeed, doing so leads to tighter regret bounds by a $\sqrt{H}$ factor in the leading term (Full details on the algorithm and proofs can be found at \Cref{appendix: cbm-ulcvi RL}).
\begin{restatable}[CBM-ULCVI]{theorem-rst}{CBMforRLUL}\label{theorem: CBM RL ULCVI}
For any adaptive adversarially chosen sequence of non-decreasing budget and initial state, $\brc*{B(t),s_{t,1}}_{t\geq 1}$, the regret of CBM-ULCBVI is upper bounded by 
$$
\Olog\br*{ \sqrt{SAH^3T}+ H^3S^2A+\sum_{t=1}^T\sqrt{\frac{|\LR|H^2}{B(t)}} + \frac{SAH^2}{B(t)} } 
$$
for any $T\geq 1$ with probability greater than $1-\delta$.
\end{restatable}
This bound results in an interesting conclusion for general RL problems, i.e., when $|\LR|=SAH$. Plugging this into Theorem~\ref{theorem: CBM RL ULCVI}, we observe that a budget of $B(t)=t$ -- instead a budget of $B(t)=Ht$ as used in standard RL -- results in order optimal regret bound. That is, it suffices for CBM-ULCVI to query reward feedback once per episode, without causing for performance degradation in a minimax sense.

\subsection{General View on CBM for Optimistic Algorithms}
The CBM principle queries for reward feedback (Algorithm~\ref{alg: CBM}, line 6) if the CI of the applied context-action is larger than a threshold, $CI_t(u_t,\pi_t) \geq \alpha_t\sqrt{1/B(t)}$, or more generally, if $CI_t(u_t,\pi_t) \geq \alpha_t f(B(t))$ for some $f:\mathbb{R}\rightarrow\mathbb{R}$. A natural question arises: how to choose $\alpha_t$ and $f$?

A useful rule of thumb to guide the choice of $\alpha_t$ and $f$ is the following: if the regret of the optimistic algorithm~$\mathbb{A}$ is bounded by $\Olog(\alpha T^\beta)$ then set $\alpha_t =\Olog(\alpha)$ and ${f(x)=\Olog(x^{\beta-1})}$. This matches the parameters chosen for both MAB and linear bandits. In RL, we relied on this rule but used a more complex function $f$, due to the application of an empirical Bernstein concentration argument~\cite{maurer2009empirical}.

The logic behind this choice is simple; it guarantees that the budget constraint is never violated, $\Bq(T)\leq B(T)$ for all $T\geq1$. Differently put, for any episode, reward feedback is not queried \emph{if and only if} $CI_t(u_t,\pi_t) \leq \alpha_t f(B(t))$. This property can be proved via similar technique as in the proof sketch of Theorem~\ref{theorem: CBM Bandits} for CBM-MAB. An informal proof for the correctness of this statement for the general case goes as follows (for unit feedback-costs),
\begin{align*}
    \Bq(T)&\leq \sum_{t=1}^T \indicator{q_t=1}CI_t(u_t,x_t)/ (\alpha B(t)^{\beta-1})\\
    &\overset{(a)}{\leq} \br*{B(T)^{1-\beta})/\alpha} \sum_{t=1}^T \indicator{q_t=1}CI_t(u_t,x_t)\\
    &\overset{(b)}{\lesssim}  \br*{B(T)^{1-\beta})/\alpha} \alpha \Bq(T)^{\beta},
\end{align*}
where$(a)$~holds since the budget is non-decreasing, and $(b)$~ since $\sum_{t=1}^T \indicator{q_t=1}CI_t(u_t,x_t)\sim \Regret(\Bq(T))$ for optimistic algorithms. Rearranging yields that $\Bq(T)^{1-\beta}\lesssim B(T)^{1-\beta}$ which implies that $\Bq(T)\leq B(T)$ by the monotonicity of $x^{1-\beta}$. Although the analysis for CBM in linear bandits and RL is more subtle, the intuition supplied by this informal reasoning is of importance; we believe it can serve as a starting point for future analysis of CBM-based algorithms in sequential budgeted learning.

\section{Related Work}

\textbf{Multi-Armed Bandits with Paid Observations~\citep{seldin2014prediction}.} Closely related to our work is the framework of MAB with paid observations. There, an agent plays with an arm $a_t$ and is allowed to query reward feedback on any subset of arms. Unlike in our case, there is no strict budget for observations, but, rather, each query comes at a cost that is subtracted from the reward. Notably, this requires translating the query costs to the same units as the reward, which is oftentimes infeasible. For example, in online recommendations, there is no clear way to quantify user dissatisfaction from feedback requests. In such cases, it is much more natural to enforce a (possibly time-varying) hard constraint on the number of feedback queries. 
Furthermore, the work of~\citet{seldin2014prediction} focus on the MAB problem, whereas in this work, we focus on more involved contextual problems (i.e., linear bandits and RL). It is important to note that the analysis in~\citep{seldin2014prediction} holds for the adversarial reward model, whereas in this work, we focused on the stochastic reward model (with adversarial contexts and budget). We believe it is an interesting question what type of guarantees can be derived for the fully adversarial setting, i.e., when the rewards, budget and contexts are adversarially chosen. Finally, when applied to the stochastic case, the algorithm of \citet{seldin2014prediction} requires $B(T)=\Omega(T^{2/3})$. In contrast, our results hold for lower budgets, while achieving similar bounds when $B(T)=\Omega(T^{2/3})$.

\textbf{MABs with Additional Observations \citep{yun2018multi}.} In this closely related MAB setting, observing the reward of arms that were not played is possible, at a certain cost, as long as a non-decreasing budget constraint is not violated. Nonetheless, a key difference from our work is that \citet{yun2018multi} assume that the reward of the played arm is \emph{always} observed and does not consume any budget. Therefore, there is no clear way to apply their results to our setting.

\textbf{Bandits with Knapsacks (BwK) \citep{badanidiyuru2013bandits}.} In the BwK model, a sampling budget is given prior to the game. At each round, the agent selects an arm and observes noisy samples of both the reward and the cost of the selected arm. That is, the agent always receives feedback on its actions. This comes in stark contrast to our model, where the budget restricts the amount of feedback an agent can obtain. Furthermore, in the BwK model, the game stops as soon as the cumulative cost exceeds the initial budget. 
In our model, where the budget serves as a constraint on the reward feedback, interaction continues even without an observation budget. When the budget is exhausted, the agent can still utilize its past information on the system to perform reasonably good actions. Notably, this forces the agent to sufficiently explore actions, even if they are costly, to identify high-rewarding ones.

We remark that there are additional extensions of the MAB setting in which arms incur costs~\citep[e.g.,][]{sinha2021multi}. There, the objective of an agent is to minimize a relaxed notion of cumulative regret and the cumulative cost. Unlike this work, we do not attribute cost to applying an action, but attribute a cost to \emph{receiving feedback} on the reward.

\textbf{RL with trajectory feedback \citep{efroni2020reinforcement}.} Under this model, instead of observing a reward for each played state-action, the agent only observes the cumulative rewards of each episode. This serves two reasons: first, and similarly to our work, it aims to reduce the feedback that the algorithm requires (by a factor of $H$), and when rewards are manually labeled, reduce the labeling load. Second, for many applications, it is much more natural to label the reward for a full trajectory than to each state-action. However, this approach comes at a noticeable cost, both in performance and computational complexity. In contrast, by sampling specific state-action pairs, our approach allows reducing the amount of feedback while maintaining similar performance and computational complexity. Nonetheless, we believe that when trajectory feedback is more natural, our approach can also be applied to further reduce the feedback for this setting. We leave such an extension for future work.

\section{Summary and Discussion}
In this work, we presented a novel framework for sequential decision-making under time-varying budget constraints. We analyzed what can and cannot be achieved by greedily using querying whenever possible. Then, we presented the CBM principle, which only queries rewards for actions with high uncertainty, compared to the current budget. We demonstrated how to apply the principle to MAB, linear bandits and RL problems and proved that it performs well also in the presence of adversities. We believe that this model can be adapted to many real-world problems and leaves room for interesting extensions, which we leave for future work.

\textbf{Is there a value in knowing the future budget?} Throughout this work, we assume the agent only observes the current budget $B(t)$ at the beginning of each round and does not have knowledge on future values of the budget $B(t')$ for $t'>t$. Intuitively, one expects that knowing the future budget would result in an improved and less conservative behavior in terms of budget allocation. Surprisingly, our matching lower and upper bounds for MAB (Corollary~\ref{corollary: lower bound mab varying budget} and Theorem~\ref{theorem: CBM Bandits}) and linear bandits (Proposition~\ref{proposition: lower bound lb unit costs} and Theorem~\ref{theorem: CBM Linear Bandits}) show that this intuition does not always hold. 
Nonetheless, understanding if or when information on future budget is of value remains an interesting open question.

\textbf{Monotonicity of the budget.} Throughout this work, we assume that the budget never decreases. Intuitively, it implies that once a budget is allocated, it does not matter when the algorithm decides to use it. Nonetheless, for some problems, different assumptions are sometimes more relevant. A budget might be given alongside an `expiration date' or might expire probabilistically. Another possible assumption is that the spare (unused) budget is bounded. Finally, in some instances, the total budget might be characterized by a specific random process, e.g., a biased random walk.

\textbf{Problem-dependent bounds.} Throughout this work, we focused on problem-independent regret bounds, that is, bounds that do not depend on the specific problem instance. Bounds that depend on specific instances usually focus on sufficiently sampling suboptimal arms, while implicitly assuming that optimal arms are sufficiently sampled~\citep{auer2002finite}. In contrast, when rewards are not always observed, algorithms must also control the number of queries from optimal arms. This becomes much harder in the presence of multiple optimal arms; in this case, an algorithm can never know if an arm is optimal or has a small suboptimal gap and might `waste' budget while trying to discern which is true. In some sense, we believe that the CBM principle is well-suited for this setting, as it prevents the agent from exhausting all budget on specific arms. 

\textbf{Adaptivity to structure.} In~\Cref{section: RL}, we proved that when rewards are sparse, our algorithm can query rewards according to the sparsity level, while maintaining the same regret bounds as the unbudgeted case. However, to do so, we required an upper bound on the sparsity of the problem. Therefore, a natural extension is to devise an algorithm that can adapt to an unknown sparsity level. Moreover, it is well known that structural assumptions can lead to improved regret bounds, and previous works proposed algorithms whose regret depends on nontrivial structural properties of the problem \citep{maillard2014hard, zanette2019tighter,foster2019model,foster2020instance,merlis2019batch,merlis2020tight}. Thus, it is interesting to understand what structural properties (beyond sparsity) affect the budgeted performance and how to design algorithms that adapt to such properties.

\section*{Acknowledgments}
This work was partially funded by the Israel Science Foundation under ISF grant number 2199/20. YE is partially supported by the Viterbi scholarship, Technion. NM is partially supported by the Gutwirth Scholarship.

\bibliography{citations}
\bibliographystyle{icml2021}

\clearpage
\onecolumn

\appendix
\addcontentsline{toc}{section}{Appendix}

\section{Lower bounds for Budgeted Problems}
\subsection{Detailed Decision-Making MAB Model and Relevant Notations}
\label{appendix: lower bounds model notations}
At each round $t$, the agent choose an arm $a_t$, which generates a reward $R_t$ from a distribution $\nu_{a_t}$, independently at random from other rounds. Throughout this section, we assume that all distributions $\nu_a$ are on $[0,1]$. A bandit problem is characterized by its arm distributions, which we denote by $\unu=\brc*{\nu_a}_{a=1}^\Narms$. When we want to emphasize the arm distribution when taking an expectation, we denote it $\E_\unu$. If the agent chooses to query this reward, we say that $q_t=1$ and, otherwise, $q_t=0$. Then, for ease of notations, we say that the agent observes $Y_t=R_t\cdot q_t$. A bandit strategy $\pi$ maps all previous information information and, possibly, internal randomization, into actions. Formally, let $U_0,U_1,\dots$ be independent and identically distributed random variables with uniform distribution over $[0,1]$. We denote the information known at time $t$ by $I_t=\br*{U_0,Y_1,U_1,\dots,Y_t,U_t}$, where $I_0=U_0$. Then, a strategy maps the current information to actions, i.e., $\pi_t^a(I_t)=a_{t+1}$ and $\pi_t^q(I_t)=q_{t+1}$.

Next, we denote $n_t(a)=\sum_{k=1}^t \indicator{a_k=a}$, the number of times that an arm $a$ was sampled up to time $t$ when playing according to strategy $\pi$, and similarly let $n^q_t(a)=\sum_{k=1}^t \indicator{a_k=a,q_k=1}$ be the total number of queries from arm $a$. Finally, let $\kl(\cdot,\cdot)$ be the Kullback-Leibler between two probability measure, and for any $x,y\in[0,1]$, we denote the KL divergence between Bernoulli random variables with expectations $x,y$ by 
\begin{align*}
    \klBin(x,y) = x\log\frac{x}{y}+(1-x)\log\frac{1-x}{1-y}\enspace.
\end{align*}

\clearpage
\subsection{Basic Inequalities}
\label{appendix: lower bound basic inequality}
\begin{lemma}
\label{lemma: kl counts inequality}
For any $T\ge1$, any $\sigma(I_T)$ measurable random variable $Z$ with values in $[0,1]$ and any two bandit problems $\unu$ and $\unu'$, it holds that
\begin{align}
    \label{eq: kl counts inequality}
    \sum_{a=1}^{\Narms}\E_\unu\brs*{n^q_T(a)}\kl(\nu_a,\nu_a') \ge \klBin\br*{\E_\unu\brs*{Z},\E_{\unu'}\brs*{Z}}
\end{align}
\end{lemma}
\begin{proof}
The proof closely follows the one of Inequality (6) in \citep{garivier2019explore} and only differs by ignoring rounds where $q_t=0$.
Formally, let $\nu_{a_t,q_t}$ be the distribution of $Y_t$ when playing $a_t$ and querying the reward according to $q_t$. Specifically, if $q_t=1$, then $\nu_{a_t,q_t}=\nu_{a_t}$, and if $q_t=0$, it deterministically outputs $Y_t=0$. Following the notations of \citep{garivier2019explore}, for any two bandit instances $\unu$ and $\unu'$, we let $\Pb_{\unu}$ and $\Pb_{\unu'}$ be their associated probability measures, defined in some common measurable space $(\Omega,\F)$ (which exists by Kolmogorov's extension theorem). Then, for any $t\ge0$ and any Borel sets $G\subset\R$ and $G'\subset[0,1]$, our model implies that
\begin{align*}
    \Pb_{\unu}\br*{Y_{t+1}\in G, U_{t+1}\in G' \vert I_t} = \nu_{\pi_t^a(I_t),\pi_t^q(I_t)}(B)\lambda(B')
\end{align*}
where $\lambda$ is the Lasbesgue measure on $[0,1]$. Next, for any $t\ge0$, let $\Pb_{\unu}^{I_t}$ and $\Pb_{\unu'}^{I_t}$ be the respective distributions of $I_t$ w.r.t. $\unu$ and $\unu'$, and similarly use this notation for $Y_t$ and $U_t$. Then, the previous relation can be written as
\begin{align*}
    \Pb_{\unu}^{(Y_{t+1}, U_{t+1})\vert I_t} = \nu_{\pi_t^a(I_t),\pi_t^q(I_t)} \otimes\lambda\enspace,
\end{align*}
where $\otimes$ denotes the product of measures. Moreover, by the chain-rule for KL divergences, for any $t\ge0$, we can write
\begin{align}
    \kl\br*{{\Pb_{\unu}^{I_{t+1}},\Pb_{\unu'}^{I_{t+1}}}}
    &= \kl\br*{{\Pb_{\unu}^{(I_t,Y_{t+1},U_{t+1})},\Pb_{\unu'}^{(I_t,Y_{t+1},U_{t+1})}}}\nonumber\\
    &= \kl\br*{{\Pb_{\unu}^{I_t},\Pb_{\unu'}^{I_t}}} + \kl\br*{{\Pb_{\unu}^{(Y_{t+1},U_{t+1})\vert I_t},\Pb_{\unu'}^{(Y_{t+1},U_{t+1})\vert I_t}}} \label{eq:kl chain rule}\enspace.
\end{align}
Notably, the second term can be simplified to 
\begin{align*}
    \kl\br*{{\Pb_{\unu}^{(Y_{t+1},U_{t+1})\vert I_t},\Pb_{\unu'}^{(Y_{t+1},U_{t+1})\vert I_t}}}
    &= \E_{\unu}\brs*{\E_\unu\brs*{\kl\br*{\nu_{\pi_t^a(I_t),\pi_t^q(I_t)} \otimes\lambda, \nu'_{\pi_t^a(I_t),\pi_t^q(I_t)} \otimes\lambda}\big\vert I_t}} \\
    &=\E_{\unu}\brs*{\E_\unu\brs*{\kl\br*{\nu_{\pi_t^a(I_t),\pi_t^q(I_t)}, \nu'_{\pi_t^a(I_t),\pi_t^q(I_t)}}\big\vert I_t}} \\
    & \overset{(*)}{=} \E_{\unu}\brs*{\E_\unu\brs*{\indicator{\pi_{t+1}^q(I_t)=1}\kl\br*{\nu_{\pi_t^a(I_t)}, \nu'_{\pi_t^a(I_t)}}\big\vert I_t}} \\
    & = \E_{\unu}\brs*{\indicator{\pi_{t+1}^q(I_t)=1}\sum_{a=1}^\Narms \kl\br*{\nu_a, \nu'_a} \indicator{\pi^a_{t+1}(I_t)=a}} \\
    & = \E_{\unu}\brs*{\sum_{a=1}^\Narms \kl\br*{\nu_a, \nu'_a} \indicator{a_{t+1}=a,q_{t+1}=1}} 
\end{align*}
where in $(*)$ we used the fact that if $q_{t+1}=\pi_{t+1}^q(I_t)=0$, then both distributions deterministically output $0$, so the KL divergence is also $0$. For the last relation, also recall that $\pi_t^a(I_t)=a_{t+1}$ and $\pi_{t+1}^q(I_t)=q_{t+1}$. Substituting back into \Cref{eq:kl chain rule} and applying the same argument recursively, we get
\begin{align*}
    \kl\br*{{\Pb_{\unu}^{I_{T}},\Pb_{\unu'}^{I_{T}}}}
    &= \sum_{t=1}^T\E_{\unu}\brs*{\sum_{a=1}^\Narms \kl\br*{\nu_a, \nu'_a} \indicator{a_t=a,q_t=1}} \\
    & = \sum_{a=1}^\Narms \E_{\unu}\brs*{n^q_T(a)}\kl\br*{\nu_a,\nu'_a}\enspace.
\end{align*}
Notice that the recursion stops at $\kl\br*{{\Pb_{\unu}^{I_{0}},\Pb_{\unu'}^{I_{0}}}}=\kl\br*{{\Pb_{\unu}^{U_0},\Pb_{\unu'}^{U_0}}} = \kl\br*{\lambda,\lambda}=0$. To conclude the proof, we apply Lemma 1 of \citep{garivier2019explore}, which implies that for any $\sigma(I_T)$ measurable random variable $Z$ over $[0,1]$, it holds that
\begin{align*}
    \kl\br*{\Pb_{\unu}^{I_T},\Pb_{\unu'}^{I_T}}
    \ge \klBin\br*{\E_\unu\brs*{Z},\E_{\unu'}\brs*{Z}}\enspace.
\end{align*}
\end{proof}

In the following, we present a simplified version for the bound of \Cref{lemma: kl counts inequality}, that will be of use in our proofs:
\begin{lemma} \label{lemma: kl counts inequality simplified}
Let $\unu$ be a bandit instance and let $k\in\brs*{\Narms}$ be some arm. Furthermore, let $\unu'$ be a bandit instance that differs from $\unu$ only at arm $k$, i.e., $\nu'_a=\nu_a$ for all $a\ne k$. Then,
\begin{align*}
    \E_\unu\brs*{n_T(k)} - T\sqrt{\frac{1}{2}\E_{\unu}\brs*{n^q_T(k)}\kl\br*{\nu_k,\nu_k'}}
    \le \E_{\unu'}\brs*{n_T(k)} 
    \le \E_\unu\brs*{n_T(k)} + T\sqrt{\frac{1}{2}\E_{\unu}\brs*{n^q_T(k)}\kl\br*{\nu_k,\nu_k'}}\enspace.
\end{align*}
Moreover, if $\nu_k$ and $\nu'_k$ are Bernoulli distributions with parameters $\frac{1}{2}$ and $\frac{1}{2}+\epsilon$, respectively, for $\epsilon\in\br*{0,\frac{1}{4}}$, then
\begin{align*}
    \E_\unu\brs*{n_T(k)} - 2\epsilon T\sqrt{\log\frac{4}{3}}\sqrt{\E_{\unu}\brs*{n^q_T(k)}}
    \le \E_{\unu'}\brs*{n_T(k)} 
    \le \E_\unu\brs*{n_T(k)} + 2\epsilon T\sqrt{\log\frac{4}{3}}\sqrt{\E_{\unu}\brs*{n^q_T(k)}}\enspace.
\end{align*}
\end{lemma}
\begin{proof}
Our proof closely follows the proof of Theorem 6 in \citep{garivier2019explore}, with small modifications due to the budget. We start by applying apply \Cref{lemma: kl counts inequality} on $\unu$ and $\unu'$ with $Z=\frac{n_T(k)}{T}$, while noticing that $\kl(\nu_a,\nu_a')=0$ for all $a\ne k$:
\begin{align*}
    \E_{\unu}\brs*{n^q_T(k)}\kl\br*{\nu_k,\nu_k'} \ge \klBin\br*{\frac{\E_\unu\brs*{n_T(k)}}{T},\frac{\E_{\unu'}\brs*{n_T(k)}}{T}}
    \ge 2\br*{\frac{\E_\unu\brs*{n_T(k)}}{T}-\frac{\E_{\unu'}\brs*{n_T(k)}}{T}}^2\enspace,
\end{align*}
where the last inequality is by Pinsker's inequality. Alternatively, we can write
\begin{align*}
    \abs*{\E_{\unu'}\brs*{n_T(k)} - \E_{\unu}\brs*{n_T(k)}}
    \le T\sqrt{\frac{1}{2}\E_{\unu}\brs*{n^q_T(k)}\kl\br*{\nu_k,\nu_k'}}\enspace,
\end{align*}
which leads to the first result of the lemma. For the second result, we directly upper bound $\kl\br*{\nu_k,\nu_k'}$ for any $\epsilon\in\br*{0,\frac{1}{4}}$ by:
\begin{align*}
    \kl\br*{\nu_k,\nu_k'} = \klBin\br*{\frac{1}{2},\frac{1}{2}+\epsilon} = \frac{1}{2}\log\frac{1}{1-4\epsilon^2} \le 8\epsilon^2\log\frac{4}{3}\enspace,
\end{align*}
where the last inequality holds since $\log\frac{1}{1-u}\le 4u\log\frac{4}{3}$ for any $u\in\br*{0,\frac{1}{4}}$. Substituting back to the first result of the lemma concludes the proof.
\end{proof}

\subsection{Proofs for lower bounds}
\label{appendix: lower bound proofs}
\mabLowerBoundEqual*
\begin{proof}
Let $\unu$ be an arm distribution such that all arms are Bernoulli-distributed with parameter $\frac{1}{2}$. By the pigeonhole principle, one can prove that for any $\Narms\ge2$, there must exist some arm $t$ such that both $\E_\unu\brs{n_T(k)}\le \frac{5T}{4\Narms}$ and $\E_\unu\brs{n^q_T(k)}\le \frac{15B}{\Narms}$ (see prove in \Cref{lemma: multiple conditions} for $\alpha = \frac{5}{4}$ and $\beta = 15$). 

Then, we define a new bandit instance $\unu'$ such that $\nu'_a=\nu_a$ for all $a\ne k$ and $\nu'_k$ is Bernoulli-distributed with parameter $\frac{1}{2}+\epsilon$, for some $\epsilon\in\br*{0,\frac{1}{4}}$ that will be determined later. Specifically, it implies that the regret of $\pi$ on $\unu'$ is 
\begin{align}
    \label{eq: lower bound mab unit costs base}
    \E_{\unu'}\brs*{\Regret(T)} = \sum_{a\ne k} \epsilon \E_{\unu'}\brs*{n_T(a)} = \epsilon T\br*{1-\frac{\E_{\unu'}\brs*{n_T(k)}}{T}}\enspace.
\end{align}
Thus, to lower bound the regret, we need to upper bound $\E_{\unu'}\brs*{n_T(k)}$. By \Cref{lemma: kl counts inequality simplified}, we have 
\begin{align*}
    \E_{\unu'}\brs*{n_T(k)} \le \E_\unu\brs*{n_T(k)} + 2\epsilon T\sqrt{\log\frac{4}{3}}\sqrt{\E_{\unu}\brs*{n^q_T(k)}}\enspace,
\end{align*}
and recalling that $\E_\unu\brs*{n_T(k)}\le \frac{5T}{4\Narms}$ and $\E_\unu\brs*{n^q_T(k)}\le \frac{15B}{\Narms}$, we get
\begin{align*}
    \E_{\unu'}\brs*{n_T(k)} \le \frac{5T}{4\Narms} + 2\epsilon T\sqrt{15\br*{\log\frac{4}{3}}\frac{B}{\Narms}}\enspace.
\end{align*}
Finally, fix $\epsilon=\frac{1}{35}\min\brc*{\sqrt{\frac{\Narms}{B}},1}$ (and, specifically, if $B=0$, then $\epsilon=\frac{1}{35}$). Since $\Narms\ge2$, this value of $\epsilon$ yields
\begin{align*}
    \E_{\unu'}\brs*{n_T(k)} \le \frac{5T}{4\Narms} + \frac{1}{8}T \le \frac{3}{4}T\enspace.
\end{align*}
Substituting this bound and $\epsilon$ back to \Cref{eq: lower bound mab unit costs base} leads to the desired result and concludes the proof.
\end{proof}

\mabLowerBoundVaryingArmCosts*
\begin{proof}
Without loss of generality, assume that $c(1)\ge c(2)\ge\dots\ge c(\Narms)$. To prove the lemma, we will show that for any $i\in\brs*{\Narms}$, there exists a bandit instance $\unu'$ such that 
\begin{align}
    \label{eq: lower bound different arm costs one arm}
    \E_{\unu'}\brs*{\Regret(T)} \ge \frac{1}{140}\min\brc*{T\sqrt\frac{i c(i)}{B(1+\log\Narms)},T}\enspace.
\end{align}
In turn, this will imply that there exists an instance such that
\begin{align}
    \label{eq: lower bound different arm costs all arms}
    \E_{\unu'}\brs*{\Regret(T)} \ge \frac{1}{140}\min\brc*{T\sqrt\frac{\max_{i\in\brs*{\Narms}}\brc*{i c(i)}}{B(1+\log\Narms)},T}\enspace.
\end{align}
Finally, we apply H\"older's inequality to lower bound the maximum; for vectors $x,y$ such that $x_i=ic(i)$ and $y_i =\frac{1}{i}$, it holds that 
\begin{align*}
    \max_{i\in\brs*{\Narms}}\brc*{i c(i)}
    = \norm*{x}_\infty 
    \ge \frac{\sum_{i=1}^\Narms x_iy_i}{\norm*{y}_1}
    = \frac{\sum_{i=1}^\Narms x_iy_i}{\norm*{y}_1}
    = \frac{\sum_{i=1}^\Narms c(i)}{\sum_{i=1}^\Narms \frac{1}{i}}
    \ge \frac{\sum_{i=1}^\Narms c(i)}{1+\log \Narms}\enspace.
\end{align*}
Substituting this bound to \Cref{eq: lower bound different arm costs all arms} will then conclude the proof.

We start by proving \Cref{eq: lower bound different arm costs one arm} when $i\ge2$, a proof that greatly resembles the one of \Cref{proposition: lower bound mab unit costs}. Fix some $2\le i \le \Narms$ and assume that $c(i)>0$, as otherwise, the bound trivially holds. Also, let $\unu$ be an arm distribution such that all arms are Bernoulli-distributed with parameter $\frac{1}{2}$. Notice that it holds that $\sum_{a=1}^i \E_\unu\brs{n_T(a)} \le T$, and since $c(j)\ge c(i)$ for all $j\le i$, it also holds that $\sum_{a=1}^i \E_\unu\brs{n^q_T(a)} \le \frac{B}{c(i)}$, or otherwise, the budget constraint is violated. Then, by \Cref{lemma: multiple conditions}, there must exist some arm $k\in\brs*{i}$ such that both $\E_\unu\brs{n_T(k)}\le \frac{5T}{4i}$ and $\E_\unu\brs{n^q_T(k)}\le \frac{15B}{ic(i)}$. 

Next, we define a new bandit instance $\unu'$ such that $\nu'_a=\nu_a$ for all $a\ne k$ and $\nu'_k$ is Bernoulli-distributed with parameter $\frac{1}{2}+\epsilon$, for some $\epsilon\in\br*{0,\frac{1}{4}}$. 
Then, by \Cref{lemma: kl counts inequality simplified}, we have 
\begin{align*}
    \E_{\unu'}\brs*{n_T(k)} \le \E_\unu\brs*{n_T(k)} + 2\epsilon T\sqrt{\log\frac{4}{3}}\sqrt{\E_{\unu}\brs*{n^q_T(k)}}\enspace,
\end{align*}
and recalling that $\E_\unu\brs*{n_T(k)}\le \frac{5T}{4i}$ and $\E_\unu\brs*{n^q_T(k)}\le \frac{15B}{i c(i)}$, we get
\begin{align*}
    \E_{\unu'}\brs*{n_T(k)} \le \frac{5T}{4i} + 2\epsilon T\sqrt{15\br*{\log\frac{4}{3}}\frac{B}{i c(i)}}\enspace.
\end{align*}
Finally, fix $\epsilon=\frac{1}{35}\min\brc*{\sqrt{\frac{i c(i)}{B}},1}$ (and, specifically, if $B=0$, then $\epsilon=\frac{1}{35}$). Since $i\ge2$, this value of $\epsilon$ yields
\begin{align*}
    \E_{\unu'}\brs*{n_T(k)} \le \frac{5T}{4\Narms} + \frac{1}{8}T \le \frac{3}{4}T\enspace.
\end{align*}
Finally, as in the proof of \Cref{proposition: lower bound mab unit costs}, we lower bound the regret by
\begin{align*}
    \E_{\unu'}\brs*{\Regret(T)} = \sum_{a\ne k} \epsilon \E_{\unu'}\brs*{n_T(a)} = \epsilon T\br*{1-\frac{\E_{\unu'}\brs*{n_T(k)}}{T}}
    \ge \frac{1}{140}\min\brc*{T\sqrt\frac{i c(i)}{B},T}\enspace,
\end{align*}
which leads to \Cref{eq: lower bound different arm costs one arm} for any $2\le i\le \Narms$.

Finally, we prove that \Cref{eq: lower bound different arm costs one arm} holds when $i=1$. If $c(1)=0$, the result trivially holds. Otherwise, by our assumptions, we have that $\E_\unu\brs*{n^q_T(k)}\le \frac{B}{c(i)}$. For $\E_\unu\brs*{n^q_T(1)}$, we divide the proof into to cases:
\begin{itemize}
    \item If $\E_\unu\brs*{n_T(1)}\le \frac{T}{2}$, we set $\unu'$ such that $\nu'_a=\nu_a$ for all $a>1$ and $\nu'_1$ is Bernoulli-distributed with parameter $\frac{1}{2}+\epsilon$. Then, by \Cref{lemma: kl counts inequality simplified}, we have 
    \begin{align*}
        \E_{\unu'}\brs*{n_T(k)} \le \E_\unu\brs*{n_T(1)} + 2\epsilon T\sqrt{\log\frac{4}{3}}\sqrt{\E_{\unu}\brs*{n^q_T(1)}}
        \le \frac{T}{2} + 2\epsilon T\sqrt{\log\frac{4}{3}}\sqrt{\frac{B}{c(1)}}\enspace,
    \end{align*}
    and fixing $\epsilon=\frac{1}{5}\min\brc*{\sqrt{\frac{c(1)}{B}},1}$ (or $\epsilon=\frac{1}{5}$ when $B=0$) leads to 
    \begin{align*}
        \E_{\unu'}\brs*{n_T(k)} \le \frac{3T}{4}\enspace.
    \end{align*}
    Then, the regret for this instance is lower bounded by 
    \begin{align*}
        \E_{\unu'}\brs*{\Regret(T)} = \sum_{a\ne 1} \epsilon \E_{\unu'}\brs*{n_T(a)} = \epsilon T\br*{1-\frac{\E_{\unu'}\brs*{n_T(1)}}{T}}
        \ge \frac{1}{20}\min\brc*{T\sqrt\frac{c(1)}{B},T}\enspace,
    \end{align*}
    
    \item If $\E_\unu\brs*{n_T(1)}\ge \frac{T}{2}$, we set $\unu'$ such that $\nu'_a=\nu_a$ for all $a>1$ and $\nu'_1$ is Bernoulli-distributed with parameter $\frac{1}{2}-\epsilon$. Then, by \Cref{lemma: kl counts inequality simplified}, we have 
    \begin{align*}
        \E_{\unu'}\brs*{n_T(k)} \ge \E_\unu\brs*{n_T(1)} - 2\epsilon T\sqrt{\log\frac{4}{3}}\sqrt{\E_{\unu}\brs*{n^q_T(1)}}
        \ge \frac{T}{2} - 2\epsilon T\sqrt{\log\frac{4}{3}}\sqrt{\frac{B}{c(1)}}\enspace,
    \end{align*}
    and fixing $\epsilon=\frac{1}{5}\min\brc*{\sqrt{\frac{c(1)}{B}},1}$  (or $\epsilon=\frac{1}{5}$ when $B=0$) leads to 
    \begin{align*}
        \E_{\unu'}\brs*{n_T(k)} \ge \frac{T}{4}\enspace.
    \end{align*}
    Then, the regret for this instance is lower bounded by 
    \begin{align*}
        \E_{\unu'}\brs*{\Regret(T)} = \epsilon \E_{\unu'}\brs*{n_T(1)}
        \ge \frac{1}{20}\min\brc*{T\sqrt\frac{c(1)}{B},T}\enspace,
    \end{align*}
\end{itemize}
Combining both cases leads to \Cref{eq: lower bound different arm costs one arm} (with a better constant) and concludes the proof.
\end{proof}

\clearpage

\begin{lemma}
\label{lemma: multiple conditions}
Let $x,y\in \R^n_+$ for some $n\ge2$, and assume that $\sum_{i=1}^n x_i \le X$ and $\sum_{i=1}^n y_i \le Y$. Then, for any $\alpha\in\br*{1,2}$ and $\beta\ge \frac{3\alpha}{\alpha-1}$, there exists an index $t$ such that $x_t\le \frac{\alpha X}{n}$ and $y_t\le \frac{\beta Y}{n}$ simultaneously.
\end{lemma}
\begin{proof}
We divide the proof into the case where $n\le \beta$ and $n\ge \beta$.

If $n\le \beta$, recall that $y_i\ge0$ and $\sum_{i=1}^n y_i\le Y$. Therefore, for any $i\in\brs*{n}$, $y_i\le Y \le \frac{\beta}{n} Y$, and the required condition holds for all coordinates. Moreover, as $\alpha>1$, by the pigeonhole principle, there exists at least one coordinate $t$ such that $x_t\le \frac{X}{n}\le \frac{\alpha X}{n}$. Then, for this coordinate, both required conditions hold.

Next, we analyze the case where $n>\beta$. First notice that by the pigeonhole principle, there are at least $\floor*{\br*{1-\frac{1}{\alpha}}n}$ coordinates such that $x_i\le \frac{\alpha X}{n}$. To see this, assume in contradiction that this does not hold. Then, there are at least $n-\floor*{\br*{1-\frac{1}{\alpha}}n}\ge \frac{n}{\alpha}$ coordinates such that  $x_i> \frac{\alpha X}{n}$, and 
\begin{align*}
    \sum_{i=1}^n x_i 
    \ge \sum_{i=1}^n x_i \cdot \indicator{x_i> \frac{\alpha X}{n}}
    > \frac{\alpha X}{n} \cdot \abs*{\brc*{i:x_i> \frac{\alpha X}{n} }}
    \ge \frac{\alpha X}{n} \cdot \frac{n}{\alpha}
    =X\enspace
\end{align*}
which violates the assumption that $\sum_{i=1}^d x_i \le X$. Similarly, there are at least $\floor*{\br*{1-\frac{1}{\beta}}n}$ coordinates such that $y_i\le \frac{\beta Y}{n}$. The number of coordinates for which at least one of the condition holds is then
\begin{align*}
    \sum_{i=1}^n \indicator{x_i \le \frac{\alpha X}{n}} + \indicator{y_i \le \frac{\beta Y}{n}}
    &\ge \floor*{\br*{1-\frac{1}{\alpha}}n} + \floor*{\br*{1-\frac{1}{\beta}}n} \\
    & \ge \br*{1-\frac{1}{\alpha}}n -1 + \br*{1-\frac{1}{\beta}}n - 1 \\
    & = \br*{1 - \frac{1}{\alpha} - \frac{1}{\beta}}n + n - 2 
\end{align*}
We then use the fact that $n>\beta$ to bound the first term as 
\begin{align*}
    \br*{1 - \frac{1}{\alpha} - \frac{1}{\beta}}n
    > \br*{1 - \frac{1}{\alpha} - \frac{1}{\beta}}\beta
    = \beta\br*{1-\frac{1}{\alpha}} - 1
    \ge 2\enspace,
\end{align*}
where the last inequality is since $\alpha>1$ and $\beta\ge \frac{3\alpha}{\alpha-1} = 3 \frac{1}{1-1/\alpha}$. Substituting back leads to 
\begin{align*}
    \sum_{i=1}^n \indicator{x_i \le \frac{\alpha X}{n}} + \indicator{y_i \le \frac{\beta Y}{n}}
    > 2 + n -2 
    = n\enspace.
\end{align*}
Thus, there are strictly more than $n$ coordinates where at least one condition holds, which implies that there is at least one coordinate $t$ for which both conditions holds, i.e., $x_t\le \frac{\alpha X}{n}$ and $y_t\le \frac{\alpha Y}{n}$
\end{proof}

\clearpage

\subsection{Lower Bounds for Linear Bandits}
\label{appendix: lower bound linear bandits}
\lbLowerBoundEqual*
\begin{proof}
Our proof closely follows the standard lower bound techniques used for proving the fundamental performance limit for linear bandits \citep[\eg][Theorem 24.1]{lattimore2020bandit}. 
Notably, we prove the lower bound for a problem with a fixed context space $\mathcal{X}=\brs*{-1,1}^d$. Therefore, the conditions of \Cref{proposition: greedy algorithm equivalent} hold; namely, for any fixed problem, there exists a modified policy $\pi'$ with the same expected regret as $\pi$ that queries all rewards at the first $B$ rounds. Thus, and without loss of generality, we assume that all rewards are queried at the $B$ initial rounds.

As described in \Cref{section: prelimineries}, we assume that the reward when choosing a context $x_t$ is $R_t=\inner{x_t,\theta}+\eta_t$ for some unknown $\theta$.In this section, we limit $\theta\in \Theta \triangleq \brc*{-\frac{1}{\sqrt B},\frac{1}{\sqrt B}}^d$. Moreover, we assume that the noises $\eta_t \sim \mathcal N(0,1)$ are i.i.d. standard Gaussian variables. Notably, for a given $\theta$, one can easily observe that the optimal context $x^*=\arg\max_{x\in\mathcal{X}}\inner{x,\theta}$ is such that $x^*(i)=\sign(\theta(i)),\forall i\in\brs*{d}$.  We also use the notation $Y_t = R_t\cdot q_t$ to denote the reward observed by the agent and write $R_t(\theta)$ when we want to emphasize the parameter that governed the reward generation. 

Suppose that $x_1,\dots,x_T$ are the contexts played by $\pi$ at rounds $1,\dots,T$ and for any fixed problem with parameter $\theta$, denote the measure of contexts and queried rewards induced by the interaction of $\pi$ with the problem by $\Pb_\theta$. 
Then, we lower bound the regret by
\begin{align}
     \E\brs*{\Regret(T)} 
     &= \E\brs*{\sum_{t=1}^T \inner{x^*,\theta}-\inner{x_t,\theta}} \nonumber\\
     &= \E\brs*{\sum_{t=1}^T \sum_{i=1}^d \abs*{\theta(i)}-x_t(i)\theta(i)}\nonumber\\
     & \ge \sum_{i=1}^d \abs*{\theta(i)}\E\brs*{\sum_{t=1}^T  \indicator{\sign(\theta(i)) \neq \sign(x_t(i))} \tag{$x_t(i)\in\brs*{-1,1}$}} \nonumber \\
    &\ge \sum_{i=1}^d\abs*{\theta(i)}\frac{T}{2}\E\brs*{\indicator{\sum_{t=1}^T  \indicator{\sign(\theta(i)) \neq \sign(x_t(i))}}\ge \frac{T}{2}} \nonumber\\
    &= \frac{T}{2\sqrt B} \sum_{i = 1}^{d}\Pb_{\theta}(i)~, \label{eq:lb1_m}
 \end{align}
 where we denoted $\Pb_{\theta}(i) := \Pb_{\theta}\br*{ \sum_{t = 1}^{T}\indicator{\sign(\theta(i)) \neq \sign(x_{t}(i)} ) \ge \frac{T}{2} }$. Next, we lower bound $\sum_{i = 1}^{d}\Pb_{\theta}(i)$.

For any choice of $\theta \in \Theta$ and $i \in [d]$, let $\theta^i$ be a vector such that $\theta^i(j) = \theta(j), \, \forall j \in [d]\setminus\{i\}$ and $\theta^i(i) = -\theta(i)$. 
Now let us define
$\Pb^c_{\theta^{i}}(i) := \Pb_{\theta^{i}}\bigg( \sum_{t = 1}^{T}(\sign(\theta^{i}_i) \neq \sign(x_{ti}) ) \le \frac{T}{2} \bigg)$. By applying \emph{Bretagnolle-Huber inequality} \citep[][Theorem 14.2]{bretagnolle1979estimation,lattimore2020bandit} we further get that for any $\theta \in \R^d$:
 \begin{align*}
 \Pb_{\theta}(i) + \Pb^c_{\theta^{i}}(i) \ge \frac{1}{2}\exp\brc*{-\kl(\Pb_{\theta},\Pb_{\theta^i})}~,
 \end{align*}
 where $\kl(\cdot,\cdot)$ denotes the Kullback–Leibler divergence. Applying the chain rule of KL-divergence while recalling that the policy for both bandit instances is the same, we get:
 \begin{align*}
  \Pb_{\theta}(i) + \Pb^c_{\theta^{i}}(i) & \ge \frac{1}{2}\exp\brc*{-\sum_{t = 1}^T \kl\br*{\Pb_{\theta}(Y_t(\theta) ),\Pb_{\theta^i}(Y_t(\theta^i)) \vert x_1,Y_1,\dots,x_{t-1},Y_{t-1},x_t}} \\
 & \overset{(a)}{=}  \frac{1}{2}\exp\brc*{-\sum_{t = 1}^B \kl\br*{\Pb_{\theta}(Y_t(\theta) ),\Pb_{\theta^i}(Y_t(\theta^i))\vert x_t}} \\
 & \overset{(b)}{=} \frac{1}{2}\exp\brc*{-\frac{1}{2}\sum_{t=1}^B\br*{\inner{x_t,\theta - \theta^i}}^2} \\
 &\overset{(c)}{\ge} \frac{1}{2}\exp(-2),
 \end{align*}
Relation $(a)$ follows since we assume the algorithm only queries for feedback for the first $B$ rounds, so $Y_t\ne0$ only generated for the first $B$ rounds, and the algorithm does not observe any feedback for the last $T-B$ rounds. Moreover, this relation also requires the assumption that the noise $\eta_t$ is i.i.d, and therefore, the reward only depends on the chosen context $x_t$.
 Furthermore, since we assume that $\eta_t \sim \mathcal N(0,1)$, $(b)$ follows by direct calculation:
 $$
 \kl\br*{\Pb_{\theta}(R_t(\theta)),\Pb_{\theta^i}(R_t(\theta^i))\vert x_t} = \kl\br*{\mathcal N(\inner{x_t,\theta},1),N(\inner{x_t,\theta^i},1)\vert x_t} = \br*{x_t^T(\theta - \theta^i)}^2.
 $$
 Relating $(c)$ holds since $\theta(j)=\theta^i(j)$ for all $j\ne i$, $\abs*{\theta(i)-\theta^i(i)}=2/\sqrt{B}$ and $x_t(i)\in [-1,1]$; thus, we have that $\br*{x_t^T(\theta - \theta^i)}^2 = x_t(i)^2\br*{\theta(i) - \theta^i(i)}^2\le 4/B$.
 
Finally, notice that this holds for any $\theta \in \Theta$ and averaging over all $\theta \in \Theta$ we get:
 $$
 \frac{1}{|\Theta|}\sum_{\theta \in \Theta}\sum_{i = 1}^d\Pb_{\theta}(i) 
  \sum_{i = 1}^d\frac{1}{|\Theta|}\sum_{\theta \in \Theta}\Pb_{\theta}(i) 
 \ge \frac{d \exp(-2)}{4}.
 $$
 This implies there exists at least one $\theta \in \Theta$, say $\tilde \theta$, for which $\sum_{i = 1}^d\Pb_{\tilde \theta}(i) > \frac{d\exp(-2)}{4}$.
 The claim now follows from \cref{eq:lb1_m} and setting $\theta = \tilde \theta$.
\end{proof}

\clearpage

\section{Gap Between Adversarial and Stochastic Contexts}\label{appendix: gap between adversarial and stochastic}

\BlackBoxReductionTheorem*
\begin{remark*}[Budget is increasing in integer units]
For simplicity, we assume the budget is an integer and is increased by an integer number. If such assumption does not hold, the analysis should be modified by rounding the budget to the closest smallest integer number.
\end{remark*}
\begin{proof}
Let $\brc{B(t)}_{t\geq 1}$ be any fixed budget sequence and let $T\in\N$ be some arbitrary time. For brevity, and with slight abuse of notation, we denote $f(u,\pi)=\E\brs*{f(R_t)\vert u,\pi}$ and $\pi^*(u) = \arg\max_{\pi\in\Pi(u)} \E\brs*{f(u,\pi)}$. Under these notations, the expected regret can also be written as
\begin{align*}
    \E\brs*{\Regret(T)} = \E\brs*{\sum_{t=1}^Tf(u_t,\pi^*(u_t))- f(u_t,\pi_t)}.
\end{align*}
Also, let $\indicator{Q_t}=\indicator{q_t=1}$ be the indicator function of the event a reward was queried at the $t^{th}$ time step. Importantly, the greedy query rule of \Cref{alg: greedy reduction} does not depend on the observations, and for this reason, the sequence $\brc*{\indicator{Q_t}}_{t\geq 1}$ is fixed given $\brc{B(t)}_{t\geq 1}$.  Let $\Bq(t) = \sum_{s=1}^{t-1} \indicator{Q_s}$ be number of times a budget was used until time step $t$. Let $\brc{E_1(t), E_2(t)}_{t\geq 0}$ be sequence of events defined as 
\begin{align*}
    &E_1(t) = \brc*{B(t)\geq \Bq(t)+1, B(t-1) < \Bq(t-1)+1}\\
    &E_2(t) = \brc*{B(t)< \Bq(t)+1, B(t-1) \geq \Bq(t-1)+1}.
\end{align*}
In words, $E_1(t)$ represents the event that at time step $t$ budget became available whereas at time step $t-1$ it was not available. Furthermore, $E_2(t)$, represents the event budget became unavailable at time step $t$ whereas it was available at time step $t-1$. Furthermore, define the following time steps (which are deterministic for a fixed budget sequence) 
\begin{align*}
    &\tau_{1}(k) = \inf\brc*{t: \sum_{i=1}^t \indicator{E_1(i)}\geq k}\qquad\quad \mathrm{and}\qquad\quad \tau_{2}(k) = \inf\brc*{t: \sum_{i=1}^t \indicator{E_2(i)}\geq k}.
\end{align*}
That is, $\tau_{1}(k)$ is the $k^{th}$ time the budget became available and $\tau_{2}(k)$ is the $k^{th}$ time the budget became unavailable. We now prove the following claim. For any time step $T$ the expected regret of Algorithm~\ref{alg: greedy reduction} is bounded by
\begin{align}
     \E[\Regret(T)]\leq \E\brs*{\sum_{t\in \mathcal{I}^q_T} f(u_t,\pi^*(u_t))- f(u_t,\pi_t) } +\sum_{t=1}^{T} \frac{\alpha}{B(t)^{1-\beta}} + \frac{C}{B(t)}, \label{eq: induction proof to show}
\end{align}
where $\mathcal{I}^q_t = \brc{k\in [t]: \indicator{Q_k}=1}$, i.e., all time steps until time step $t$ in which a reward was queried. Notice that these are the only time steps where algorithm $\mathbb{A}$ advances, so the algorithm has been effectively applied only on these time steps. Then, by the anytime regret assumption on algorithm $\mathbb{A}$ it holds that the first term in~\eqref{eq: induction proof to show} is bounded by $\alpha \abs*{\mathcal{I}^q_T}^\beta = \alpha\br*{\Bq(T)}^\beta \leq \alpha \min\brc*{B(T),T}^\beta + C$ which completes the proof.

We establish \Cref{eq: induction proof to show} via an induction on time steps in which the budget becomes unavailable, that is on $\tau_2(1),\tau_2(2),\cdots$. Assume that at $t=1$, there is an available budget, that is, $B(1)\geq 1$.

\paragraph{Base case. The claim holds for all  $t\in[1,\tau_{2}(k=1) -1]$.} By assumption and definition, for all $t\in [1,\tau_{2}(1)-1]$ Algorithm~\ref{alg: greedy reduction} query rewards. For this reason, the regret is bounded by the regret of algorithm$~\mathbb{A}$. Thus, for all $t\in [1,\tau_{2}(1)-1]$
\begin{align*}
    &\Regret(t) = \E\brs*{\sum_{t\in \mathcal{I}^q_t} f(u_t,\pi^*(u_t))- f(u_t,\pi_t) } \leq \E\brs*{\sum_{t\in \mathcal{I}^q_t} f(u_t,\pi^*(u_t))- f(u_t,\pi_t) } +\sum_{l=1}^{t} \frac{\alpha}{B(l)^{1-\beta}} + \frac{C}{B(l)}.
\end{align*}
\paragraph{Induction Step: Proving that the claim holds for all $t\in [1,\tau_{2}(k+1)-1]$.} Assume the claim holds until time step $t_k= \tau_{2}(k)-1$, that is, until the $k^{th}$ time the budget becomes unavailable. We now prove the regret bound holds until the $(k+1)^{th}$ time the budget becomes unavailable for $k\geq 1$ (if no such time exists then the exact same proof holds by replacing $\tau_{2}(k+1)$ by $\min\brc*{\tau_{2}(k+1),T+1}$ and $\tau_1(k+ 1)$ by $\min\brc*{\tau_{1}(k+1),T+1}$). Consider the time steps between the $k^{th}$ and $(k+1)^{th}$ time the budget becomes unavailable. We can partition the time steps into two parts $[\tau_{2}(k),\tau_{1}(k+1)-1]$,  $\brc{\tau_{1}(k+1),..,\tau_{2}(k+1)-1}$, that it, between the $k^{th}$ and $(k+1)^{th}$ time the budget becomes unavailable there exists a time step $\tau_{1}(k+1)$ in which the budget becomes available for the $(k+1)^{th}$ time (since we assumed at the first time step a budget is available, so $\tau_1(1)=1$ and $\tau_2(1)>1$).
Furthermore, let $\Regret(t_1:t_2)=\sum_{t_1}^{t_2} f(u_t,\pi^*(u_t))- f(u_t,\pi_t)$ be partial sum of the cummulative regret. The following relations hold.
\begin{align}
    \E[\Regret(&\tau_{2}(k+1)-1)] \nonumber\\
    &= \E[\Regret(\tau_{2}(k)-1)] + \E[\Regret(\tau_{2}(k):\tau_{2}(k+1)-1)] \tag{Additive form of regret}\\
    &\leq  \E\brs*{\sum_{t\in \mathcal{I}^q_{\tau_{2}(k)}} f(u_t,\pi^*(u_t))- f(u_t,\pi_t) } +\sum_{t=1}^{\tau_{2}(k)-1} \br*{\frac{\alpha}{B(t)^{1-\beta}} + \frac{C}{B(t)}} + \E[\Regret(\tau_{2}(k):\tau_{2}(k+1)-1)] \tag{Induction hypothesis}\\
    &=  \E\brs*{\sum_{t\in \mathcal{I}^q_{\tau_{2}(k)}} f(u_t,\pi^*(u_t))- f(u_t,\pi_t) } + \sum_{t=1}^{\tau_{2}(k)-1} \br*{\frac{\alpha}{B(t)^{1-\beta}} + \frac{C}{B(t)}} + \E[\Regret(\tau_{2}(k):\tau_{1}(k+1)-1)] \nonumber\\
    &\quad+ \E\brs*{\sum_{t=\tau_1(k+1)}^{\tau_2(k+1)-1} f(u_t,\pi^*(u_t))- f(u_t,\pi_t) } \nonumber\\
    &= \E\brs*{\sum_{t\in \mathcal{I}^q_{\tau_{2}(k+1)-1}} f(u_t,\pi^*(u_t))- f(u_t,\pi_t) } 
    + \sum_{t=1}^{\tau_{2}(k)-1} \br*{\frac{\alpha}{B(t)^{1-\beta}} + \frac{C}{B(t)}} + \E[\Regret(\tau_{2}(k):\tau_{1}(k+1)-1)]. \label{eq: gap adv stoc contexts}
\end{align}
where the equalities are by the regret definition and since in the interval $\brs*{\tau_{2}(k),\tau_{2}(k+1)-1}$, rewards were queried only at $\brs*{\tau_{1}(k),\tau_{2}(k+1)-1}$; therefore, $\mathcal{I}^q_{\tau_{2}(k+1)-1} = \mathcal{I}^q_{\tau_{2}(k)-1}\cup\brc*{\tau_{1}(k),\dots,\tau_{2}(k+1)-1}$. 
All that remains is to bound $\E[\Regret(\tau_{2}(k):\tau_{1}(k+1)-1)]$. In time steps $t\in [\tau_{2}(k):\tau_{1}(k+1)-1]$ there is no available budget and Algorithm~\ref{alg: greedy reduction} acts in accordance to one of the previous queried rounds; i.e., $j\sim Uniform(I^q_t)$ and $\pi_t\!=\!A_j(u_t)$. Notably, in these time steps, we have $I^q_t=I^q_{\tau_2(k)-1}$. Also recall that $\mathcal{I}^q_{\tau_{2}(k)-1}$ is deterministic given $\brc{B(t)}_{t\geq 1}$. Thus, the expected value of $\pi_t$ is given by
\begingroup
\allowdisplaybreaks
\begin{align*}
    \E[ f(u_t,\pi^*(u_t))- f(u_t,\pi_t)] 
    &= \E[\E[ f(u_t,\pi^*(u_t))- f(u_t,\pi_t) \vert F_{\tau_{2}(k)-1}]]\\
    & = \E\brs*{\frac{1}{ \abs*{I^q_{\tau_2(k)-1}} }\sum_{j\in I^q_{\tau_2(k)-1}} \E[ f(u_t,\pi^*(u_t))- f(u_t,\mathbb{A}_j(u_t)) \vert F_{\tau_{2}(k)-1}]}\\
    & \overset{(1)}{=} \E\brs*{\frac{1}{ \abs*{I^q_{\tau_2(k)-1}} }\sum_{j\in I^q_{\tau_2(k)-1}} \E[ f(u_t,\pi^*(u_t))- f(u_t,\mathbb{A}_j(u_t)) \vert F_{\tau_{2}(j)-1}]} \\
    & \overset{(2)}{=} \E\brs*{\frac{1}{ \abs*{I^q_{\tau_2(k)-1}} }\sum_{j\in I^q_{\tau_2(k)-1}} \E[ f(u_j,\pi^*(u_j))- f(u_t,\mathbb{A}_j(u_j)) \vert F_{\tau_{2}(j)-1}]} \\
    & = \frac{1}{ \abs*{I^q_{\tau_2(k)-1}} }E\brs*{\sum_{j\in I^q_{\tau_2(k)-1}} f(u_j,\pi^*(u_j))- f(u_t,\mathbb{A}_j(u_j))} \\
    &\overset{(3)}{\le}  \frac{\alpha\abs*{I^q_{\tau_2(k)-1}}^\beta + C}{\abs*{I^q_{\tau_2(k)-1}}} \\
    & \overset{(4)}{=}  \frac{\alpha}{B(\tau_{2}(k))^{1-\beta}} + \frac{C}{B(\tau_{2}(k))}.
\end{align*}
\endgroup
Relation $(1)$ is since rewards and contexts are i.i.d. between rounds and $\mathbb{A}_j$ only depends on samples in $\mathcal{I}^q_{\tau(j)-1}$. Equality $(2)$ is since the contexts are i.i.d. and $\mathbb{A}_j$ only depends on samples in $F_{\tau_{2}(j)-1}$ and internal randomness; therefore, for any $t\ge j$, 
\begin{align*}
    \E[ f(u_t,\pi^*(u_t))- f(u_t,\mathbb{A}_j(u_t)) \vert F_{\tau_{2}(j)-1}]
     = \E[ f(u_j,\pi^*(u_j))- f(u_j,\mathbb{A}_j(u_j)) \vert F_{\tau_{2}(j)-1}]\enspace.
\end{align*}
Next, $(3)$ is by the regret bound of algorithm $\mathbb{A}$ on samples in $I^q_{\tau_2(k)-1}$, the only rounds where the algorithm advanced. Finally, $(4)$ holds since the the budget at time step $\tau_{2}(k)$ equals to the number of queries at time step $\tau_{2}(k)-1$ (if not, then there is an available budget at time step $\tau_{2}(k)$, in contradiction to its definition). Plugging this back into~\eqref{eq: gap adv stoc contexts} we get
\begin{align*}
    \eqref{eq: gap adv stoc contexts} &\leq \E\brs*{\sum_{t\in \mathcal{I}^q_{\tau_{2}(k+1)-1}} f(u_t,\pi^*(u_t))- f(u_t,\pi_t) } +\sum_{t=1}^{\tau_{2}(k)} \br*{\frac{\alpha}{B(t)^{1-\beta}} +  \frac{C}{B(t)}} + (\tau_{1}(k+1)-1 - \tau_{2}(k))\alpha B(\tau_{2}(k)) ^{\beta -1 }\\
    & = \E\brs*{\sum_{t\in \mathcal{I}^q_{\tau_{2}(k+1)-1}} f(u_t,\pi^*(u_t))- f(u_t,\pi_t) } +\sum_{t=1}^{\tau_{1}(k+1)-1} \br*{\frac{\alpha}{B(t)^{1-\beta}} +  \frac{C}{B(t)}} \tag{$B(t)$ is fixed for $t\in [\tau_{2}(k),\tau_{1}(k+1)]$}\\
    &\leq \E\brs*{\sum_{t\in \mathcal{I}^q_{\tau_{2}(k+1)-1}} f(u_t,\pi^*(u_t))- f(u_t,\pi_t) } +\sum_{t=1}^{\tau_{2}(k+1)-1} \br*{\frac{\alpha}{B(t)^{1-\beta}} +  \frac{C}{B(t)}}, \tag{Budget is positive}
\end{align*}
which proves the induction hypothesis.

\end{proof}

\clearpage
\GreedyEquivalent*
\begin{proof}
We prove the claim iteratively; specifically, we prove that for any $b\in\brc*{0,\dots,B-1}$ and any policy $\pi^b$ for which $q_t=1 ,\forall t \in\brs*{b}$, there exists a policy $\pi^{b+1}$ such that $q_t=1 ,\forall t \in\brs*{b+1}$ and $\E\brs*{\Regret(T)\vert \pi^{b+1}}=\E\brs*{\Regret(T)\vert \pi^b}$. Then, we can choose $\pi^0=\pi$ and apply this result $B$ times to obtain $\pi'=\pi^B$ for which $q_t=1$ for all $t\in\brs*{B}$ and $\E\brs*{\Regret(T)\vert \pi'}=\E\brs*{\Regret(T)\vert \pi}$.

To prove this claim, we further delve into the probabilistic model of the decision-making problem, which we carefully choose for this prove (see, e.g., \citealt{lattimore2020bandit}, Chapters 4.6,4.7, for more details). Notice that under our assumptions, the action set is fixed for all time steps and only reward is generated. Also recall that under our model, and given the action, the reward is generated independently of other rounds. Denote by $\nu_a$, the distribution of $R_t$ given an action $a$ was taken, and let $\Bq(t)$ be the number of queries taken up to time $t$. Then, we describe the decision-process through the following random variables:
\begin{itemize}
    \item $U$ is a uniform random that is generated prior to the game and represents all randomness of the agent
    \item $a_1\dots,a_T$ and $q_1,\dots,q_T$ are the (random) actions and queries taken by the agent
    \item We let $R_1^q,\dots,R_T^q$ and $R_1^u,\dots,R_T^u$ be two reward sequences, one for queried actions and one for unqueried ones. Formally, if $q_t=1$, then $R_{\Bq(t)}^q$ is sampled independently at random from $\nu_{a_t}$ and $R_t = R_{\Bq(t)}^q$. Similarly, if $q_t=0$, then $R_{t-\Bq(t)}^u$ is sampled independently at random from $\nu_{a_t}$ and $R_t = R_{t-\Bq(t)}^u$. We brevity, we denote $Y_t = R_t\cdot q_t$.
    \item We emphasize that under this model, a policy $\pi_t$ deterministically maps $U$, $t$ and the sequence $\brc*{Y_k}_{k:q_k=1}$ to action and queries $a_t^\pi$ and $q_t^\pi$ (given the internal randomization, the decision rule is deterministic, so previous actions are not needed to describe the mapping).
\end{itemize}

Now, let $b\in\brc*{0,\dots,B-1}$ and let $\pi=\pi^b$ be a policy such that $q_t=1$ for all $n\in\brs*{b}$, and let $U,Y_1,\dots,\dots, Y_{b}$ an instantiation of the internal randomness and the rewards. Notably, under our model, $Y_t=R_t^q$ for all $t\in\brs*{b}$. Also note that until a new query is taken, the policy receives no new input. This implies that the action sequence up to the $(b+1)^{th}$-query is deterministic given $U,Y_1,\dots,\dots, Y_{b}$. We denote this time by $\tau(b+1)$ and say that $\tau(b+1)=\phi$ if no additional query is taken.

Next, we define a new policy $\pi'=\pi^{b+1}$ as follows:
\begin{itemize}
    \item For any $t\le b$ we fix $\pi'=\pi$.
    \item For any $U,Y_1,\dots,\dots, Y_{b}$, if $\tau(b+1)=\phi$, we set $q_{b+1}^{\pi'}=1$ but continue choosing $a_t^{\pi'} = a_t^{\pi}$ for all $t\ge b+1$ and $q_{b+1}^{\pi'}=0$ for any $t\ge b+2$.
    \item For any $U,Y_1,\dots,\dots, Y_{b}$, if $\tau(b+1)\ne\phi$, we permute the action at time $\tau(b+1)$ to time $b+1$ and delay the actions of $\pi$ at times $b+1,\dots,\tau(b+1)-1$ by a single time step. Formally:
    \begin{itemize}
        \item $a_{b+1}^{\pi'} = a_{\tau(b+1)}^{\pi^{b}}$ and $q_{b+1}=1$.
        \item $a_{t}^{\pi'} = a_{t-1}^{\pi^{b}}$ and $q_{t}=0$ for all $t\in\brc*{b+2,\dots,\tau(b+1)}$.
        \item $\pi'_t=\pi_t$ for all $t>\tau(b+1)$.
    \end{itemize}
\end{itemize}
Clearly, $q_t^{\pi'}=1 ,\forall t \in\brs*{b+1}$. Thus, it remains to prove that $\E\brs*{\Regret(T)\vert \pi'}=\E\brs*{\Regret(T)\vert \pi}$. Denote the instantaneous regret of algorithm $\pi$ by $r_t^\pi = \max_{\pi\in\Pi}\E\brs*{f(R_t)\vert \pi} - f(R_t)$, where $R_t$ is generated when playing according to $\pi$. We use a coupling argument, where as long as the policies agree, the model simultaneously generates the same reward for both process. As soon as either $R_k^q$ or $R_k^u$ is supposed to be generated from a different action, then the processes split into two independent process and continue separately.
Since $\pi'_t=\pi_t$ for all $t\le b$, it also implies that the action and reward processes for both polices are identical up to time $b$ and $r_t^{\pi'} = r_t^\pi$ for all $t\le b$. 

For the rest of the time steps, for any instantiation of $U,Y_1,\dots,Y_b$, we divide the analysis into two cases:
\begin{enumerate}
    \item $\tau(b+1)\ne\phi$. In this case, notice that during time steps $b+1,\dots,\tau(b+1)$, both policies generate $R_{b+1}^q$ from $a_{\tau(b+1)}^\pi$ and generate $R_1^u,\dots,R_{\tau(b+1)-b-1}$ from $a_{b+1}^\pi,\dots a_{\tau(b+1)-1}^\pi$. Therefore, for these time steps, both policies agree on the actions and generate the same rewards, albeit in a different order. Specifically, this implies that $\sum_{t=b+1}^{\tau(b+1)}r_t^{\pi'} = \sum_{t=b+1}^{\tau(b+1)}r_t^\pi$. Moreover, both policies generated the same queried reward $R_{b+1}^q$ and continues the same for any $t>\tau(b+1)$, which implies that they will generate the same actions and rewards until the end of the interactions. Thus, in this case, we have that $\sum_{t=b+1}^{T}r_t^{\pi'} = \sum_{t=b+1}^{T}r_t^\pi$.
    \item $\tau(b+1)=\phi$. In this case, no additional reward is queried, and the  same sequence of actions $a_{b+1},\dots,a_T$ is deterministically chosen as a function of $U,Y_1,\dots,Y_b$, for both $\pi$ and $\pi'$. Thus, as the reward is generated from the same distribution for both policies, we have for any $t>b+1$
    \begin{align*}
        E\brs*{r_t^{\pi'} \vert U,Y_1,\dots,Y_b} = E\brs*{r_t^\pi \vert U,Y_1,\dots,Y_b}
    \end{align*}
\end{enumerate}
Combining both parts, and using the tower property we have
\begin{align*}
    \E\brs*{\sum_{t=b+1}^T r_t^{\pi'}}
    &= \E\brs*{\E\brs*{\sum_{t=b+1}^T r_t^{\pi'}\bigg\vert U,Y_1,\dots,Y_b}} \\
    & = \E\brs*{\E\brs*{\indicator{\tau(b+1)\ne\phi}\underbrace{\sum_{t=b+1}^T r_t^{\pi'}}_{=\sum_{t=b+1}^T r_t^{\pi}}\bigg\vert U,Y_1,\dots,Y_b} + \sum_{t=b+1}^T \indicator{\tau(b+1)=\phi}\underbrace{\E\brs*{r_t^{\pi'}\bigg\vert U,Y_1,\dots,Y_b}}_{=\E\brs*{r_t^{\pi}\bigg\vert U,Y_1,\dots,Y_b}}} \\
    & = \E\brs*{\E\brs*{\indicator{\tau(b+1)\ne\phi}\sum_{t=b+1}^T r_t^{\pi}\bigg\vert U,Y_1,\dots,Y_b} + \sum_{t=b+1}^T \indicator{\tau(b+1)=\phi}\E\brs*{r_t^{\pi}\bigg\vert U,Y_1,\dots,Y_b}} \\
    & = \E\brs*{\E\brs*{\sum_{t=b+1}^T r_t^{\pi}\bigg\vert U,Y_1,\dots,Y_b}} \\
    & = \E\brs*{\sum_{t=b+1}^T r_t^{\pi}}.
\end{align*}
Finally, recalling that the polices are identical for $t\le b$, we get the desired result for $\pi'=\pi^{b+1}$ and $\pi=\pi^b$:
\begin{align*}
    \E\brs*{\Regret(T)\vert\pi'} 
    &= \E\brs*{\sum_{t=1}^b r_t^{\pi'}} + \E\brs*{\sum_{t=b+1}^T r_t^{\pi'}} 
    = \E\brs*{\sum_{t=1}^b r_t^\pi} + \E\brs*{\sum_{t=b+1}^T r_t^{\pi}}  
    = \E\brs*{\Regret(T)\vert\pi}\enspace. 
\end{align*}

\end{proof}

\clearpage

\CounterExamplesPropositions*
\begin{proof}
Consider a contextual multi-armed bandit instance with two contexts.  That is, the environment has two contexts $u=1,2$; if we observe $u=1$, then we interact with the first MAB problem, and if $u=2$, we interact with the second MAB problem. We assume there is no relation between the first and second MAB problems.

{\bf Greedy Reduction Algorithm in the Presence of Adversarial Contexts.} \\
Assume the budget sequence increases at each episode w.p. $\frac{1}{2}$ by a unit. Assume the adversary picks $u=1$ if the budget increases by one and $u=2$ if the budget does not increase. Applying the Greedy Reduction, the algorithm will only query information from MAB $u=1$ and will have no information on the rewards for MAB $u=2$. Let the bandit problems of both contexts be two-armed problems. Then, there exists an arm $a^*$ of context $u=2$ that the algorithm samples, in expectation, at most $\frac{T}{4}$ times, namely, 
$$\E\brs*{\sum_{t=1}^T \indicator{u_t = 2, a_t= a^*}}\le \frac{1}{2}\E\brs*{\sum_{t=1}^T \indicator{u_t = 2}} =  \frac{T}{4},$$
where we used the fact that $\Pr\br*{u_t=2} = \Pr\br*{B(t)=B(t-1)+1}=\frac{1}{2}$. For this arm, we fix the reward to be equal $R=1$, and for the other arm, we let $R=0$. For simplicity, we fix the means of all arms in context $u=1$ to be identical. Then, the regret is lower bounded by 
\begin{align*}
    \E\brs*{\Regret(T)} 
    =  \E\brs*{\sum_{t=1}^T \indicator{u_t = 2, a_t\ne a^*}}
    = \E\brs*{\sum_{t=1}^T \indicator{u_t = 2}} - \E\brs*{\sum_{t=1}^T \indicator{u_t = 2, a_t= a^*}}
    \ge \frac{T}{2} - \frac{T}{4}
    = \frac{T}{4}
\end{align*}
Notice that by definition, $\E\brs*{B(t)}=\frac{t}{2}$ for all $t\in\brs*{T}$.

{\bf Greedy Reduction Algorithm in the Presence of Adversarial Budget.}\\
The example is symmetric to the previous one, i.e., we exchange the roles of the context and the budget. Assume that the context are chosen stochastically such that where $\Pr(u_t=1) = \Pr(u_t=2) =\frac{1}{2}$. If the budget increases by one each time $u_t=1$ then Greedy Reduction will not acquire any information on the MAB with $u_t=2$. Repeating the same analysis as the previous case results in $\E[\Regret(T)]\geq \Omega(T)$ as well as $\E\brs*{B(t)} = \frac{t}{2}$ for all $t\in\brs*{T}$.
\end{proof}

\clearpage


\section{Confidence-Budget Matching for Multi Armed Bandits}\label{appendix: cbm for bandits}
\begin{algorithm}[H]
\caption{CBM-UCB} \label{alg: CBM-UCB}
\begin{algorithmic}[1]
\STATE {\bf Initialize:} $n^q_1(a)=0, \bar{r}_1(a)=0$
\FOR{$t=1,...,T$}
\STATE Observe current budget $B(t)$
\STATE Act with $a_t\in\arg\max_a UCB_t(a)$
\IF{$CI_t(a_t) \ge 4\sqrt{\frac{6\log(\Narms t)\sum_{a=1}^{\Narms}c(a)}{B(t)}}$ (or, alternatively, $n^q_{t-1}(a_t)\le \frac{B(t)}{4\sum_{a=1}^{\Narms}c(a)}$)} 
    \STATE Ask for feedback ($q_t=1$) 
    \STATE Observe $R_t$ and update $n^q_t(a_t), \bar{r}_t(a_t)$
\ENDIF
\ENDFOR
\end{algorithmic}
\end{algorithm}
We start with some notations: let $n^q_t(a) = \sum_{k=1}^t \indicator{a_t=a,q_t=1}$ be the number of times arm $a$ was queried up to time $t$ and let $\bar{r}_t(a)=\frac{1}{n^q_t(a)\vee 1}\sum_{k=1}^t R_t \indicator{a_t=a,q_t=1}$ be its empirical mean. We use Hoeffding-based CI, i.e., if $b_{t}^r(a)\triangleq \sqrt{\frac{3\log(\Narms t)}{2n^q_{t-1}(a)\vee 1}}$, then 
\begin{align*}
    UCB_t(a) = \bar{r}_{t-1}(a)+b_{t}^r(a)\qquad \mathrm{and}\qquad LCB_t(a) = \bar{r}_{t-1}(a)- b_{t}^r(a),
\end{align*}
which leads to  $CI_t(a) = UCB_t(a) - LCB_t(a) = 2 b_{t}^r(a) = \sqrt{6\frac{\log  \Narms t}{n^q_t(a)\vee 1}}$. 

\begin{remark*}
The observant reader might find the CBM condition wasteful; a better condition, for example, would be to sample arms if $n^q_{t-1}(a_t)\le \frac{B(t)}{\sum_{a=1}^{\Narms}c(a)}+c(a_t)$. Indeed, doing so will improve the constants of the regret bounds. Nonetheless, the main goal of this section is to demonstrate the techniques we use for the more complex settings (linear bandits and RL), where the CBM principle is not equivalent to count-thresholding and there is no clear way to tune the querying condition to be tighter.
\end{remark*}

We now prove the regret bound for CBM-UCB:

\setcounter{theorem}{1}
\begin{theorem}[Confidence Budget Matching for Multi Armed Bandits]\label{supp theorem: CBM Bandits}
For any querying costs $c(1),\dots,c(\Narms)\ge0$, any adaptive non-decreasing adversarially chosen sequence $\brc{B(t)}_{t\geq 1}$ and for any $T\ge1$, the expected regret of CBM-UCB is upper bounded by 
\begin{align*}
\E[\Regret(T)] \leq \sqrt{24\Narms T\log(\Narms T)} + \Narms\sqrt{6\log(\Narms T)} +  4\sqrt{6\log(\Narms T)\sum_{a=1}^{\Narms}c(a)}\sum_{t=1}^T \E\brs*{\sqrt{\frac{1}{B(t)}}} +  2
\end{align*}
\end{theorem}

\begin{proof}
Define the filtration $\brc{F_t}_{t\geq 0}$ where $F_{t-1}$ contains $\brc{B(1), a_1,R_{1},\dots,B(t-1),a_{t-1},R_{t-1},B(t)}$, i.e., all past actions, rewards and budgets, combined with the value of total budget at time step $t$. We also assume w.l.o.g. that $\Narms\ge2$, otherwise the regret is always zero. The good event is defined as
\begin{align*}
    &E^r(t) = \brc*{a \in A:\ |\bar{r}_{t-1}(a) -r(a)|  \leq \sqrt{ \frac{3\log(\Narms t) }{2n^q_{t-1}(a)\vee 1}  } \eqdef b_{t}^r(a)}.
\end{align*}
Importantly, notice that when $E^r(t)$ holds, then 
$r(a^*)\le UCB_t(a^*) \le UCB_t(a_t)$, and thus, the UCB of the chosen arm is optimistic. Moreover, the event directly implies that $r(a_t)\ge LCB_t(a_t)$.

{\bf Meeting the budget constraint} We prove that the CBM principle  of \Cref{alg: CBM-UCB} never violates the budget constraint in \Cref{lemma: cbm ucb budget constraint is satisfied}, that is, $\Bq(t)\le B(t)$ for all $t\in\brs*{T}$. Thus, throughout the proof, we assume that $q_t=1$ if and only if the CBM query rule decides so (i.e., the case where the algorithm wants to set $q_t=1$ and does not have enough budget to do so cannot happen).

{\bf Regret analysis.} We decouple the regret as follows.
\begin{align*}
    \E[\Regret(T)] &= \sum_{t=1}^T \E\brs*{(r^* - r(a_t))\indicator{E^r(t)}} + \sum_{t=1}^T \E\brs*{(r^* - r(a_t))\indicator{\overline{E^r(t)}}}\\
    & \le \sum_{t=1}^T \E\brs*{(r^* - r(a_t))\indicator{E^r(t)}} + \sum_{t=1}^T \E\brs*{\indicator{\overline{E^r(t)}}} \tag{$(r^* - r(a)\in 1$ for all $a\in [A]$}\\
    & \le \sum_{t=1}^T \E\brs*{(r^* - r(a_t))\indicator{E^r(t)}} + \sum_{t=1}^T \frac{1}{t^2} \tag{Lemma~\ref{lemma: good event bandits}}\\
    &\le \underbrace{\sum_{t=1}^T \E\brs*{\indicator{q_t=1}(r^* - r(a_t))\indicator{E^r(t)}}}_{(i)} + \underbrace{\sum_{t=1}^T \E\brs*{\indicator{q_t=0}(r^* - r(a_t))\indicator{E^r(t)}}}_{(ii)}+2.
\end{align*}
Terms $(i)$ and $(ii)$ represent the regret over episodes in which feedback was queried and not queried, respectively. We bound each of the terms separately.

{\bf Bound on term $(i)$, episodes in which feedback is queried, under the good event.} The following relations hold
\begin{align*}
    (i) &= \sum_{t=1}^T \E\brs*{\indicator{q_t=1}(r^* - r(a_t))\indicator{E^r(t)}}\\
    &\leq \sum_{t=1}^T \E\br*{\indicator{q_t=1}\indicator{E^r(t)}(UCB(a_t) - r(a_t))} \tag{Optimism in the event $E^r(t)$}\\
    &\leq \sum_{t=1}^T \E\br*{\indicator{q_t=1}\indicator{E^r(t)}(UCB(a_t) - LCB(a_t))} \tag{In the event $E^r(t)$}\\
    & \le  \E\brs*{ \sum_{t=1}^T \indicator{q_t=1} \sqrt{\frac{6\log(\Narms t)}{n_{t-1}^q(a_t)\vee 1}}}.
\end{align*}
We bound the term in the expectation as follows (for every history sequence).
\begin{align*}
    \sum_{t=1}^T \indicator{q_t=1} \sqrt{\frac{6\log(\Narms t)}{n_{t-1}^q(a_t)\vee 1}}
    &\leq  \sqrt{6\log(\Narms T)}\sum_{t=1}^{T} \frac{\indicator{q_t=1}}{\sqrt{n_{k-1}^q(a_t)\vee 1}}\\
    & \overset{(*)}{=}\sqrt{6\log(\Narms T)}\sum_{a=1}^{\Narms} \sum_{i=0}^{n^q_T(a)} \frac{1}{\sqrt{i \vee 1}}\\
    &\leq \sqrt{6\log(\Narms T)}\sum_{a} \br*{2\sqrt{n^q_T(a)} + 1} \tag{$\sum_{i=1}^T\frac{1}{\sqrt{i}}\le 2\sqrt{T}$}\\
    &\leq \sqrt{24\log(\Narms T)} \sqrt{\Narms n^q_T} + \Narms\sqrt{6\log(\Narms T)} \tag{Jensen's inequality and $\sum_{a}n^q_T(a) = n^q_T$},
\end{align*}
where $(*)$ holds since every time an action $a$ was queried, its counter advanced by 1.
In the second relation, the summation is performed over time steps a reward is queried. Lastly, using $n^q_T \leq T$, i.e., the number of times the algorithm queried feedback is smaller than the total number of round, we get,
\begin{align*}
    (i) \leq \sqrt{24\Narms T\log(\Narms T)} + \Narms\sqrt{6\log(\Narms T)}. 
\end{align*}
{\bf Bound on term~$(ii)$, episodes in which feedback is not queried.} To bound this term, we use the query rule, that is,
\begin{align*}
    (ii) &= \sum_{t=1}^T \E[\indicator{q_t=0}\indicator{E^r(t)}(r^* - r(a_t))] \\
    &\leq  \sum_{t=1}^T \E[\indicator{q_t=0}\indicator{E^r(t)}(UCB(a_t) - LCB(a_t))] \tag{Optimism}\\
    &\leq \sum_{t=1}^T \E\brs*{4\sqrt{\frac{6\log(\Narms t)\sum_{a=1}^{\Narms}c(a)}{B(t)}}} \tag{CBM query rule}\\
    &\leq 4\sqrt{6\log(\Narms T)\sum_{a=1}^{\Narms}c(a)}\sum_{t=1}^T \E\brs*{\sqrt{\frac{1}{B(t)}}}.
\end{align*}
{\bf Combining the bounds.} Combining the two bounds we conclude the proof,
\begin{align*}
     \E[\Regret(T)] \leq \sqrt{24\Narms T\log(\Narms T)} + \Narms\sqrt{6\log(\Narms T)} +  4\sqrt{6\log(\Narms T)\sum_{a=1}^{\Narms}c(a)}\sum_{t=1}^T \E\brs*{\sqrt{\frac{1}{B(t)}}} +  2.
\end{align*}
\end{proof}

\subsection{The Good Event}
\begin{lemma}[The Good Event]\label{lemma: good event bandits}
For any $\Narms\ge2$ and $t\ge1$, it holds that $\Pr(\overline{E^r(t)})\leq \frac{1}{t^2}$.
\end{lemma}
\begin{proof}
Fix $t\ge1$. The following relations hold by applying Hoeffding's inequality and the union bound. Denote by $\hat{r}_n(a)$, the empirical mean of $n$ i.i.d random variables over $[0,1]$ with an expectation $r(a)$. Then, we have
\begin{align*}
    \Pr&\br*{\abs*{\bar{r}_{t-1}(a_t) -r(a_t)}  \geq \sqrt{ \frac{3\log (\Narms t)}{2n^q_{t-1}(a_t)\vee 1}}}\\
    &=\sum_{a=1}^{\Narms}\Pr\br*{\abs*{\bar{r}_{t-1}(a) -r(a)}  \geq \sqrt{ \frac{3\log (\Narms t)}{2n^q_{t-1}(a)\vee 1}}, a_t=a}\\
    &\le \sum_{a=1}^{\Narms}\Pr\br*{\abs*{\bar{r}_{t-1}(a) -r(a)}  \geq \sqrt{ \frac{3\log (\Narms t)}{2n^q_{t-1}(a)\vee 1}}}\\
    &= \sum_{a=1}^{\Narms}\Pr\br*{\cup_{n=0}^t\brc*{\abs*{\bar{r}_{t-1}(a) -r(a)}  \geq \sqrt{ \frac{3\log (\Narms t)}{2n\vee 1}}, n_{t-1}^q(a)=n}}\\
    &= \sum_{a=1}^{\Narms}\Pr\br*{\cup_{n=1}^t\brc*{\abs*{\hat{r}_n(a) -r(a)}  \geq \sqrt{ \frac{3\log (\Narms t)}{2n}}, n_{t-1}^q(a)=n}} \tag{Holds trivially for $n=0$}\\
    &\le \sum_{a=1}^{\Narms}\sum_{n=1}^t\Pr\br*{\abs*{\hat{r}_{n}(a) -r(a)}  \geq \sqrt{ \frac{3\log (\Narms t)}{2n}}}\tag{Union bound}\\
    &\leq \sum_{a=1}^{\Narms}\sum_{n=1}^t\frac{1}{At^3} \tag{Hoeffding's inequality} \\
    & = \frac{1}{t^2}.
\end{align*}
\end{proof}

\subsection{The Budget Constraint is not Violated}

\begin{lemma}[CBM-UCB: Budget Constraint is Satisfied]\label{lemma: cbm ucb budget constraint is satisfied}
For any $t\geq 1$ the budget constraint is not violated (a.s.), $\Bq(t)\leq B(t)$.
\end{lemma}
\begin{proof}
Assume w.l.o.g. that $c(a)>0$ for at least one action, as otherwise, the budget constraint can never be violated. Similarly, assume that $B(1)>0$; otherwise, until a budget becomes available, no reward will be queried, and the same analysis would hold, starting from the first time querying became available. The following relations hold for any history.
\begin{align*}
    \Bq(t) &= \sum_{k=1}^{t} c(a_k)\indicator{q_k=1}\\
    &= \sum_{a=1}^{\Narms}\sum_{k=1}^{t} c(a)\indicator{a_k=a,q_k=1}\\
    &\leq  \sum_{a=1}^{\Narms}c(a)\sum_{k=1}^t \indicator{a_k=a,q_k=1} \frac{CI_t(a_t)}{ 4\sqrt{6\log(\Narms k)(\sum_{a'=1}^{\Narms} c(a'))/{B(k)}}} \tag{By the update rule, when $q_k=1$}\\
    &=\sum_{a=1}^{\Narms}c(a)\sum_{k=1}^t\indicator{a_k=a,q_k=1} \frac{\sqrt{6 \log(\Narms k)/{(n_{k-1}^q(a)}\vee 1)}}{4\sqrt{6 \log(\Narms k)(\sum_{a'=1}^{\Narms} c(a'))/{B(k)}}}\\
    &=\frac{1}{4\sqrt{\sum_{a'=1}^{\Narms} c(a')}}\sum_{a=1}^{\Narms}c(a)\sum_{k=1}^t\indicator{a_k=a,q_k=1} \sqrt{\frac{1}{{(n_{k-1}^q(a)}\vee 1)}}\sqrt{B(k)}\\
    &\leq \frac{1}{4\sqrt{\sum_{a'=1}^{\Narms} c(a')}}\sqrt{B(t)}\sum_{a=1}^{\Narms} c(a)\sum_{i=0}^{n^q_t(a)}  \sqrt{\frac{1}{i\vee 1}} \tag{$B(t)$ is increasing}\\
    & \le \frac{1}{2\sqrt{\sum_{a'=1}^{\Narms} c(a')}}\sqrt{B(t)}\sum_{a=1}^{\Narms} c(a)\sum_{i=1}^{n^q_t(a)}  \sqrt{\frac{1}{i}} \\
    &\leq \frac{1}{\sqrt{\sum_{a'=1}^{\Narms} c(a')}}\sqrt{B(t)}\sum_{a=1}^{\Narms}c(a)\sqrt{n^q_{t}(a)}\tag{$\sum_{i=1}^T\frac{1}{\sqrt{i}}\le 2\sqrt{T}$}\\
    &\leq \frac{1}{\sqrt{\sum_{a'=1}^{\Narms} c(a')}}\sqrt{B(t)}\sqrt{\sum_{a'=1}^{\Narms} c(a')} \sqrt{\sum_{a=1}^{\Narms} c(a)n^q_{t}(a)} \tag{Cauchy-Schwartz inequality}\\
    &= \sqrt{B(t)}\sqrt{\Bq(t)}. 
\end{align*}
Rearranging we get that $\Bq(t)\leq B(t)$ for any $t\geq 0$.
\end{proof}

\newpage
\section{Confidence-Budget Matching for Linear Bandits}\label{appendix: cbm for linear bandits}

\begin{algorithm}[H]
\caption{CBM-OFUL} \label{alg: CBM-OFUL}
\begin{algorithmic}[1]
\STATE {\bf Require:} $\delta\in\br*{0,1}, \lambda,d,\sigma,L,D>0$
\STATE {\bf Set:} $V_0 = \lambda I_d, \hat\theta_0=\mathbf{0}_d$, $v_t =  \sqrt{2d\log\br*{1+\frac{t L^2}{d\lambda}}}$, $l_t =\max\brc*{1,\sigma\sqrt{2d\log\brc*{\frac{1+ t L^2/\lambda}{\delta}}} + \lambda^{1/2}D}$
\FOR{$t=1,...,T$}
\STATE Observe current budget $B(t)$ and context space $\mathcal{X}_t$
\STATE Act with $x_t\in\arg\max_{x\in\mathcal{X}_t}\max_\theta\in C_{t-1} \inner{x,\theta}$ for $C_t = \brc*{\theta\in \mathbb{R}^d:  \norm{\hat{\theta}_t-\theta}_{V_t}\leq l_t}$
\STATE Calculate $CI_t(x_t) = 2 l_{t-1}\min\brc*{ \norm{x_t}_{V_{t-1}^{-1}},1}$
\IF{$CI_t(x_t) \ge \frac{l_{t-1}v_{B(t)}}{2\sqrt{B(t)}}$ (or, alternatively, $\norm{x_t}_{V_{t-1}^{-1}} \ge \frac{v_{B(t)}}{\sqrt{B(t)}}$)} 
    \STATE Ask for feedback ($q_t=1$) 
    \STATE Observe $R_t$ and update $V_t=V_{t-1} + x_tx_t^T$ and $\hat{\theta}_t$ according to \eqref{eq: budgeted LS}
\ENDIF
\ENDFOR
\end{algorithmic}
\end{algorithm}

We start by more formally define the linear bandit model with budget constraints. At the beginning of each round $t$, an adaptive adversary reveals to the learner a budget $B(t)$ (such that $B(t)\ge B(t-1)$) and a context space $\mathcal{X}_t\subset\R^d$ (which serves as the `context' $u_t=\mathcal{X}_t$).  Then the learner selects a context $x_t\in\mathcal{X}_t$ and, if enough budget is available, she can choose to query for a noisy reward feedback $R_t=\inner{x_t,\theta^*} +\eta_t$, for some unknown $\theta^*\in\R^d$ such that $\inner{x,\theta^*}\in\brs*{-1,1}$ for all $x\in\mathcal{X}_t,t\ge1$. In this section, we assume that querying rewards incur unit costs. the noise is assumed to be zero-meaned and conditionally $\sigma^2$-subgaussian. We also assume that for all $\norm{x}_2\le L, \forall x\in\mathcal{X}_t$ and that $\norm{\theta^*}_2\le D$. The (pseudo) regret in this setting is defined as $\Regret(T) = \sum_{t=1}^T \max_{x\in\mathcal{X}_t}\inner{x,\theta^*} - \inner{x_t,\theta^*}$. 

 Define $V_t = \lambda I_d + \sum_{k=1}^t x_kx_k^T\indicator{q_k=1}$ for some $\lambda>0$.  At the end of each round, the algorithm calculates the regularized least-squares estimator for $\theta^*$ over the queried rewards, namely 
\begin{align}
    \label{eq: budgeted LS}
    \hat\theta_t = V_t^{-1}\sum_{k=1}^t x_tR_t\indicator{q_k=1}.
\end{align}
Then, we define the confidence set $C_t = \brc*{\theta\in \mathbb{R}^d:  \norm{\hat{\theta}_t-\theta}_{V_t}\leq l_t}$, where $\norm{x}_A=\sqrt{x^TAx}$ and 
\begin{align*}
    &l_t =\max\brc*{1,\sigma\sqrt{2d\log\br*{\frac{1+ t L^2/\lambda}{\delta}}} + \lambda^{1/2}D}.
\end{align*}
As a result, and since $\inner{x_t,\theta^*}\in\brs*{-1,1}$, the CI at the beginning of each round is 
\begin{align*}
    \inner{x_t,\theta^*}\in\brs*{\max\brc*{\min_{\theta\in C_{t-1}} \inner{x_t,\theta},-1} , \min\brc*{\max_{\theta\in C_{t-1}} \inner{x_t,\theta},1}},
\end{align*}
and its width is upper bounded by
\begin{align*}
    \min\brc*{\max_{\theta\in C_{t-1}} \inner{x_t,\theta} - \min_{\theta\in C_{t-1}} \inner{x_t,\theta},2} 
    &= \min\brc*{\max_{\theta\in C_{t-1}} \inner{x_t,\theta-\hat\theta_t} - \min_{\theta\in C_{t-1}} \inner{x_t,\theta-\hat\theta_t},2} \\
    &= 2\min\brc*{l_{t-1}\norm{x_t}_{V_{t-1}^{-1}},1} \tag{By \Cref{lemma: solution of inner production uncertaintly}}\\
    & \le 2l_{t-1}\min\brc*{\norm{x_t}_{V_{t-1}^{-1}},1} \tag{$l_t\ge1$}\\
    &\triangleq CI_t(x_t).
\end{align*}
Finally, we say that the CBM condition queries reward if $CI_t(x_t) \ge \frac{l_{t-1}v_{B(t)}}{2\sqrt{B(t)}}$ for
\begin{align*}
    v_t =  \sqrt{2d\log\br*{1+\frac{t L^2}{d\lambda}}}.
\end{align*}

We now prove the regret bound of \Cref{theorem: CBM Linear Bandits} when setting $\lambda = \max\brc*{D^{-1/2},1}$:
\begin{theorem}[Confidence Budget Matching for Linear Bandits]\label{theorem appendix: CBM Linear Bandits}
For any adaptive adversarially chosen sequence of non-decreasing budget and context set $\brc{B(t), u_t}_{t\geq 1}$ the  regret of CBM-OFUL is upper bounded by 
\begin{align*}
    \Regret(T) &\leq 2 l_Tv_T \br*{\sum_{t=1}^T \frac{1}{\sqrt{B(t)}} + \sqrt{T}}  \\
    & = \Ocal\br*{\br*{d\sigma +\sqrt{d\lambda }D}\log\br*{\frac{1+ t L^2/\lambda}{\delta}}\br*{\sum_{t=1}^T \frac{1}{\sqrt{B(t)}} + \sqrt{T}}}.
\end{align*}
for any $T\geq 1$ with probability greater than $1-\delta$, where the $\Ocal$-notation holds when $d\sigma +\sqrt{d\lambda }D\ge1$.
\end{theorem}
\begin{proof}
With probability greater than $1-\delta$ it holds that $\theta^* \in C_t = \brc*{\theta\in \mathbb{R}^d:  \norm{\hat{\theta}_t-\theta}_{V_t}\leq l_t }$, uniformly for all $t\ge0$ (see Theorem~\ref{theorem: abassi confidence interval} which generalizes \citep[][Theorem 2]{abbasi2011improved} . We define this event as the good event and denote it by $\G$. Also, we denote the parameter vector that maximizes the UCB by $\tilde\theta_t$, i.e., $(x_t,\tilde\theta_t)\in\arg\max_{x\times\theta\in \mathcal{X}_t\times C_{t-1}} \inner{x,\theta}$. Then, under the good event we have that $\max_{x\in\mathcal{X}_t}\inner{x,\theta^*}\le \inner{x_t,\tilde\theta_t}$ (`optimism').

As in the MAB setting (\Cref{theorem: CBM Bandits}), we prove that when CBM-OFUM asks for a query, it always has sufficient budget (\Cref{lemma: cbm oful budget constraint is satisfied}). Therefore, throughout the proof, we assume that the algorithm observes reward iff it sets $q_t=1$ and that $\Bq(t)\le \min\brc*{T,B(t)}$ for all $t$.

We now derive the performance bound by bounding the regret of the algorithm. The following relations hold for any $T\geq 1$:
\begin{align*}
    \Regret(T) = \underbrace{\sum_{t=1}^T \indicator{q_t = 1} \br*{\max_{x\in\mathcal{X}_t}\inner{x,\theta^*} - \inner{x_t,\theta^*}}}_{(i)} + \underbrace{\sum_{t=1}^T \indicator{q_t = 0} \br*{\max_{x\in\mathcal{X}_t}\inner{x,\theta^*} - \inner{x_t,\theta^*}}}_{(ii)}.
\end{align*}
Terms $(i)$ and $(ii)$ represent the regret over episodes in which feedback was queried and not queried, respectively. We bound each of the terms.

{\bf Bound on term~$(i)$, episodes in which feedback is queried.} The following relations hold.
\begin{align*}
    (i)=&\sum_{t=1}^T \indicator{q_t = 1} \br*{\max_{x\in\mathcal{X}_t}\inner{x,\theta^*} - \inner{x_t,\theta^*}}  \\
    &=\sum_{t=1}^T \indicator{q_t = 1} \min\brc*{\br*{\max_{x\in\mathcal{X}_t}\inner{x,\theta^*} - \inner{x_t,\theta^*}},2} \tag{Bounded expected reward}  \\
    &\leq \sum_{t=1}^T \indicator{q_t = 1}  \min\brc*{\br*{\inner{x_t,\tilde\theta_t} - \inner{x_t,\theta^*}},2} \tag{Optimism}  \\
    &= \sum_{t=1}^T \indicator{q_t = 1}  \min\brc*{\br*{\inner{x_t,\tilde\theta_t-\hat\theta_t} + \inner{x_t,\hat\theta_t - \theta^*}},2} \\
    &=\sum_{t=1}^T \indicator{q_t = 1} \min\brc*{\norm{x_t}_{V_{t-1}^{-1}}\br*{\norm{\tilde\theta_t-\hat\theta_t}_{V_{t-1}}+\norm{\hat\theta_t-\theta^*}_{V_{t-1}}},2}  \tag{Cauchy-Scwartz inequality}  \\
    &\leq 2l_T\sum_{t=1}^T \indicator{q_t = 1} \min\brc*{\norm{x_t}_{V_{t-1}^{-1}},1}  \tag{Conditioned on $\G$ \& $l_t\ge1$ is increasing in $t$}
\end{align*}
Next, denote the $k^{th}$ time that CBM-OFUL queried a reward by $\tau_k$. Then, we can write $V_t = \lambda I_d+\sum_{k=1}^{\Bq(t)}x_{\tau_k}x_{\tau_k}^T$. Rewriting the bound on $(i)$ and then using the the elliptical potential lemma (\Cref{lemma: eliptical potential lemma abbasi}) and get
\begin{align*}
    (i)&\leq 2l_T\sum_{k=1}^{\Bq(T)} \min\brc*{\norm{x_{\tau_k}}_{V_{\tau_k-1}^{-1}},1} \\
    & \le 2l_T\sqrt{\Bq(t)}\sqrt{\sum_{k=1}^{\Bq(T)} \min\brc*{\norm{x_{\tau_k}}_{V_{\tau_k-1}^{-1}}^2,1}} \tag{$\br*{\sum_{i=1}^n x_i}^2 \le n\sum_{i=1}^n x_i^2$}\\
    &\le 2l_T\sqrt{\Bq(T)}\sqrt{2 d\log\br*{1 + \frac{\Bq(T)L^2}{d\lambda}}} \tag{By \Cref{lemma: eliptical potential lemma abbasi}}\\
    & \quad\,\,\le 2l_T\sqrt{T}\sqrt{ 2d\log\br*{1 + \frac{T L^2}{d\lambda}}}\enspace.
\end{align*}

{\bf Bound on term~$(ii)$, episodes in which feedback is not queried.} The following relations hold.
\begin{align*}
    (ii)=&\sum_{t=1}^T \indicator{q_t = 0} \br*{\max_{x\in\mathcal{X}_t}\inner{x,\theta^*} - \inner{x_t,\theta^*}} \\
    &\leq \sum_{t=1}^T \indicator{q_t = 0}\min \brc*{\inner{x_t,\tilde\theta_t} - \inner{x_t,\theta^*},2} \tag{Optimism}\\
    &\leq \sum_{t=1}^T \indicator{q_t = 0} \underbrace{\min\brc*{ \max_{\theta\in C_t}\inner{x_t,\theta} - \min_{\theta\in C_t}\inner{x_t,\theta},2}}_{\le CI_t(x_t)}\\
    &\leq \sum_{t=1}^T \indicator{q_t = 0} \frac{2l_{t-1}v_{B(t)}}{\sqrt{B(t)}} \tag{By update rule}\\
    &\leq 2l_Tv_T\sum_{t=1}^T \frac{1}{\sqrt{B(t)}} \tag{$l_t,v_t$ are increasing in $t$}.
\end{align*}

The forth relation holds since the algorithm does not query for feedback \emph{only} when $CI_t(x_t)\leq 2l_{t-1} v_{B(t)}/\sqrt{B(t)}$ (a case where $q_t=0$ since the algorithm ran out of budget cannot happen).

{\bf Combining the bounds.} Combining the bounds over $(i)$ and $(ii)$ we get that conditioned on the good event (which holds with probability greater than $1-\delta$), the regret is bounded by
\begin{align*}
    \Regret(T) &\leq 2l_Tv_T\sum_{t=1}^T \frac{1}{\sqrt{B(t)}} + 2l_T\sqrt{T} \underbrace{\sqrt{2d\log\br*{1 + \frac{TL^2}{d\lambda}}}}_{\eqdef v_T} \\
    & = 2l_Tv_T\br*{\sum_{t=1}^T \frac{1}{\sqrt{B(t)}} + \sqrt{T}},
\end{align*}
for all $T\geq 1$.
\end{proof}

\clearpage

\subsection{CBM-OFUL: Budget Constraint is Satisfied}
\begin{lemma}[CBM-OFUL: Budget Constraint is Satisfied]\label{lemma: cbm oful budget constraint is satisfied}
Conditioned on the good event, for any $T\geq 1$ the budget constraint is not violated, $\Bq(T)\leq B(T)$.
\end{lemma}
\begin{proof}
Before supplying the proof, notice that the function $f(x) = x/\log(1+x\alpha)$ is a strictly increasing function in $\alpha>0,x\ge 0$ (since $f'(x)>0$ for $x>0$). This implies that the ratio 
$$\frac{B(T)}{v_{B(T)}^2} = \frac{B(T)}{4d\log(1+ B(T)L^2/\lambda)}=\frac{f(B(T))}{4d}$$
is a non-decreasing function in $T$ since $B(T)$ is non-decreasing in $T$. Also, recall that $CI_t(x_t) =2 l_t \min\brc*{\norm{x_t}_{V_t^{-1}},1}$. Using these facts we get the following relations for any $T\geq 1$.
\begin{align*}
    &\Bq(T) = \sum_{t=1}^T \indicator{q_t=1} \\
    &\leq \sum_{t=1}^T \indicator{q_t=1} \frac{CI_t(x_t)}{2l_{t-1} v_{B(t)}\sqrt{1/B(t)}} \tag{Update rule}\\
    &=\sum_{t=1}^T \indicator{q_t=1} \frac{2 l_{t-1} \min\brc*{\norm{x_t}_{V_{t-1}^{-1}},1}}{2l_{t-1} v_{B(t)}\sqrt{1/B(t)}} \tag{Definition of $CI_t(x_t)$}\\
    & \leq  \sqrt{B(T)/v_{B(T)}} \sum_{t=1}^T \indicator{q_t=1}\min\brc*{\norm{x_t}_{V_{t-1}^{-1}},1}\tag{$B(T)/v_{B(T)}^2$ is non-decreasing in $T$}\\
    &\leq \sqrt{B(T)/v_{B(T)}}\sqrt{\sum_{t=1}^{T} \indicator{q_t=1}}\sqrt{\sum_{t=1}^{T} \indicator{q_t=1}\min\brc*{\norm{x_t}_{V_{t-1}^{-1}}^2,1}} \tag{Cauchy-Schwartz Inequality}\\
    &\leq \sqrt{B(T)/v_{B(T)}}\sqrt{\Bq(T)}\sqrt{2d\log\br*{1+\frac{\Bq(T)L^2}{d}}} \tag{Lemma~\ref{lemma: eliptical potential lemma abbasi}}\\
    &= \sqrt{\frac{B(T)}{\log\br*{1+\frac{B(T)L^2}{d\lambda}}}}\sqrt{\Bq(T)\log\br*{1+\frac{\Bq(T)L^2}{d}}}.
\end{align*}
Rearranging we get 
\begin{align}
    \frac{\Bq(T)}{\log\br*{1+\Bq(T)L^2/\lambda}} \leq \frac{B(T)}{\log\br*{1+B(T)L^2/\lambda}}. \label{eq: budget linear bandit pre final}
\end{align}
Since $f(x) = x/\log(1+x\alpha)$ is strictly increasing function for any $\alpha>0,x\geq 0$,~\eqref{eq: budget linear bandit pre final} implies that $\Bq(T)\leq B(T)$.
\end{proof}

\newpage
\section{Confidence Budget Matching for Reinforcement Learning: CBM-UCBVI} \label{appendix: cbm-ucbvi RL}
We start by defining the feedback model of RL with budget constraints and by introducing some notations. At the beginning of each round $t$, the learner acts with a non-stationary policy $\pi:\Scal\times [H] \rightarrow\Acal$ for $h$ time steps. The learner observes the trajectory of the $t^{th}$ episode, $\brc*{(s_{t,h},a_{t,h})}_{h=1}^H$, and an adaptive adversary reveals to the learned the budget $B(t)$. Then, the learner is allowed to query for reward feedback on the states observed at the $t^{th}$ episode. That is, the learner can ask for a noisy version of the reward  $R_{t,h}(s,a)= r_h(s,a) +\eta_t$ where $\E[R_{t,h}(s,a)]= r_h(s,a)$ and $R_{t,h}(s,a)\in [0,1]$ a.s. as long as (i) $(s,a)$ was observed at the $h^{th}$ time step at the $t^{th}$ episode, i.e., $s_{t,h}=s,a_{t,h}=a$, and, (ii) there is a spare budget. In this section, we assume that querying rewards incur a unit costs and we denote $q_{t,h}=1$ as the event reward feedback is queried in the $t^{th}$ episode at the $h^{th}$ time step (which implies there is an available budget for this event to occur). Furthermore, we define the set
\begin{align}
    \LR \eqdef \brc{(s,a,h)\in \Scal \times \Acal\times[H]: r_h(s,a)>0}, \label{eq: definition of LR}
\end{align}
that is, the set of state, action and time step tuples such that the reward is not zero\footnote{A bit more generally, our results holds would we define the set of $(s,a,h)$ with deterministic reward, $\brc{(s,a,h)\in \Scal \times \Acal\times[H]: \VAR(R_h(s,a))=0}$. However, since sparse reward is a more natural measure, we chose to work with the set that defined in~\eqref{eq: definition of LR}.}.

The CBM-UCBVI algorithm (Algorithm~\ref{alg: budget RL CBM UCBVI}) combines the CBM principle into the UCBVI algorithm~\citep{azar2017minimax}. Similarly to UCBVI, CBM-UCBVI solves an optimistic MDP, $\mathcal{M}_t=(\Scal,\Acal,\bar{r}_{t-1}+ b^r_{t} + b^{p}_{t},\bar{P}_{t-1},H)$, at the beginning of each round, where $\bar{r}_{t-1},\bar{P}_{t-1}$ are the empirical reward and transition model and $ b^r_{t} + b^{p}_{t}$ are bonus terms (all defined below). CBM-UCBVI interacts with the environment with the optimal policy of $\mathcal{M}_t$ and samples a trajectory of state-action pairs $\brc*{(s_{t,h},a_{t,h})}_{h=1}^H$. Then, instead of receiving the reward feedback on all the state-action pairs within the trajectory, CBM-UCBVI utilizes the CBM principle to decide in which state-action pairs reward feedback is queried. Specifically, it queries for reward feedback at $(s_{t,h},a_{t,h})$ if 
$$
2b^r_{t,h}(s_{t,h},a_{t,h})
\triangleq CI_{t,h}^r(s_{t,h},a_{t,h}) \geq 2L_{t,\delta}\br*{\sqrt{\frac{|\LR|}{B(t)}} + \frac{SAH\log(1+B(t))}{B(t)}}.
$$
where $L_{t,\delta}\eqdef  \log \br*{\frac{12S^2AH t^2(t+1)}{\delta}}$. Indeed, Lemma~\ref{lemma: cbm-ucbvi budget constraint is satisfied} establishes that this query rule does not violate the budget constraint. Thus, if reward feedback is queried, there is an available budget. 

Let $n_{t,h}(s,a)= \sum_{k=1}^{t} \indicator{s_{k,h}=s,a_{k,h}=a}$ be the number of times $(s,a)$ was sampled at the $h^{th}$ time step until the end of the $t^{th}$ episode, and, $n^q_{t,h}(s,a)= \sum_{k=1}^{t} \indicator{s_{k,h}=s,a_{k,h}=a,q_{k,h}=1}$ be the number of times the learner queried for reward feedback in $(s,a)$ at the $h^{th}$ time step until the end of the $t^{th}$ episode. We denote by $\bar{r}_t,\bar{P}_t$ the empirical reward and empirical transition model, that is
\begin{align*}
    &\bar{P}_{t,h}(s'|s,a) = \frac{1 }{n_{t,h}(s,a) \vee 1} \sum_{k=1}^t \indicator{s_{k,h+1}=s',s_{k,h}=s,a_{k,h}=a},\\
    &\bar{r}_{t,h} = \frac{1}{n^q_{t-1,h}(s,a)\vee 1} \sum_{k=1}^t \indicator{s_{k,h}=s,a_{k,h}=a, q_{k,h}=1} R_{t,h}(s,a).
\end{align*}
Observe that, unlike in classic RL, $n_{t,h}(s,a) \neq n^q_{t,h}(s,a)$; the number of times \emph{reward feedback} was queried in $(s,a)$ for the $h^{th}$ time step is not equal to the number of times $(s,a)$ was \emph{visited} at the $h^{th}$ time step. Lastly, the bonus terms $b_t^r,\ b_t^p$ used by CBM-UCBVI are given as follows.
\begin{align*}
    &b_{t,h}^r(s,a) \eqdef \sqrt{ \frac{2 \widehat{\mathrm{Var}}_{R,t-1,h}(s,a)L_{t,\delta} }{n^q_{t-1,h}(s,a)\vee 1}} +  \frac{5L_{t,\delta}}{ n^q_{t-1,h}(s,a)\vee 1},\\
    &b_{t,h}^p(s,a) \eqdef  \sqrt{\frac{2H^2L_{t,\delta}}{n_{t-1,h}(s,a)\vee 1}} + \frac{5H L_{t,\delta}}{n_{t-1,h}(s,a)},
\end{align*}
where $L_{t,\delta}\eqdef  \log \br*{\frac{12S^2AH t^2(t+1)}{\delta}}$ and 
\begin{align*}
  \widehat{\mathrm{Var}}_{R,t-1,h}(s,a)= 
  \begin{cases}
            \frac{\sum_{k,k'=1}^{n^q_{t-1,h}(s,a)}\br*{R_{k',h}(s,a) -R_{k,h}(s,a)}^2}{2(n^q_{t-1,h}(s,a)(n^q_{t-1,h}(s,a)-1)\vee 1)}  & n^q_{t-1,h}(s,a)\geq 2 \\
            0 & o.w.
\end{cases}
\end{align*}
is the (unbiased) empirical estimate of the reward's variance.

\begin{algorithm}[t]
\caption{CBM-UCBVI} \label{alg: budget RL CBM UCBVI}
\begin{algorithmic}
\STATE {\bf Require:} $\delta\in(0,1)$
\FOR{$t=1,2,...$}
    \STATE $\mathcal{M}_t=(\Scal,\Acal,\bar{r}_{t-1}+ b^r_{t} + b^{p}_{t},\bar{P}_{t-1},H)$
    \STATE Get optimistic Q-functions $\brc{\bar{Q}_{t,h}(s,a)}_{(s,a)\in \Scal\times\Acal,h\in [H]}$ via truncated value iteration on $\mathcal{M}_t$ (\Cref{alg: trunctated value iteration})
    \STATE Act with $\pi_{t,h}=\arg\max_a \bar{Q}_{t,h}(s,a)$ and observe a trajectory $\brc{(s_{t,h},a_{t,h})}_{h=1}^H$
    \STATE Observe current budget $B(t)$
    \STATE Ask for feedback on $(s_{t,h},a_{t,h})\in \brc{(s_{t,h},a_{t,h})}_{h=1}^H$ if
    \begin{align*}
         CI_{t,h}^R(s_{t,h},a_{t,h}) \geq L_{t,\delta}\br*{6\sqrt{\frac{|\LR|}{B(t)}} + 4SAH\frac{\log(1+B(t))+1}{B(t)}}.
    \end{align*}
\ENDFOR
\end{algorithmic}
\end{algorithm}
\begin{algorithm}[t]
\caption{Truncated Value Iteration} \label{alg: trunctated value iteration}
\begin{algorithmic}
\STATE {\bf Require:} An MDP $\mathcal{M}=(\Scal, \Acal,r,P,H)$.
\STATE {\bf Initialize:} $V_{H+1}(s)=0$ for all $s\in \Scal$.
\FOR{$h=H,H-1,..,1$}
    \STATE For all $(s,a)\in \Scal\times \Acal,\ Q_h(s,a) =  r_h(s,a) + \E_{P_h}[V_{h+1}(s')|s_h=s,a_h=a]$
    \STATE For all $s\in \Scal$,  $V_{h}(s) = \min\brc*{\max_{a\in \Acal} Q_h(s,a),H-h}$
\ENDFOR
\STATE {\bf Return:} $\brc{Q(s,a)}_{(s,a)\in \Scal\times\Acal,h\in [H]}$
\end{algorithmic}
\end{algorithm}

\paragraph{Notations for the Proof}
Let  $f: \Scal\rightarrow \mathbb{R}$ and $P(s'|s,a)$ be a transition model. We use the following notation $$\E_{P_h(\cdot|s,a)}[f(s')] = \sum_{s'} P_h(s'|s,a) f(s').$$
Furthermore, we define $F_{t,h-1}$ as the $\sigma$-algebra generated by all the event until the $h^{th}$ time step within the $t^{th}$ episode; namely, $s_{t,h}$ and $a_{t,h}$ are $F_{t,h-1}$-measurable, (but not $R_{t,h}$, which we assume that is generated by the end of the episode). Thus, this definition implies that $F_{t,0}=F_{t-1}$ (where $F_{t-1}$ is defined in Section~\ref{section: prelimineries}). Observe that this definition also implies that
\begin{align}
    \E_{P_{h}(\cdot|s_{t,h},a_{t,h})}[f(s')] = \E[f(s_{t,h+1}) | F_{t,h-1}], \label{eq: relation between next state distribution and conditionig on F_{t,h-1}}
\end{align}
for $f$ which is an $F_{t,h-1}$ measurable function (i.e., that depends on the history until time step $h$ at the $t^{th}$ episode). 

We are now ready to state the central result of this section which gives a performance guarantee on the regret of CBM-UCBVI.
\begin{theorem}[CBM-UCBVI]\label{theorem appendix: CBM UCBVI}
For any adversarially adaptive sequence $\brc*{B(t),s_{t,1}}_{t\geq 1}$ of budget and initial states the  regret of CBM-ULCBVI is upper bounded by 
\begin{align*}
\Regret(T) \leq 36  L_{T,\delta} \br*{\sqrt{SAH^4T} + \sum_{t=1}^T\br*{\sqrt{\frac{|\LR|H^2}{B(t)}} + SAH^2\frac{\log(1+ B(t))+1}{B(t)}}} + 306 H^3S^2A L^2_{T,\delta},
\end{align*}
for any $T\geq 1$ with probability greater than $1-\delta$.
\end{theorem}
To establish to proof of this result we prove some preliminary results. We first define the set of good events, which hold with high probability (Section~\ref{appendix: ucbvi the first good event} and Section~\ref{appendix: ucbvi the second good event}), and establish the optimism of CBM-UCBVI (Section~\ref{appendix: optimsim ucbvi}). Furthermore, in Section~\ref{appendix: budget constraint is satisfied CBM-UCBVI}, we establish that CBM-UCBVI does not violate the budget constraint. Later, this important property allows us to bound the reward bonus when reward feedback is not queried (see \Cref{lemma: bound on cummulative reward bonus RL}). Given these tools, we prove Theorem~\ref{theorem appendix: CBM UCBVI} (Section~\ref{appendix: actual proof of CBM UCBVI}) by relaying on a key recursion lemma that bounds the on-policy errors at time step $h$ by the on-policy errors of time step $h+1$.

\subsection{The First Good Event - Concentration Events}\label{appendix: ucbvi the first good event}

\begin{align*}
    &E^r(t) = \brc*{\forall s\in S,a \in A:\ |\bar{r}_{t-1,h}(s,a) -r_h(s,a)|  \leq \sqrt{ \frac{2 \widehat{\mathrm{Var}}_{R,t-1,h}(s,a)\log\frac{12SAH t^2(t+1)}{\delta} }{n^q_{t-1,h}(s,a)\vee 1}} +  \frac{5\log \frac{12SAH t^2(t+1)}{\delta}}{ n^q_{t-1,h}(s,a)\vee 1}} \\
    &E^p(t) = \brc*{\forall s,s'\in S, a\in A:\ |P_h\br*{s'|s,a} - \bar{P}_{t,h}\br*{s'|s,a}| \le \sqrt{\frac{2P(s'|s,a)\log\frac{12S^2AH t^2(t+1)}{\delta}}{n_{t-1,h}(s,a)\vee 1}} + \frac{2\log \frac{12S^2AH t^2(t+1)}{\delta}}{n_{t-1,h}(s,a)\vee 1}} \\
    &E^{pv}(t)=\brc*{\forall s,a,h:\ \abs*{\br*{\bar{P}_{t,h}(\cdot \mid s,a)-P_h(\cdot \mid s,a)}^T V_{h+1}^*} \leq \sqrt{\frac{2H^2\log \frac{12SAH t^2(t+1)}{\delta}}{n_{t-1,h}(s,a)\vee 1}} + \frac{5H\log \frac{12SAH t^2(t+1)}{\delta}}{n_{t-1,h}(s,a)}}\\
\end{align*}

Notice that the bonus $b_{t,h}^{pv}(s,a)$ depends on the number of times we visited the state $s$ and took action $a$, denoted by $n$, whereas the bonus $b_{t,h}^r(s,a)$ depends on the number of times we queried a reward, denoted by $n^q$. Proving this set of events hold jointly is standard, based upon the empirical Bernstein concentration bound~\citep{maurer2009empirical}.

\begin{lemma}[The First Good Event]\label{lemma: the first good event RL}
Let $\G_1 = \cap_{t\geq 1} E^r(t) \cap_{t\geq 1} E^p(t) \cap_{t\geq 1} E^{pv}(t)$ be the good event. It holds that $\Pr\br*{\G_1}\geq 1-\delta/2$.
\end{lemma}
\begin{proof}

We prove that each part of the good event holds with a probability of at least $1-\delta/6$.

{\bf  The event $\cap_{t\ge1} E^r(t)$ holds with high probability.} Fix an episode $t\geq 1$ and $s,a,h\in \mathcal{S}\times \mathcal{A}\times [H]$. There are at most $t$ reward samples from $(s,a)$ at time step $h$ at episode $t$. Taking a union bound over these possible values and scaling $\delta\rightarrow \delta/t$ we get that
\begin{align*}
    \Pr\br*{|\bar{r}_{t-1,h}(s,a) -r_h(s,a)|  \leq \sqrt{ \frac{2 \widehat{\mathrm{Var}}_{R,t-1,h}(s,a)\log\frac{2t}{\delta} }{n^q_{t-1,h}(s,a)\vee 1}} +  \frac{14\log \frac{2t}{\delta}}{3 n^q_{t-1,h}(s,a)\vee 1}  } \geq 1-\delta,
\end{align*}
by~\citep[][Theorem 4]{maurer2009empirical}. Note that we used the relation $1/(n-1)\leq 2/n$ for $n\geq 2$ and that the bound trivially holds for $n<2$ since the reward is in $[0,1]$. Furthermore, taking a union bound on all $s,a,h\in  \mathcal{S}\times \mathcal{A} \times [H]$ and setting $\delta \rightarrow \delta/6SAHt(t+1)$ results in $\Pr\br*{\cap_{t\ge1} E^r(t)} \leq 1-\delta/6$ by using $\sum_{t=1}^\infty \frac{\delta}{t(t+1)}=\delta$.

{\bf  The event $\cap_{t\ge1} E^p(t)$ holds with high probability.} Fix an episode $t\geq 1$ and $s,a,s',h\in \mathcal{S}\times \mathcal{A} \times \mathcal{S} \times [H]$. At the $t^{th}$ episode the tuple can be sampled for at most $t$ times. Taking a union bound  on all these possible values and applying Bennet's inequality \citep[\eg][Theorem 3]{maurer2009empirical}, we get
\begin{align*}
    \Pr\br*{|P_h\br*{s'|s,a} - \bar{P}_{t,h}\br*{s'|s,a}| \le \sqrt{\frac{2P(s'|s,a))\log\frac{2t}{\delta}}{n_{t-1}(s,a)\vee 1}} + \frac{4\log\frac{2t}{\delta}}{3n_{t-1}^p(s,a)\vee 1}} \geq 1-\delta,
\end{align*}
since the variance of a Bernoulli random variable is $P(s'|s,a)(1-P(s'|s,a)\leq P(s'|s,a)$. Taking a union bound over all $s,a,s',h\in \mathcal{S}\times \mathcal{A} \times \mathcal{S} \times [H]$ and scaling $\delta \rightarrow \delta/6S^2AHt(t+1)$ and repeating the same reasoning as before concludes the high probability bound for $\cap_{t\ge1} E^p(t)$.

{\bf  The event $\cap_{t\ge1} E^{pv}(t)$ holds with high probability.} Repeating the same arguments as for the event $\cap E^r(t)$ while noticing that $\bar{P}_{t,h}(\cdot \mid s,a)^TV_{h+1}^*\in\brs*{0,H}$ (and, thus, also the empirical variance, for all $n\ge2$) concludes the proof.

{\bf Combining the results.} Taking a union bound concludes the proof.
\end{proof}

\subsection{Optimism}\label{appendix: optimsim ucbvi}
We can prove that the value is optimistic using standard techniques. 
\begin{lemma}[Value Function is Optimistic] \label{lemma: optimism cbm-ucbvi}
Conditioning on the first good event $\G_1$ the value function of CBM-UCBVI is optimism for all $s\in \mathcal{S},h\in [H],t\geq 1$, i.e.,
\begin{align*}
    \forall s\in \Scal,h\in[H], t\ge1:\ V^*_h(s)\leq \bar{V}_{t,h}(s).
\end{align*}
\end{lemma}
\begin{proof}
We prove this result via induction.

{\bf Base case, $h=H$ and for all $s\in \mathcal{S}$.} Let $a^*(s)\in  \arg\max_{a\in \Acal} r_{H}(s,a)$. For any $s\in \Scal$ it holds that
\begin{align}
    &V^*_{H}(s) -\bar{V}_{t,H}(s) = r_{H}(s,a^*(s)) - \min\brc*{\max_a \brc*{\bar{r}_{t,H}(s,a) + b_{t,H}^r(s,a)},1}. \label{eq: optimism base case ucbvi rel 1}
\end{align}
Assume that $\max_a \bar{r}_{t,H}(s,a) + b_{t,H}^r(s,a) < 1$ then,
\begin{align*}
    &\eqref{eq: optimism base case ucbvi rel 1} \leq  r_{H}(s,a^*(s)) - \bar{r}_{t,H}(s,a^*(s)) - b_{t,H}^r(s,a^*(s))\\
    &\leq b_{t,H}^r(s,a^*(s)) - b_{t,H}^r(s,a^*(s))\leq 0 \tag{Event $\cap_t E^r(t)$ holds}.
\end{align*}
If $\max_a \brc*{\bar{r}_{t,H}(s,a) + b_{t,H}^r(s,a)}\geq 1$ then trivially $\eqref{eq: optimism base case ucbvi rel 1}\leq 0$ since $r_{H}(s,a)\in [0,1]$ for all $a\in \Acal$. Overall, we conclude that $V^*_{H}(s) \leq \bar{V}_{t,H}(s)$ for all $s\in \Scal$ for $h=H$.

{\bf Induction step, for $h\in [H]$ $s\in \mathcal{S}$ assuming it holds for all $h'\geq h+1$.} Let $a^*(s)\in \arg\max_{a\in \Acal} Q^{*}_h(s,a)$ The following relations hold.
\begin{align}
     &V^*_h(s)- \bar{V}_{t,h}(s) =   Q^{*}_h(s,a^*(s))- \min\brc*{\max_a\bar{Q}_{t,h}(s,a),H-h} \label{eq: optimism base case ucbvi rel 2}
\end{align}
Assume that $\max_a\bar{Q}^{\pi}_{t,h}(s,a)< H-h$, then,
\begin{align*}
    \eqref{eq: optimism base case ucbvi rel 2} &\leq Q^{*}_h(s,a^*(s)) - \bar{Q}_{t,h}(s,a^*(s))\\
    & = r_h(s,a^*(s)) + P_h(\cdot| s,a^*(s))^TV^*_{h+1} \\
    &\quad- \br*{\bar{r}_{t-1,h}(s,a^*(s)) + b_{t,h}^r(s,a^*(s)) + b_{t,h}^p(s,a^*(s)) + \bar{P}_h(\cdot| s,a^*(s))^T\bar{V}_{t,h+1}}\\
    &= r_h(s,a^*(s)) - \bar{r}_{t-1,h}(s,a^*(s)) - b_{t,h}^r(s,a^*(s)) 
    + (P_h-\bar{P}_{t-1,h})(\cdot| s,a^*(s))^TV^*_{h+1}  - b_{t,h}^p(s,a^*(s)) \\
    &\quad+ \E_{\bar{P}_{t-1,h}(\cdot|s,a^*(s))}[ \underbrace{V^{*}_{h+1}(s') - \bar{V}_{t,h+1}(s')}_{\leq 0\ \mathrm{Induction\ hypothesis}}]\\
    &\leq  b_{t,h}^r(s,a^*(s)) - b_{t,h}^r(s,a^*(s)) - b_{t,h}^p(s,a^*(s)) +  b_{t,h}^p(s,a^*(s)) \tag{$\cap_t E^{pv}(t)\cup\cap_t E^{r}(t)$ holds}\\
    &=0. 
\end{align*}
If  $\max_a\bar{Q}^{\pi}_{t,h}(s,a)\geq H-h$ then trivially $\eqref{eq: optimism base case ucbvi rel 2}\leq 0$ since $Q^*_h(s,a)\leq H-h$ since it is an expectation over $H-h$ terms, each bounded by $[0,1]$. Thus, $V^*_{h}(s) \leq \bar{V}_{t,h}(s)$ for all $s\in \Scal$ for $h'\geq h$ which proves the induction step and concludes the proof.
\end{proof}
\clearpage

\subsection{The Second Good Event - Optimism Bound}\label{appendix: ucbvi the second good event}
We now prove a high probability bound which holds conditioned on the first good event $\G_1$.
\begin{lemma}[The Second Good Event]\label{lemma: the good event cbm rl}
Let $\G_1$ be the event defined in Lemma~\ref{lemma: the first good event RL}.Let $\brc*{Y_{t ,h}}_{t\geq 1}$ the random variables defined as
\begin{align*}
    Y_{t ,h} \eqdef \bar{V}_{t,h+1}(s_{t,h+1}) - V^{\pi_t}_{h+1}(s_{t,h+1})\\
\end{align*}
The second good event is defined as $\G_2 =E^{O} $ where 
\begin{align*}
    E^{O}=\brc*{\forall h\in[H-1], T\geq 1:\ \sum_{t=1}^T \E[Y_{t ,h}|F_{t,h-1}]\leq \br*{1+\frac{1}{2H}} \sum_{t=1}^T Y_{t ,h} + 18H^2 \log\frac{2HT(T+1)}{\delta}}
\end{align*}
Then, the good event $\G = \G_1\cap \G_2$ holds with probability greater than $1-\delta$.
\end{lemma}
\begin{proof}

Fix $h\in [H],T$. We start by defining the random variable $W_t=\indicator{\bar{V}_{t,h+1}(s) - V^{\pi_t}_{h}(s)\ge0, \forall h\in\brs*{H}.s\in\Scal}$ and the random process $Y_t = Y_{t ,h}= \bar{V}_{t,h+1}(s_{t,h+1}) - V^{\pi_t}_{h+1}(s_{t,h+1})$ w.r.t. the filtration $\brc*{F_{t+1,h-1}}_{t\geq 1}$ (observe that $Y_{t ,h}$ is $F_{t+1,h-1}$ measurable). Also notice that $W_t$ is $F_{t,h-1}$ measurable, as both $\pi_t$ and $\bar{V}_{t,h}$ are $F_{t-1}$-measurable. Finally, define $\tilde{Y}_t = W_tY_t$. Importantly, $\tilde{Y}_t\in\brs*{0,H}$ almost surely, by definition of $W_t$ and since $\bar{V}_{t,h+1}(s),V^{\pi_t}_{h}(s)\in [0,H]$ by the update rule. Thus, using \Cref{lemma: consequences of optimism and freedman's inequality} with $C=H\ge1$, we get
\begin{align*}
    \sum_{t=1}^T \E[\tilde{Y}_{t ,h}|F_{t,h-1}]\leq \br*{1+\frac{1}{2H}} \sum_{t=1}^T \tilde{Y}_{t ,h}  + 18H^2 \log\frac{1}{\delta},
\end{align*}
with probability greater than $1-\delta$, and since $W_t$ is $F_{t,h-1}$-measurable, we can write
\begin{align}
    \sum_{t=1}^T W_t\E[Y_{t ,h}|F_{t,h-1}]\leq \br*{1+\frac{1}{2H}} \sum_{t=1}^TW_tY_{t ,h}  + 18H^2 \log\frac{1}{\delta}. \label{eq: good event relation 1}
\end{align}
Importantly, notice that under $\G_1$, it holds that $W_t\equiv1$ (by \Cref{lemma: optimism cbm-ucbvi}). Therefore, applying the union bound and setting $\delta= \delta/2HT(T+1)$ we get 
\begin{align*}
    &\Pr(\overline{E^{O}}\cap \G_1)\\
    &\leq \sum_{h=1}^H\sum_{T=1}^\infty \Pr\br*{\brc*{\sum_{t=1}^T \E[Y_{t ,h} |F_{t,h-1}]\geq \br*{1+\frac{1}{2H}} \sum_{t=1}^T Y_{t ,h} + 18H^2 \log\frac{2HT(T+1)}{\delta}}\cap\G_1} \tag{Union bound}\\
    & = \sum_{h=1}^H\sum_{T=1}^\infty \Pr\br*{\brc*{\sum_{t=1}^T W_t\E[Y_{t ,h} |F_{t,h-1}]\geq \br*{1+\frac{1}{2H}} \sum_{t=1}^T W_tY_{t ,h} + 18H^2 \log\frac{2HT(T+1)}{\delta}}\cap\G_1}\tag{$W_t\equiv1$ under $\G_1$}\\
    & \le \sum_{h=1}^H\sum_{T=1}^\infty \Pr\br*{\sum_{t=1}^T W_t\E[Y_{t ,h} |F_{t,h-1}]\geq \br*{1+\frac{1}{2H}} \sum_{t=1}^T W_tY_{t ,h} + 18H^2 \log\frac{2HT(T+1)}{\delta}}\\
    &\leq \sum_{h=1}^H\sum_{T=1}^\infty \frac{\delta}{2HT(T+1)}=\delta/2, \tag{By~\eqref{eq: good event relation 1}}.
\end{align*}
Finally, we have 
\begin{align*}
    \Pr(\overline{\G}) \le \Pr(\overline{\G_2}\cap\G_1) + \Pr(\overline{\G_1}) \le \frac{\delta}{2} + \frac{\delta}{2} = \delta
\end{align*}
\end{proof}

\clearpage
\subsection{CBM-UCBVI: Budget Constraint is Satisfied}\label{appendix: budget constraint is satisfied CBM-UCBVI}

\begin{lemma}[CBM-UCBVI: Budget Constraint is Satisfied]\label{lemma: cbm-ucbvi budget constraint is satisfied}
For any $T\geq 0$ the budget constraint is not violated, that is $\Bq(T)\leq B(T)$ almost surely.
\end{lemma}
The number of times feedback was queried  after $T$ episodes $\Bq(T)$ is by definition
$
\Bq(T)= \sum_{t=1}^T \sum_{h=1}^H \indicator{q_{t,h}=1}.
$
We prove the lemma by extending the techniques used for CBM-UCB and CBM-OFUL (see Lemma~\ref{lemma: cbm ucb budget constraint is satisfied} and Lemma~\ref{lemma: cbm oful budget constraint is satisfied}). 
\begin{proof}
It holds that
\begin{align}
     \Bq(T)&= \sum_{t=1}^T \sum_{h=1}^H \indicator{q_{t,h}=1} \nonumber \\
     &\leq  \sum_{t=1}^T \sum_{h=1}^H \indicator{q_{t,h}=1}\frac{CI_{t,h}^R(s_{t,h},a_{t,h})}{  L_{t,\delta} \br*{6\sqrt{\frac{|\LR|}{B(t)}} + 4SAH\frac{\log(1+B(t))+1}{B(t)}}}  \nonumber\\
     &\leq  \sum_{t=1}^T \sum_{h=1}^H \indicator{q_{t,h}=1}\frac{\sqrt{ \frac{2 \widehat{\mathrm{Var}}_{R,t-1,h}(s_{t,h},a_{t,h})   }{n^q_{t-1,h}(s_{t,h},a_{t,h})\vee 1}} +  \frac{2 }{n^q_{t-1,h}(s_{t,h},a_{t,h})\vee 1}}{6\sqrt{\frac{|\LR|  }{B(t)}} + 4SAH\frac{\log(1+B(t)) + 1}{B(t)}}  \tag{$\frac{\sqrt{ L_{t,\delta}}}{L_{t,\delta}},\frac{L_{t,\delta}}{L_{t,\delta}} \leq 1$}  \nonumber\\
     &\leq   \br*{6\sqrt{\frac{|\LR|  }{B(T)}} + 4SAH\frac{ \log(1+B(T)) + 1}{B(T)}}^{-1} \cdot \nonumber\\ &\hspace{4cm}\sum_{t=1}^T \sum_{h=1}^H \indicator{q_{t,h}=1}\br*{\sqrt{ \frac{2 \widehat{\mathrm{Var}}_{R,t-1,h}(s_{t,h},a_{t,h})   }{n^q_{t-1,h}(s_{t,h},a_{t,h})\vee 1}} +  \frac{2 }{n^q_{t-1,h}(s_{t,h},a_{t,h})\vee 1}}, \label{eq: RL budget constraint relation 1}
\end{align}
where the last relation holds since the budget is non-decreasing in $T$ and both $\frac{1}{\sqrt{x}},\frac{\log(1+x)+1}{x}$ are monotonically decreasing functions for $x>0$. 
We now upper bound the sum in the last relation. Recall that up to time $t$, a state action pair $(s,a)$  at time step $h$ was queried $n^q_{T,h}(s,a)$ times. Also, notice that $\widehat{\mathrm{Var}}_{R,t-1}(s,a)=0$ if $(s,a,h)\notin \LR$ and $\widehat{\mathrm{Var}}_{R,t-1}(s,a)\leq 1$ if $(s,a,h)\in \LR$. Then, we can write 
\begin{align}
    \sum_{t=1}^T\sum_{h=1}^H  \indicator{q_{t,h}=1}&\br*{\sqrt{ \frac{2 \widehat{\mathrm{Var}}_{R,t-1,h}(s_{t,h},a_{t,h})   }{n^q_{t-1,h}(s_{t,h},a_{t,h})\vee 1}} +  \frac{2 }{n^q_{t-1,h}(s_{t,h},a_{t,h})\vee 1}} \nonumber \\
    &\leq \sum_{s,a,h\in \LR} \sum_{i=0}^{n^q_{T,h}(s,a)}\sqrt{ \frac{2}{i\vee 1}} + \sum_{s,a,h} \sum_{i=0}^{n^q_{T,h}(s,a)}\frac{2 }{i\vee 1} \tag{Bounding $\widehat{\mathrm{Var}}_{R,t-1}(s,a)$}\\
    &\leq 6\sum_{s,a,h\in \LR} \sqrt{n^q_{T,h}(s,a))} +2\sum_{s,a,h}\br*{2+\log(n^q_{T,h}(s,a)) \vee  1)}\nonumber\\
    &\overset{(a)}{\leq} 6\sqrt{|\LR|\Bq(T)} + 4SAH\log(1+\Bq(T)) + 4SAH, \label{eq: RL budget constraint relation 2}
\end{align}
where $(a)$ follows by Jensen's inequality and $\sum_{s,a,h\in \LR} n^q_{T,h}(s,a)) \leq \sum_{s,a,h} n^q_{T,h}(s,a))=\Bq(T)$. Plugging this back into~\eqref{eq: RL budget constraint relation 1} we get
\begin{align}
    &\Bq(T)\leq \eqref{eq: RL budget constraint relation 1}\leq  \br*{6\sqrt{\frac{|\LR|  }{B(T)}} + 4SAH\frac{\log(1+B(T)) + 1}{B(T)}}^{-1} \br*{6\sqrt{|\LR|\Bq(T)} + 4 SAH\log(1+\Bq(T))+4SAH} \nonumber\\
    \iff & 6\sqrt{\frac{|\LR|  }{B(T)}} + 4SAH\frac{\log(1+B(T))+1}{B(T)}\leq  6\sqrt{\frac{|\LR|}{\Bq(T)}} + 4SAH\frac{\log (1+\Bq(T))+1 }{\Bq(T)}. \label{eq: RL budget constraint almost final}
\end{align}

Remember that if $f(x)$ is strictly monotonically decreasing function then $f(x_1)\leq f(x_2) \iff x_1\geq x_2$. Furthermore, see that $f_{\alpha,\beta}(x) = \frac{\alpha}{\sqrt{x}} + \beta\frac{\log(1+x)+1}{x}$ is a strictly monotonically decreasing function for $x>0, \alpha\geq 0,\beta>0$ (since both $\frac{1}{\sqrt{x}}$ and $\frac{\log(1+x)+1}{x}$ are strictly decreasing). We can then write~\eqref{eq: RL budget constraint almost final}, equivalently as
\begin{align*}
    f_{\alpha,\beta}(B(T)) \leq f_{\alpha,\beta}(\Bq(T)),
\end{align*}
for $\alpha = 6|\LR|,\beta = 4SAH$, and since $f_{\alpha,\beta}$ is strictly monotonically decreasing it implies that $\Bq(T)\leq B(T)$.

\end{proof}

\subsection{Proof of Theorem~\ref{theorem appendix: CBM UCBVI}} \label{appendix: actual proof of CBM UCBVI}
Before establishing the proof Theorem~\ref{theorem appendix: CBM UCBVI} we establish the following key lemma that bounds the on-policy errors at time step $h$ by the on-policy errors at time step $h+1$ and additional additive terms.  Given this result, the analysis follows with relative ease.
We are now ready to establish Theorem~\ref{theorem appendix: CBM UCBVI}. 
\begin{lemma}[CBM-UCBVI, Key Recursion Bound]\label{lemma: key recursion bound RL}
Conditioned on the good event $\G$, the following bound holds for all $h\in [H]$. 
\begin{align*}
    &\sum_{t=1}^{T} \bar{V}_{t,h}(s_{t,h}) - V^{\pi_t}_{h}(s_{t,h})\\
    &\leq  27H^2\log\br*{\frac{2HT(T+1)}{\delta}}+ 2\sum_{t=1}^T b^r_{t,h}(s_{t,h},a_{t,h}) + \sum_{t=1}^T 3H\sqrt{\frac{ L_{t,\delta}}{n_{t-1,h}( s_{t,h},a_{t,h})\vee 1}} + \frac{13H^2S L_{t,\delta}}{ n_{t-1,h}(s_{t,h},a_{t,h})\vee 1}  \\
    &\quad\quad + \br*{1+ \frac{1}{2H}}^2\sum_{t=1}^T \bar{V}_{t,h+1}(s_{t,h+1}) - V^{\pi_t}_{h+1}(s_{t,h+1}).
\end{align*}
\end{lemma}

\begin{proof}
Let $\Delta P_{t-1,h}(\cdot |s,a) = (\bar{P}_{t-1,h}-P_h)(\cdot |s,a)$. We bound each of the terms in the sum as follows.
\begin{align}
    \bar{V}_{t,h}(s_{t,h})& - V_{h}^{\pi_t}(s_{t,h}) \nonumber \\
    &= \bar{r}_{t-1,h}(s_1,a_{t,h}) + b^r_{t,h}(s_{t,h},a_{t,h}) - r_h(s_{t,h},a_{t,h}) + b^p_{t,h}(s_{t,h},a_{t,h})  \nonumber\\
    &\quad+ \E_{\bar{P}_{t-1,h}(\cdot| s_{t,h},a_{t,h})}[ \bar{V}_{t,h+1}(s_{h+1})] - \E_{P_{h}(\cdot| s_{t,h},a_{t,h})}[  V_{h+1}^{\pi_t}(s_{h+1})] \nonumber \\
    &\leq  2b^r_{t,h}(s_{t,h},a_{t,h}) + b^p_{t,h}(s_{t,h},a_{t,h}) \nonumber\\
    &\quad+ \E_{\bar{P}_{t-1,h}(\cdot| s_{t,h},a_{t,h})}[ \bar{V}_{t,h+1}(s_{h+1})] - \E_{P_{h}(\cdot| s_{t,h},a_{t,h})}[  V_{h+1}^{\pi_t}(s_{h+1})]\tag{$\G_1$ holds}\nonumber\\
    & =  2b^r_{t,h}(s_{t,h},a_{t,h}) + b^p_{t,h}(s_{t,h},a_{t,h})  + \E_{P_{h}(\cdot| s_{t,h},a_{t,h})}[  \bar{V}_{t,h+1}(s_{h+1}) - V_{h+1}^{\pi_t}(s_{h+1}) ] \nonumber\\
    &\quad+\Delta P_{t-1,h} (\cdot| s_{t,h},a_{t,h})^T \bar{V}_{t-1,h+1} \nonumber\\
    & = 2b^r_{t,h}(s_{t,h},a_{t,h}) + b^p_{t,h}(s_{t,h},a_{t,h}) + \E_{P_{h}(\cdot| s_{t,h},a_{t,h})}[  \bar{V}_{t,h+1}(s_{h+1}) - V_{h+1}^{\pi_t}(s_{h+1}) ]  \nonumber\\
    &\quad+ \Delta P_{t-1,h} (\cdot| s_{t,h},a_{t,h})^T(\bar{V}_{t,h+1} - V^*_{h+1}) +\underbrace{\Delta P_{t-1,h} (\cdot| s_{t,h},a_{t,h})^TV^*_{h+1}}_{\leq b_{t,h}^{p}(s_{t,h},a_{t,h})} \nonumber \\
    &\leq 2b^r_{t,h}(s_{t,h},a_{t,h}) + 2b^p_{t,h}(s_{t,h},a_{t,h})  + \E_{P_{h}(\cdot| s_{t,h},a_{t,h})}[  \bar{V}_{t,h+1}(s_{h+1}) - V_{h+1}^{\pi_t}(s_{h+1}) ]\nonumber \\
    & \quad + \Delta P_{t-1,h} (\cdot| s_{t,h},a_{t,h})^T(\bar{V}_{t,h+1} - V^*_{h+1})\tag{$\G_1$ holds}\nonumber\\
    &\overset{(a)}{\le}2b^r_{t,h}(s_{t,h},a_{t,h}) + 3H\sqrt{\frac{ L_{t,\delta}}{n_{t-1,h}( s_{t,h},a_{t,h})\vee 1}} + \frac{ 13H^2SL_{t,\delta}}{n_{t-1,h}( s_{t,h},a_{t,h})\vee 1} \nonumber\\
    &\quad + \br*{1+\frac{1}{2H}}\E_{P_{h}(\cdot| s_{t,h},a_{t,h})}[  \bar{V}_{t,h+1}(s_{h+1}) - V_{h+1}^{\pi_t}(s_{h+1})],
\end{align}
where $(a)$ by substituting $b^p_{t,h}(s,a)$ while bounding its last term by $\frac{5H^2SL_{t,\delta}}{n_{t-1,h}(s,a)\vee1}$ and by applying Lemma~\ref{lemma: transition different to next state expectation}. Specifically, we set ${\alpha=2H,C_1=C_2=2L_{t,\delta}}$ and, thus, $HS(C_2+\alpha C_2/4)\le 3H^2SL_{t,\delta}$ and observe that the conditions of \Cref{lemma: transition different to next state expectation} hold since the event $\cap_t E^p(t)$ holds and by the optimism of \Cref{lemma: optimism cbm-ucbvi} under the good event. 

Taking the sum over the latter inequality we conclude the proof since
\begin{align*}
    &\sum_{t=1}^T \bar{V}_{t,h}(s_{t,h}) - V_{h}^{\pi_t}(s_{t,h})\\
    &\leq 2\sum_{t=1}^Tb^r_{t,h}(s_{t,h},a_{t,h}) + \sum_{t=1}^T 3H\sqrt{\frac{ L_{t,\delta}}{n_{t-1,h}( s_{t,h},a_{t,h})\vee 1}} + \frac{ 13H^2SL_{t,\delta}}{n_{t-1,h}( s_{t,h},a_{t,h})\vee 1} \\
    &\quad + \br*{1+\frac{1}{2H}}\sum_{t=1}^T \E_{P_{h}(\cdot| s_{t,h},a_{t,h})}[  \bar{V}_{t,h+1}(s_{h+1}) - V_{h+1}^{\pi_t}(s_{h+1})]\\
    &=2\sum_{t=1}^Tb^r_{t,h}(s_{t,h},a_{t,h}) + \sum_{t=1}^T 3H\sqrt{\frac{ L_{t,\delta}}{n_{t-1,h}( s_{t,h},a_{t,h})\vee 1}} + \frac{ 13H^2SL_{t,\delta}}{n_{t-1,h}( s_{t,h},a_{t,h})\vee 1} \\
    &\quad + \br*{1+\frac{1}{2H}}\sum_{t=1}^T \E[  \bar{V}_{t,h+1}(s_{h+1}) - V_{h+1}^{\pi_t}(s_{h+1})|F_{t,h-1}]\\
    &\leq 27H^2 \log\br*{\frac{2HT(T+1)}{\delta}} + 2\sum_{t=1}^Tb^r_{t,h}(s_{t,h},a_{t,h}) + \sum_{t=1}^T 3H\sqrt{\frac{ L_{t,\delta}}{n_{t-1,h}( s_{t,h},a_{t,h})\vee 1}} + \frac{ 13H^2SL_{t,\delta}}{n_{t-1,h}( s_{t,h},a_{t,h})\vee 1} \\
    &\quad + \br*{1+\frac{1}{2H}}^2\sum_{t=1}^T   \bar{V}_{t,h+1}(s_{t,h+1}) - V_{h+1}^{\pi_t}(s_{t,h+1}) \tag{The second good event holds, $\G_2 = E^O$},
\end{align*}
where in the last relation we also bounded $18\br*{1+\frac{1}{2H}}\le 27$.
\end{proof}

We are now ready to prove Theorem~\ref{theorem appendix: CBM UCBVI}.
\begin{proof}
We start by conditioning on the good event $\G$, which holds with probability greater than $1-\delta$. Conditioned on the good event, the value is optimistic (Lemma~\ref{lemma: optimism cbm-ucbvi}). This fact, together with the key recursion lemma (Lemma~\ref{lemma: key recursion bound RL}) yields the following.
\begin{align*}
    \sum_{t=1}^{T}& V_{1}^*(s_{t,1}) - V_{1}^{\pi_t}(s_{t,1}) \leq \sum_{t=1}^{T} \bar{V}_{t,1}(s_{t,1}) - V_{1}^{\pi_t}(s_{t,1})\\
    &\leq  27H^2 \log\br*{\frac{2HT(T+1)}{\delta}} + \sum_{t=1}^T2 b^r_{t,1}(s_{t,1},a_{t,1}) + 3H\sqrt{\frac{ L_{t,\delta}}{n_{t-1,h}( s_{t,h},a_{t,h})\vee 1}} + \frac{ 13H^2SL_{t,\delta}}{n_{t-1,1}( s_{t,1},a_{t,1})\vee 1} \\
    &\quad +  \br*{1+\frac{1}{2H}}^2 \sum_{t=1}^T \bar{V}_{t,2}(s_{t,h+1}) - V_{2}^{\pi_t}(s_{t,2}) . \tag{Lemma~\ref{lemma: key recursion bound RL}}
\end{align*}
Iterating on this relation over $h\in\brc*{2,\dots,H}$ and using $\br*{1+\frac{1}{2H}}^{2h}\leq e\leq 3$ for $h\le H$ and $\bar{V}_{t,H+1}(s) = V_{H+}^{\pi_t}(s) = 0$, we conclude that
\begin{align}
    \sum_{t=1}^{T} V_{1}^*(s_{t,1}) - V_{1}^{\pi_t}(s_{t,1})
    &\leq 81H^3 \log\br*{\frac{2HT(T+1)}{\delta}} + 6\sum_{t=1}^T \sum_{h=1}^H b^r_{t,h}(s_{t,h},a_{t,h}) \nonumber \\
    &\quad + \sum_{t=1}^T \sum_{h=1}^H 9H\sqrt{\frac{ L_{t,\delta}}{n_{t-1,h}( s_{t,h},a_{t,h})\vee 1}} + \frac{ 39 H^2S L_{t,\delta} }{n_{t-1,h}( s_{t,h},a_{t,h}) \vee 1}. \label{eq: rl final bound relation 1}
\end{align}
The first sum in~\eqref{eq: rl final bound relation 1} is bounded in \Cref{lemma: bound on cummulative reward bonus RL} as
\begin{align*}
     6\sum_{t=1}^T &\sum_{h=1}^H b^r_{t,h}(s_{t,h},a_{t,h}) \\
     &\leq  3L_{t,\delta}\br*{6  \sqrt{|\LR| HT  } + 10 SAH\log(HT) + 23SAH + \sum_{t=1}^T\br*{6\sqrt{\frac{|\LR|H^2}{B(t)}} + 4SAH\frac{\log(1+ B(t))+1}{B(t)}}}.
\end{align*}
The second sum in~\eqref{eq: rl final bound relation 1} is bounded via the following standard analysis as follows.
\begin{align}
    \sum_{t=1}^T \sum_{h=1}^H \frac{ 1}{n_{t-1,h}( s_{t,h},a_{t,h}) \vee 1}
    &= \sum_{s,a,h}\sum_{i=0}^{n_{T,h}(s,a)} \frac{ 1}{i \vee 1} \tag{Rewriting the summation} \nonumber\\
    &\leq \sum_{s,a,h}(2+\log(n_{T,h}(s,a)\vee 1))  \nonumber\\
    &\leq SAH(2+\log(TH))  \label{eq: standard RL analysis sum of 1/n}
\end{align}
where the last relation holds by Jensen's inequality while using $\sum_{s,a,h}n_{T,h}(s,a) = TH$. Similarly, we bound
\begin{align}
    \sum_{t=1}^T \sum_{h=1}^H \frac{ 1}{\sqrt{n_{t-1,h}( s_{t,h},a_{t,h)} \vee 1}}
    &= \sum_{s,a,h}\sum_{i=0}^{n_{T,h}(s,a)} \frac{ 1}{\sqrt{i \vee 1}} \tag{Rewriting the summation} \nonumber\\
    &\leq \sum_{s,a,h}\br*{1+2\sqrt{n_{T,h}(s,a)}}  \nonumber\\
    &\leq SAH + 2\sqrt{SAH^2T} \label{eq: standard RL analysis sum of 1/sqrt(n)}
\end{align}
Thus, the second sum in~\eqref{eq: rl final bound relation 1} is bounded by
\begin{align*}
    &\sum_{t=1}^T \sum_{h=1}^H 9H\sqrt{\frac{ L_{t,\delta}}{n_{t-1,h}( s_{t,h},a_{t,h})\vee 1}} + \sum_{t=1}^T \sum_{h=1}^H \frac{ 39 H^2S L_{t,\delta} }{n_{t-1,h}( s_{t,h},a_{t,h}) \vee 1}\\
    &\leq  9H  \sqrt{L_{T,\delta}}\sum_{t=1}^T \sum_{h=1}^H \sqrt{\frac{1}{n_{t-1,h}( s_{t,h},a_{t,h})\vee 1}} + 39 H^2S L_{T,\delta} \sum_{t=1}^T \sum_{h=1}^H \frac{ 1}{n_{t-1,h}( s_{t,h},a_{t,h}) \vee 1} \tag{$L_{t,\delta}$ is increasing in $t$}\\
    &\leq 18  \sqrt{SAH^4T L_{T,\delta}} +39 H^3S^2A L_{T,\delta}\log(TH) + 87H^3S^2A L_{T,\delta}. \tag{By~\eqref{eq: standard RL analysis sum of 1/n}, \eqref{eq: standard RL analysis sum of 1/sqrt(n)} and since $L_{T,\delta}\ge1$}
\end{align*}
Plugging the bounds on the first and second sums in~\eqref{eq: rl final bound relation 1} concludes the proof.
\end{proof}

\begin{restatable}[Bound on Cumulative Reward Bonus]{lemma-rst}{BoundCumulativeRewardBonusRL}\label{lemma: bound on cummulative reward bonus RL}
The following bound holds almost surely
\begin{align*}
    \sum_{t=1}^T \sum_{h=1}^H &CI^r_{t,h}(s_{t,h},a_{t,h})
    = 2\sum_{t=1}^T \sum_{h=1}^H b^r_{t,h}(s_{t,h},a_{t,h})\\
    &\leq L_{t,\delta}\br*{6  \sqrt{|\LR| HT  } + 10 SAH\log(HT) + 23SAH + \sum_{t=1}^T\br*{6\sqrt{\frac{|\LR|H^2}{B(t)}} + 4SAH^2\frac{\log(1+ B(t))+1}{B(t)}}}.
\end{align*}
\end{restatable}

We bound this sum based on the analysis technique that was utilized for CBM-UCB and UCB-OFUL.
\begin{proof}
Let $q_{t,h}=1$ be the event reward feedback was queried at the $h^{th}$ time step of the $t^{th}$ episode and $q_{t,h}=0$ be its complement event. The following relations hold. 
\begin{align}
    & \sum_{t=1}^T \sum_{h=1}^H CI^r_{t,h}(s_{t,h},a_{t,h}) = 2\sum_{t=1}^T \sum_{h=1}^H b^r_{t,h}(s_{t,h},a_{t,h}) \nonumber \\
    &=\underbrace{\sum_{t=1}^T \sum_{h=1}^H \indicator{q_{t,h}=1}CI^r_{t,h}(s_{t,h},a_{t,h})}_{(i)} + \underbrace{\sum_{t=1}^T \sum_{h=1}^H \indicator{q_{t,h}=0}CI^r_{t,h}(s_{t,h},a_{t,h})}_{(ii)} \label{eq: central term reward bonus lemma}
\end{align}

{\bf Bound on term~$(i)$, episodes in which feedback is queried.} The following relations hold.
\begin{align*}
    (i) &= 2\sum_{t=1}^T \sum_{h=1}^H \indicator{q_{t,h}=1}b^r_{t,h}(s_{t,h},a_{t,h})\\
    &\leq 2\sum_{t=1}^T \sum_{h=1}^H \indicator{q_{t,h}=1}\sqrt{ \frac{2 \widehat{\mathrm{Var}}_{r,t-1,h}(s_{t,h},a_{t,h}) L_{t,\delta}}{n^q_{t-1,h}(s_{t,h},a_{t,h})\vee 1}} +  2\sum_{t=1}^T \sum_{h=1}^H \indicator{q_{t,h}=1}\frac{5 L_{t,\delta}}{n^q_{t-1,h}(s_{t,h},a_{t,h})\vee 1}\\
    &\overset{(a)}{\leq} 3  \sqrt{L_{T,\delta}}\sum_{s,a,h\in \LR} \sum_{i=0}^{n^q_{T,h}(s,a)}\sqrt{ \frac{1}{n^q_{t-1,h}(s_{t,h},a_{t,h})\vee 1}} +  10L_{T,\delta}\sum_{s,a,h} \sum_{i=0}^{n_{T,h}(s,a)} \frac{1}{ n^q_{t-1,h}(s_{t,h},a_{t,h})\vee 1} \\
    &\leq 6  \sqrt{|\LR|\Bq(T) L_{T,\delta} } + 10 L_{T,\delta}SAH\log(\Bq(T))  + 23SAHL_{T,\delta} \tag{As~\eqref{eq: standard RL analysis sum of 1/n},\eqref{eq: standard RL analysis sum of 1/sqrt(n)}, \& $\sum_{s,a,h}n^q_{T,h}(s,a) = \Bq(T)$}\\
    &\leq 6  \sqrt{|\LR| TH L_{T,\delta} } + 10 L_{T,\delta}SAH\log(TH)+ 23SAHL_{T,\delta}, \tag{$\Bq(T)\leq TH$}
\end{align*}
Relation $(a)$ holds since $\widehat{\mathrm{Var}}_{r,t-1,h}(s_{t,h},a_{t,h})=0$ for all rewards with deterministic reward (and zero reward is deterministic); for non-deterministic rewards, $\widehat{\mathrm{Var}}_{r,t-1,h}(s_{t,h},a_{t,h})\leq 1$.

{\bf Bound on term~$(ii)$, episodes in which feedback is not queried.}
Due to Lemma~\ref{lemma: cbm-ucbvi budget constraint is satisfied} it holds that  $\Bq(T)\leq B(T)$, that is, the budget constraint is never violated. For this reason, if reward is not queried, it does not satisfy the query rule of CBM-UCBVI. Using this, we get the following relations.
\begin{align*}
    (ii) &= \sum_{t=1}^T \sum_{h=1}^H \indicator{q_{t,h}=0}CI^r_{t,h}(s_{t,h},a_{t,h})\\
    &\leq \sum_{t=1}^T \sum_{h=1}^H \br*{L_{t,\delta}\br*{6\sqrt{\frac{|\LR|}{B(t)}} + 4SAH \frac{\log(1+B(t))+1}{B(t)}}}\\
    &\leq L_{T,\delta}\sum_{t=1}^T \sum_{h=1}^H \br*{6\sqrt{\frac{|\LR|}{B(t)}} +4SAH \frac{\log(1+B(t))+1}{B(t)}}.
\end{align*}

{\bf Combining the bounds.} Combining the bounds on $(i)$ and $(ii)$ we conclude the proof,
\begin{align*}
    \eqref{eq: central term reward bonus lemma} \leq L_{t,\delta}\br*{6  \sqrt{|\LR| HT  } + 10 SAH\log(HT) + 23SAH + \sum_{t=1}^T\br*{6\sqrt{\frac{|\LR|H^2}{B(t)}} + 4SAH^2\frac{\log(1+ B(t))+1}{B(t)}}}.
\end{align*}
\end{proof}

\newpage


\newpage
\section{Confidence Budget Matching for Reinforcement Learning -- CBM-ULCVI}\label{appendix: cbm-ulcvi RL}

\begin{algorithm}[h]
\caption{CBM-ULCVI} \label{alg: budget RL}
\begin{algorithmic}
\STATE {\bf Require:} $\delta\in(0,1)$
\FOR{$t=1,2,...$}
    \STATE $\brc*{\bar{Q}_{t,h}(s,a),\underline{Q}_{t,h}(s,a),\pi_{t,h}(a|s)}_{(s,a)\in \Scal\times\Acal,h\in [H]}$ via Optimistic-Pessimistic VI (\Cref{alg: optimistic pessimistic value iteration})
    \STATE Act with $\pi_{t,h}(a|s)$ and observe a trajectory $\brc{(s_{t,h},a_{t,h})}_{h=1}^H$
    \STATE Observe current budget $B(t)$
    \STATE Ask for feedback on $(s_{t,h},a_{t,h})$ if
    \begin{align*}
         CI_{t,h}^R(s_{t,h},a_{t,h}) \geq L_{t,\delta}\br*{6\sqrt{\frac{|\LR|}{B(t)}} + 4SAH\frac{\log(1+B(t))+1}{B(t)}}.
    \end{align*}
\ENDFOR
\end{algorithmic}
\end{algorithm}

In the previous section, we analyzed the performance of CBM-UCBVI algorithm, which incorporates the CBM principle into the Hoeffding-based UCBVI-CH~\citep{azar2017minimax}. However, the performance of CBM-UCBVI is worse by a factor of $\sqrt{H}$ relatively to the minimax performance even when the budget is $B(t)=Ht$. Then, the minimax regret bound is $\sqrt{SAH^3T}$ \citep[\eg][]{jin2019learning}. Building on ideas from~\citep{azar2017minimax,dann2019policy} we analyze the CBM Upper Lower Confidence Interval VI (CBM-ULCVI), which uses tighter CI to shave an additional $\sqrt{H}$ from the asymptotic performance.

The idea is essentially the one used in~\citep{azar2017minimax,jin2018q,zanette2019tighter,dann2019policy,efroni2019tight}. Instead of bounding $\VAR_{P(\cdot|s,a)} V^*_{h+1}\leq H^2$ and constructing the bonus as $b^p_{t,h}(s,a)\sim \frac{H}{\sqrt{n_{t-1,h}(s,a)}}$ (see Section~\ref{appendix: cbm-ucbvi RL}), we build the following bonus to compensate on errors in transition model:
\begin{align}
    b^p_{t,h}(s,a)\sim \frac{\sqrt{\VAR_{\bar{P}_{t-1,h}(\cdot |s,a)}(\bar{V}_{h+1}^{k})} }{\sqrt{n_{t-1,h}(s,a)}} + \frac{1}{n_{t-1,h}(s,a)}. \label{eq: sketch b_p cbvi UL}
\end{align}
This allows to get a performance guarantee that depends on the sum of variances along the trajectories $\sum_{t}\sum_{h} \VAR_{P_{h}(\cdot|s_{t,h},a_{t,h})}$. These can be bounded by $\sim H^2T$ via the law of total variance~\citep{azar2017minimax}. Ultimately, this shaves a $\sqrt{H}$ factor in the final bound: a `na\"ive' bound, without applying the law of total variance, would result in $H^3T$. Our actual bonus is similar to the one used in~\cite{dann2019policy}, i.e., we use proper lower and upper value functions in addition to the bonus~\eqref{eq: sketch b_p cbvi UL}. However, our analysis is more `direct' relatively to~\citep{dann2019policy}; we bound the regret by on-policy errors, and not expected on-policy errors as in~\cite{dann2019policy}. This becomes crucial due to the usage in the CBM scheme in our algorithm. Specifically, the CBM allows us to bound only $$CI^R_{t,h}(s_{t,h},a_{t,h})\indicator{q_{t,h}=0}\leq \Olog\br*{\sqrt{\frac{|\LR|}{B(t)}} +\frac{SAH}{B(t)}},$$
only on encountered state-action pairs, i.e., not under expectation.

Formally, we work with the following bonuses:
\begin{align*}
    &b_{t,h}^r(s,a) = \sqrt{ \frac{2 \widehat{\mathrm{Var}}_{R,t-1,h}(s,a) L_{t,\delta} }{n^q_{t-1,h}(s,a)\vee 1}} +  \frac{5 L_{t,\delta}}{ n^q_{t-1,h}(s,a)\vee 1},\\
    &b_{t,h}^{p}(s,a; \bar{V}_{t,h+1},\ubar{V}_{t,h+1}) = \sqrt{\frac{2\VAR_{\bar{P}_{t-1,h}(\cdot| s,a)}(\bar{V}_{t,h+1}) L_{t,\delta}}{n_{t-1,h}(s,a)\vee 1}} + \frac{44H^2S L_{t,\delta}}{n_{t-1,h}(s,a) \vee 1}  \\
    &\qquad\qquad\qquad\qquad\qquad\quad + \frac{1}{16H}\E_{\bar{P}_{t-1,h}(\cdot|s,a)}\brs{\bar{V}_{t,h+1}(s') - \ubar{V}_{t,h+1}(s')}
\end{align*}
where $L_{t,\delta}=\log \frac{16S^2AH t^2(t+1)}{\delta}$. As with CBM-UCBVI, we define $CI_{t,h}^R(s_{t,h},a_{t,h}) = 2b_{t,h}^r(s,a)$ and similarly ask for feedback only if 
\begin{align*}
     CI_{t,h}^R(s_{t,h},a_{t,h}) \geq L_{t,\delta}\br*{6\sqrt{\frac{|\LR|}{B(t)}} + 4SAH\frac{\log(1+B(t))+1}{B(t)}},
\end{align*}
where $\LR$ is the set of rewarding state-actions, defined in \Cref{appendix: cbm-ucbvi RL}. This leads to \Cref{alg: budget RL}, which has the following regret guarantees:
\begin{theorem}[CBM-ULCVI]\label{theorem appendix: CBM UCBVI UL}
For any adversarially adaptive sequence $\brc*{B(t),s_{t,1}}_{t\geq 1}$ of budget and initial states the  regret of CBM-ULCBVI is upper bounded by 
$$
\Regret(T) \leq  45 L_{T,\delta}^2 \br*{\sqrt{SAH^3 T} + \sum_{t=1}^T\br*{\sqrt{\frac{|\LR|H^2}{B(t)}} + SAH^2\frac{\log(1+ B(t))+1}{B(t)}}}+ 2904 H^3S^2A L^2_{T,\delta}.
$$
for any $T\geq 1$ with probability greater than $1-\delta$.
\end{theorem}

\begin{algorithm}[t]
\caption{Optimistic-Pessimistic Value Iteration} \label{alg: optimistic pessimistic value iteration}
\begin{algorithmic}
\STATE {\bf Require:} For all $s,a,h$, $n_{t-1,h}(s,a),n^{q}_{t-1,h}(s,a), \bar{P}_{t-1,h}(s), \bar{r}_{t-1,h}(s,a)$
\STATE {\bf Initialize:} $\bar{V}_{H+1}(s)=\ubar{V}_{H+1}(s)=0$ for all $s\in \Scal$.
\FOR{$h=H,H-1,..,1$}
    \FOR{$s\in \Scal$}
        \STATE $\forall a\in\Acal,\ \bar{Q}_{t,h}(s,a) =  \bar{r}_{t-1,h}(s,a) + b^r_{t,h}(s,a) + b_{t,h}^{p}(s,a; \bar{V}_{t,h+1},\ubar{V}_{t,h+1}) + \E_{\bar{P}_{t-1,h}(\cdot|s,a)}[\bar{V}_{t,h+1}(s')]$
        \STATE $\pi_{t,h}(s)\in \arg\max_a \bar{Q}_{t,h}(s,a)$
        \STATE $\forall a\in\Acal,\ \underline{Q}_{t,h}(s,a) =  \bar{r}_{t-1,h}(s,a) - b^r_{t,h}(s,a) - b_{t,h}^{p}(s,a; \bar{V}_{t,h+1},\ubar{V}_{t,h+1})  +\E_{\bar{P}_{t-1,h}(\cdot|s,a)}[\ubar{V}_{t,h+1}(s')]$
        \STATE \textcolor{gray}{// Truncate values outside the interval $[0,H-h+1]$}
        \STATE $\bar{V}_{t,h}(s) = \min\brc*{\max_{a\in \Acal} \bar{Q}_{t,h}(s,a),H-h+1}$,  $\ubar{V}_{t,h}(s) = \max( \underline{Q}_{t,h}(s,\pi_{t,h}(s)),0)$
    \ENDFOR
\ENDFOR
\STATE {\bf Return:} $\brc*{\bar{Q}_t(s,a), \underline{Q}_t(s,a), \pi_{t,h}(a|s)}_{(s,a)\in \Scal\times\Acal,h\in [H]}$
\end{algorithmic}
\end{algorithm}

\subsection{The First Good Event - Concentration}

Define the following events:
\begin{align*}
    &E^r(t) = \brc*{\forall s\in S,a \in A:\ |\bar{r}_{t-1,h}(s,a) -r_h(s,a)|  \leq \sqrt{ \frac{2 \widehat{\mathrm{Var}}_{R,t-1,h}(s,a)\log\frac{16SAH t^2(t+1)}{\delta} }{n^q_{t-1,h}(s,a)\vee 1}} +  \frac{5\log \frac{16SAH t^2(t+1)}{\delta}}{ n^q_{t-1,h}(s,a)\vee 1}} \\
    &E^p(t) = \brc*{\forall s,s'\in S, a\in A:\ |P_h\br*{s'|s,a} - \bar{P}_{t-1,h}\br*{s'|s,a}| \le \sqrt{\frac{2P(s'|s,a)\log\frac{16S^2AH t^2(t+1)}{\delta}}{n_{t-1,h}(s,a)\vee 1}} + \frac{2\log \frac{16S^2AH t^2(t+1)}{\delta}}{n_{t-1,h}(s,a)\vee 1}} \\
    &E^{pv1}(t)=\brc*{\forall s,a,h:\ \abs*{\br*{\bar{P}_{t-1,h}(\cdot \mid s,a)-P_h(\cdot \mid s,a)}^T V_{h+1}^*} \leq \sqrt{\frac{2\VAR_{P_h(\cdot| s,a)}(V^*_{t+1})\log \frac{16SAH t^2(t+1)}{\delta}}{n_{t-1,h}(s,a)\vee 1}} + \frac{5H\log \frac{16SAH t^2(t+1)}{\delta}}{n_{t-1,h}(s,a)\vee 1}}\\
    &E^{pv2}(t)=\brc*{\forall s,a,h:\ \abs*{ \sqrt{\VAR_{P_h(\cdot| s,a )}(V_{h+1}^*)} -  \sqrt{\VAR_{\bar{P}_{t-1,h}(\cdot| s,a )}(V_{h+1}^*)} } \leq \sqrt{\frac{12 H^2 \log \frac{16SAH t^2(t+1)}{\delta}}{n_{t-1,h}(s,a)\vee 1}} }\\
\end{align*}
For brevity, we denote
\begin{align*}
    b_{t,h}^{pv1}(s,a) = \sqrt{\frac{2\VAR_{P_h(\cdot| s,a)}(V^*_{t+1})L_{t,\delta}}{n_{t-1,h}(s,a)\vee 1}} + \frac{5HL_{t,\delta}}{n_{t-1,h}(s,a)\vee 1}.
\end{align*}
\begin{lemma}[The First Good Event]\label{lemma: the first good event RL UL}
Let $\G_1 = \cap_{t\geq 1} E^r(t) \cap_{t\geq 1} E^p(t) \cap_{t\geq 1} E^{pv1}(t) \cap_{t\geq 1} E^{pv2}(t)$ be the good event. It holds that $\Pr(\G_1)\geq 1-\delta/2$.
\end{lemma}
\begin{proof}
Each one of the events $\cap_t E^r(t), \cap_t E^p(t),\cap_{t} E^{pv1}(t)$ holds with probability greater than $1-\delta/8$ via similar analysis as in Lemma~\ref{lemma: the first good event RL}. We now turn to establish the high probability guarantee for the rest of events.

{\bf  The event $\cap_t E^{pv2}(t)$ holds with high probability.} The event holds with high probability due to~\citep[][Theorem 10]{maurer2009empirical}. To see this, fix $s,a,h$ and assume it was sampled for a fixed number of times ${n_{t-1,h}(s,a)=n>1}$.  Let $V_{s,a,h,n} = \frac{1}{2n(n-1)} \sum_{i,j} (V^*_{h+1}(s_i) -  V^*_{h+1}(s_j))^2$ where $s_i\sim P_{h}(\cdot|s,a)$ and are i.i.d.. See that this is an unbiased estimator for the variance, i.e., $\E[V_{s,a,h,n}] = \VAR_{P_h(\cdot| s,a )}(V_{h+1}^*)$. Applying~\citep[][Theorem 10]{maurer2009empirical}, we get that
\begin{align}
    \Pr&\br*{\abs*{\sqrt{V_{s,a,h,n}} - \sqrt{\E[V_{s,a,h,n}]}} \leq \sqrt{\frac{2H^2\log(2/\delta)}{n-1}}} \nonumber\\
    &\qquad\qquad\qquad\qquad= \Pr\br*{\abs*{\sqrt{V_{s,a,h,n}} -  \sqrt{\VAR_{P_h(\cdot| s,a )}(V_{h+1}^*)}} \leq \sqrt{\frac{2H^2\log(2/\delta)}{n-1}}}\leq \delta. \label{eq: event PV2 maurer theorem 10}
\end{align}
On the other hand,
\begin{align*}
    V_{s,a,h,n} = \frac{1}{2n(n-1)} \sum_{i,j=1}^n (V^*_{h+1}(s_i) -  V^*_{h+1}(s_j))^2
    &= \sum_{s'} \frac{N(s')}{n}  V^*_{h+1}(s')^2 - \br*{\sum_{s'}\frac{N(s')}{n} V^*_{h+1}(s')}^2\\
    &\quad + \frac{1}{n(n-1)}\sum_{i=1}^n V^*_{h+1}(s_i)^2 - \frac{1}{n-1}\br*{\frac{1}{n}\sum_{i=1}^n V^*_{h+1}(s_i)}^2, \tag{by elementary algebra}
\end{align*}
where $N(s')$ is the number of times $s'$ was sampled. By definition, $\frac{N(s')}{n}\equiv \bar{P}_{n,h}(s'|s,a)$, that is, $\frac{N(s')}{n}$ is the estimated transition model. Thus, the $\VAR_{\bar{P}_{n,h}(\cdot|s,a)} =\sum_{s'} \frac{N(s')}{n}  V^*_{h+1}(s')^2 - \br*{\sum_{s'}\frac{N(s')}{n} V^*_{h+1}(s_i)}^2$. This, together with $|V^*_{h}(s)|\leq H$, implies that $\abs*{V_{s,a,h,n} - \VAR_{\bar{P}_{n,h}(\cdot| s,a )}(V_{h+1}^*)} \leq \frac{H^2}{n-1}$, and, thus
\begin{align}
   \abs*{ \sqrt{V_{s,a,h,n}} - \sqrt{\VAR_{\bar{P}_{n,h}(\cdot| s,a )}(V_{h+1}^*)}}
   \leq \sqrt{\abs*{V_{s,a,h,n} - \VAR_{\bar{P}_{n,h}(\cdot| s,a )}(V_{h+1}^*)}}\leq   \sqrt{\frac{H^2}{n-1}},\label{eq: event PV2 maurer theorem 10 relation 1}
\end{align}
by $\abs{\sqrt{a}-\sqrt{b}}\leq \sqrt{\abs*{a-b}}$. Combining~\eqref{eq: event PV2 maurer theorem 10} and~\eqref{eq: event PV2 maurer theorem 10 relation 1} yields
\begin{align*}
   \Pr\br*{\abs*{\sqrt{\VAR_{\bar{P}_{n,h}(\cdot| s,a )}(V_{h+1}^*)} -  \sqrt{\VAR_{P_h(\cdot| s,a )}(V_{h+1}^*)}} \leq \sqrt{ \frac{12H^2\log(2/\delta)}{n\vee 1}}}\leq \delta,
\end{align*}
since $\frac{2}{n\vee 1}\geq \frac{1}{n-1}$ for $n\geq 2$ and for $n\leq 1$ the bound trivially holds. Utilizing the same techniques as in \Cref{lemma: the first good event RL} to generalize this result to hold for any (random) value of $n=n_{t-1,h}(s,a)$ and any $s,a,h$ establishes the high probability guarantee.

{\bf Combining the results.} Taking a union bound concludes the proof.
\end{proof}

\subsection{Optimism and Pessimism of Upper and Lower Value Functions}
We can prove that the value is optimistic using standard techniques. 
\begin{lemma}[Upper Value Function is Optimistic, Lower Value Function is Pessimistic] \label{lemma: optimism cbm-ucbvi-UL}
Conditioned on the first good event $\G_1$ the value function of CBM-UCBVI is optimism for all $s\in \mathcal{S},h\in [H]$, i.e.,
\begin{align*}
    \forall s\in \Scal,h\in[H]:\ \ubar{V}_{t,h}(s) \leq V^{\pi_t}_h(s)\leq V^*_h(s)\leq \bar{V}_{t,h}(s).
\end{align*} 
\end{lemma}
\begin{proof}
Since $V^\pi_h(s)\leq V^*_h(s)$ for all $\pi,s\in \Scal,h\in[H]$, we only need to prove the leftmost and rightmost inequalities of the claim. We prove this result via induction.

{\bf Base case, the claim holds for $h=H$.} 

{\bf Rightmost inequality, optimism, .} Let $a^*(s)\in  \arg\max_{a\in \Acal} r_{H}(s,a)$. For any $s\in \Scal$ it holds that
\begin{align}
    &V^*_{H}(s) -\bar{V}_{t,H}(s) = r_{H}(s,a^*(s)) - \min\brc*{\max_a \brc*{\bar{r}_{t,H}(s,a) + b_{t,H}^r(s,a)},1}. \label{eq: optimism base case ulcvi rel 1}
\end{align}
Assume that $\max_a \bar{r}_{t,H}(s,a) + b_{t,H}^r(s,a) < 1$ (otherwise, the inequality holds since $r(s,a)\leq 1$). Then,
\begin{align*}
    \eqref{eq: optimism base case ulcvi rel 1} 
    &\leq  r_{H}(s,a^*(s)) - \bar{r}_{t,H}(s,a^*(s)) - b_{t,H}^r(s,a^*(s))\\
    &\leq b_{t,H}^r(s,a^*(s)) - b_{t,H}^r(s,a^*(s))\leq 0 \tag{event $\cap_t E^r(t)$ holds}.
\end{align*}
Thus, $V^*_{H}(s) \leq \bar{V}_{t,H}(s)$ for all $s\in \Scal$ for $h=H$.

{\bf Leftmost inequality, pessimism.}  For any $s\in \Scal$ it holds that
\begin{align}
    &V^{\pi_t}_{H}(s) -\ubar{V}_{t,H}(s) = r_{H}(s,\pi_{t,h}(s)) - \max(\bar{r}_{t,H}(s,\pi_{t,h}(s)) - b_{t,H}^r(s,\pi_{t,h}(s)),0). \label{eq: optimism base case ucbvi UL rel 1 pessimsm}
\end{align}
Assume that $\bar{r}_{t,H}(s,\pi_{t,h}(s)) - b_{t,H}^r(s,\pi_{t,h}(s)) > 0$ (otherwise, the inequality holds since $r(s,a)\geq 0$). Then,
\begin{align*}
    &\eqref{eq: optimism base case ucbvi UL rel 1 pessimsm} \geq -b_{t,H}^r(s,\pi_{t,h}(s)) + b_{t,H}^r(s,\pi_{t,h}(s))\geq 0 \tag{event $\cap_t E^r(t)$ holds}.
\end{align*}
Thus, $V^*_{H}(s) \geq \ubar{V}_{t,H}(s)$ for all $s\in \Scal$ for $h=H$.

{\bf Induction step, prove for $h\in [H]$ assuming the claim holds for all $h+1\leq  h'\leq H$.}

{\bf Rightmost inequality, optimism.} Let $a^*(s)\in \arg\max_{a\in \Acal} Q^{*}_h(s,a)$. The following relations hold.
\begin{align}
     &V^*_h(s)- \bar{V}_{t,h}(s) =   Q^{*}_h(s,a^*(s))- \min\brc*{\max_a\bar{Q}_{t,h}(s,a),H-h+1} \label{eq: optimism base case ucbvi rel 2 UL}
\end{align}
Assume that $\max_a\bar{Q}^{\pi}_{t,h}(s,a)< H-h+1$ (otherwise, the inequality is satisfied since $Q^*_h(s,a)\leq H-h+1$). Then,
\begin{align}
    \eqref{eq: optimism base case ucbvi rel 2 UL} &\leq Q^{*}_h(s,a^*(s))- \bar{Q}_{t,h}(s,a^*(s)) \nonumber \\
    &= r_h(s,a^*(s)) - \bar{r}_{t-1,h}(s,a^*(s)) - b_{t,h}^r(s,a^*(s)) - b_{t,h}^{p}(s,a^*(s))  \nonumber\\
    &\quad + (P_h- \bar{P}_{t-1,h})(\cdot| s,a^*(s))V^*_{h+1} + \E_{\bar{P}_{t-1,h}(\cdot|s,a^*(s))}[ \underbrace{V^{*}_{h+1}(s') - \bar{V}_{t,h+1}(s')}_{\leq 0\ \mathrm{Induction\ hypothesis}}] \nonumber\\
    &\leq  b_{t,h}^r(s,a^*(s)) - b_{t,h}^r(s,a^*(s)) - b_{t,h}^{p}(s,a^*(s)) +  b_{t,h}^{pv1}(s,a^*(s)) \tag{events $\cap_t E^{pv1}(t)\ \&\ \cap_t E^{r}(t)$ holds} \nonumber\\
    &= b_{t,h}^{pv1}(s,a^*(s)) - b_{t,h}^{p}(s,a^*(s)) \label{eq: optimism base case ucbvi rel 3 UL}
\end{align}
We now analyze this term.
\begin{align}
    \eqref{eq: optimism base case ucbvi rel 3 UL} 
    &= b_{t,h}^{pv1}(s,a^*(s)) - b_{t,h}^{p}(s,a^*(s)) \nonumber \\
    &\overset{(a)}{\leq} \sqrt{\frac{2\VAR_{P(\cdot| s,a^*(s))}(V^*_{h+1})L_{t,\delta}}{n_{t-1,h}(s,a^*(s)) \vee1}}  +\frac{5H L_{t,\delta}}{n_{t-1,h}(s,a^*(s))\vee 1} \nonumber \\
    &\quad -\sqrt{\frac{2\VAR_{\bar{P}_{t-1,h}(\cdot|s,a^*(s))}(\bar{V}_{t,h+1})L_{t,\delta}}{n_{t-1,h}(s,a^*(s)) \vee 1}} - \frac{18H^2L_{t,\delta}}{n_{t-1,h}(s,a^*(s))\vee 1} - \frac{1}{16H} \E_{\bar{P}_{t-1,h}(\cdot|s,a)}\brs*{\bar{V}_{t,h+1}(s')- V^*_{h+1}(s')} \nonumber\\
    &= \sqrt{2L_{t,\delta}}\frac{\sqrt{\VAR_{P_h(\cdot|s,a^*(s))}(V^*_{h+1})} -\sqrt{\VAR_{\bar{P}_{t-1,h}(\cdot| s,a^*(s))}(\bar{V}_{t,h+1})}}{\sqrt{n_{t-1,h}(s,a^*(s)) \vee1}}   \nonumber\\
    &\quad - \frac{1}{16H}  \E_{\bar{P}_{t-1,h}(\cdot|s,a)}\brs*{\bar{V}_{t,h+1}(s')- V^*_{h+1}(s')}- \frac{13H^2 L_{t,\delta}}{n_{t-1,h}(s,a^*(s))\vee 1} \nonumber\\
    &\overset{(b)}{\leq}  \frac{1}{16H}\E_{\bar{P}_{t-1,h}(\cdot|s,a)}\brs*{\bar{V}_{t,h+1}(s')- V^*_{h+1}(s')} + \frac{13H^2L_{t,\delta}}{ n_{t-1,h}(s,a)\vee 1}  \nonumber\\
    &\quad - \frac{1}{16H}  \E_{\bar{P}_{t-1,h}(\cdot|s,a)}\brs*{\bar{V}_{t,h+1}(s')- V^*_{h+1}(s')} - \frac{13H^2L_{t,\delta}}{n_{t-1,h}(s,a^*(s)) \vee1} = 0 \nonumber 
\end{align}
where $(a)$ holds by plugging the bonus and bounding $-44 H^2S L_{t,\delta} \leq -18H^2 L_{t,\delta}$, and by the induction hypothesis ($\ubar{V}_{t-1,h+1}(s)\le V^*_{h+1}(s),\forall s\in\Scal$). $(b)$ holds by Lemma~\ref{lemma: variance diff is upper bounded by value difference} while setting $\alpha= 16H$ and bounding $(5H+ H\alpha/2)\leq 13H^2$. 
Combining all the above we conclude the proof of the rightmost inequality since
\begin{align*}
    V^*_h(s)- \bar{V}_{t,h}(s) \leq \eqref{eq: optimism base case ucbvi rel 2 UL} \leq \eqref{eq: optimism base case ucbvi rel 3 UL}\leq 0.
\end{align*}

{\bf Leftmost inequality, pessimism.} The following relations hold.
\begin{align}
     &V^{\pi_t}_h(s)- \ubar{V}_{t,h}(s) =   Q^{\pi_t}_h(s,\pi_{t,h}(s))- \max(\bar{Q}_{t,h}(s,\pi_{t,h}(s)),0). \label{eq: optimism base case ucbvi UL rel 2, pessimsm}
\end{align}
Assume that $\bar{Q}_{t,h}(s,\pi_{t,h}(s))>0$ (otherwise, the claim holds since $ Q^{\pi_t}_h(s,\pi_{t,h}(s))\geq 0$). Then,
\begin{align} 
    \eqref{eq: optimism base case ucbvi UL rel 2, pessimsm} &=  Q^{\pi_t}_h(s,\pi_{t,h}(s))- \underline{Q}_{t,h}(s,\pi_{t,h}(s)) \nonumber \\
    &= r_h(s,\pi_{t,h}(s)) - \bar{r}_{t-1,h}(s,\pi_{t,h}(s)) + b_{t,h}^r(s,\pi_{t,h}(s)) + b_{t,h}^{p}(s,\pi_{t,h}(s))  \nonumber\\
    &\quad + (P_h-\bar{P}_{t-1,h})(\cdot| s,\pi_{t,h}(s))V^{\pi_t}_{h+1} + \E_{\bar{P}_{t-1,h}(\cdot|s,a)}[ \underbrace{V^{\pi_t}_{h+1}(s') - \ubar{V}_{t,h+1}(s')}_{\geq 0\ \mathrm{Induction\ hypothesis}}|s_h=s,a_h=\pi_{t,h}(s)] \nonumber\\
    &\geq  -b_{t,h}^r(s,\pi_{t,h}(s)) + b_{t,h}^r(s,\pi_{t,h}(s)) + b_{t,h}^{p}(s,\pi_{t,h}(s)) + (P_h-\bar{P}_{t-1,h})(\cdot| s,\pi_{t,h}(s))V^{\pi_t}_{h+1} \tag{$\cap_t E^{r}(t)$ holds} \nonumber \\
    &= b_{t,h}^{p}(s,\pi_{t,h}(s)) + (P_h-\bar{P}_{t-1,h})(\cdot| s,\pi_{t,h}(s))V^{\pi_t}_{h+1}. \label{eq: optimism base case ucbvi UL rel 3, pessimsm}
\end{align}
We now focus on the last term. Observe that
\begin{align*}
    &(P_h-\bar{P}_{t-1,h})(\cdot| s,\pi_{t,h}(s))V^{\pi_t}_{h+1} = (P_h-\bar{P}_{t-1,h})(\cdot| s,\pi_{t,h}(s))V^{*}_{h+1} + (P_h-\bar{P}_{t-1,h})(\cdot| s,\pi_{t,h}(s))(V^{\pi_t}_{h+1} - V^*_{h+1})\\
    & \geq -b^{pv1}_{t-1,h}(s,\pi_{t,h}(s))+(P_h-\bar{P}_{t-1,h})(\cdot| s,\pi_{t,h}(s))(V^{\pi_t}_{h+1} - V^*_{h+1}) \tag{$\cap_t E^{pv1}(t)$ holds}\\
    &\overset{(a)}{\geq}  -b^{pv1}_{t-1,h}(s,\pi_{t,h}(s)) - \frac{18H^2 SL_{t,\delta}}{n_{t-1,h}(s,\pi_{t,h}(s))\vee 1} - \frac{1}{32H}\E_{\bar{P}_{t-1,h}(\cdot|s,\pi_{t,h}(s))}\brs*{(V^{\pi_t}_{h+1} - V^*_{h+1})(s')} \\
    &\geq  -b^{pv1}_{t-1,h}(s,\pi_{t,h}(s)) - \frac{18H^2SL_{t,\delta}}{n_{t-1,h}(s,\pi_{t,h}(s))\vee 1} - \frac{1}{32H}\E_{\bar{P}_{t-1,h}(\cdot|s,\pi_{t,h}(s))}\brs*{(\bar{V}_{t,h+1} - \ubar{V}_{t-1,h+1})(s')} \tag{Induction hypothesis}\\
    &\geq - \sqrt{\frac{2\VAR_{P_h(\cdot| s,\pi_{t,h}(s))}(V^*_{h+1})L_{t,\delta}}{n_{t-1,h}(s,\pi_{t,h}(s))\vee 1 }}- \frac{23H^2SL_{t,\delta}}{n_{t-1,h}(s,\pi_{t,h}(s))\vee 1} - \frac{1}{32H}\E_{\bar{P}_{t-1,h}(\cdot|s,\pi_{t,h}(s))}\brs*{(\bar{V}_{t,h+1} - \ubar{V}_{t-1,h+1})(s')} \tag{Plugging $b^{pv1}_{t-1,h}$ and elementary bounds},
\end{align*}
where $(a)$ holds by applying Lemma~\ref{lemma: transition different to next state expectation} while setting $\alpha=32H, C_1=2L_{t,\delta},C_2=2 L_{t,\delta}$ and bounding $HS(C_2+ \alpha C_1/4)\le 18H^2S L_{t,\delta}$ (assumption holds since $\cap_t E^{p}(t)$ holds). Plugging this back into~\eqref{eq: optimism base case ucbvi UL rel 3, pessimsm} and plugging the explicit form of the bonus $b^p_{t,h}(s,a)$ we get
\begin{align*}
    \eqref{eq: optimism base case ucbvi UL rel 3, pessimsm} 
    &= \sqrt{2L_{t,\delta}}\frac{\sqrt{\VAR_{\bar{P}_{t-1,h}(\cdot| s,\pi_{t,h}(s))}(\bar{V}_{t,h+1})} - \sqrt{\VAR_{P_h(\cdot| s,\pi_{t,h}(s))}(V^*_{h+1})}}{\sqrt{n_{t-1,h}(s,\pi_{t,h}(s))\vee 1}} \\
    &\quad + \frac{21H^2SL_{t,\delta}}{n_{t-1,h}(s,\pi_{t,h}(s))\vee 1} +  \frac{1}{32H}\E_{\bar{P}_{t-1,h}(\cdot|s,\pi_{t,h}(s))}\brs*{\bar{V}_{t,h+1}(s') - \ubar{V}_{t-1,h+1}(s')}\\
    &\overset{(a)}{\geq} -\frac{1}{32H}\E_{\bar{P}_{t-1,h}(\cdot|s,\pi_{t,h}(s))}\brs*{\bar{V}_{t,h+1}(s')- V^*_{h+1}(s')} - \frac{21H^2 L_{t,\delta}}{n_{t-1,h}(s,\pi_{t,h}(s))}\\
    &\quad +  \frac{1}{32H}\E_{\bar{P}_{t-1,h}(\cdot|s,\pi_{t,h}(s))}\brs*{\bar{V}_{t,h+1}(s') - \ubar{V}_{t-1,h+1}(s')} + \frac{21H^2S L_{t,\delta}}{n_{t-1,h}(s,\pi_{t,h}(s))}\geq 0
\end{align*}
where $(a)$ holds by Lemma~\ref{lemma: variance diff is upper bounded by value difference} while setting $\alpha=32H$ and bounding $(5H+ H\alpha/2)L_{t,\delta}\leq 21H^2 L_{t,\delta}$. Combining all the above we concludes the proof as
\begin{align*}
    V^{\pi_t}_h(s)- \ubar{V}_{t,h}(s)\geq\eqref{eq: optimism base case ucbvi UL rel 2, pessimsm}\geq \eqref{eq: optimism base case ucbvi UL rel 3, pessimsm}\geq 0.
\end{align*}
\end{proof}

\clearpage
\subsection{The Good Event}
We now prove an additional high probability bound which holds alongside first good event $\G_1$. 
\begin{lemma}[The Good Event]
Let $\G_1$ be the event defined in Lemma~\ref{lemma: the first good event RL UL}. Let $\brc*{Y_{1, t ,h}, Y_{2, t, h}}_{t\geq 1}$ the following random variables.
\begin{align*}
    &Y_{1, t ,h} \eqdef \bar{V}_{t,h+1}(s_{t,h+1}) - \ubar{V}_{t,h+1}(s_{t,h+1})\\
    &Y_{2, t, h} = \VAR_{P_h(\cdot|s_{t,h},a_{t,h})}( V^{\pi_t}_{h+1}).
\end{align*}
The second good event is the intersection of two events $\G_2 =E^{OP} \cap  E^{\VAR}$ defined as follows.
\begin{align*}
    &E^{OP}=\brc*{\forall h\in[H], T\geq 1:\ \sum_{t=1}^T \E[Y_{1, t ,h}|F_{t,h-1}]\leq \br*{1+\frac{1}{2H}} \sum_{t=1}^T Y_{1, t ,h} + 18H^2 \log\frac{4HT(T+1)}{\delta}}\\
    &E^{\VAR}= \brc*{ T\geq 1:\  \sum_{t=1}^T \sum_{h=1}^H Y_{2, t, h}\leq 2\sum_{t=1}^T \sum_{h=1}^H\E[Y_{2, t, h}|F_{t-1}]  + 4H^3 \log\frac{4HT(T+1)}{\delta}},
\end{align*}
Then, the good event $\G = \G_1\cap \G_2$ holds with probability greater than $1-\delta$.
\end{lemma}
\begin{proof}
{\bf Event $E^{OP}$.} The proof follows the lines of \Cref{lemma: the good event cbm rl} while noting that $0\leq \bar{V}_{t,h+1}(s_{t,h+1}) - \ubar{V}_{t,h+1}(s_{t,h+1})\leq H$ (replacing $\bar{V}_{t,h+1}(s_{t,h+1}) - V^{\pi_t}_{h}(s_{t,h})$ as in Lemma~\ref{lemma: the good event cbm rl}). This holds conditioned on the first good event due to the optimism-pessimism lemma (Lemma~\ref{lemma: optimism cbm-ucbvi-UL}). Notice that we replaced $\delta\to\delta/2$, and therefore, the proof results with $\Pr(\overline{E^{OP}}\cap \G_1)\le\frac{\delta}{4}$.

{\bf Event $E^{\VAR}$.} Fix $h\in [H]$. Set $Y_t = Y_{2, t, h} = \VAR_{P_h(\cdot|s_{t,h},a_{t,h})}( V^{\pi_t}_{h+1})$, and the filtration as $\brc*{F_{t}}_{t\geq 0}$ (observe that $Y_t$ is $F_{t}$ measurable). Furthermore, see that $0\leq Y_{2, t, h}\leq H$ a.s. . Applying the second statement of \Cref{lemma: consequences of optimism and freedman's inequality} we get that
\begin{align*}
    \sum_{t=1}^T   Y_{2, t, h} \leq  2\sum_{t=1}^T\E[Y_{2, t, h}|F_{t-1}]  + 4H^2 \log\frac{1}{\delta}.
\end{align*}
By taking union bound, as in the proof of the first statement of the lemma and on all $h\in [H]$ and summing over $h\in [H]$, we get that with probability greater than $1-\delta/4$ for all $T\geq1$ it holds that
\begin{align*}
    \sum_{t=1}^T \sum_{h=1}^H  Y_{2, t, h} \leq  2\sum_{t=1}^T\sum_{h=1}^H \E[Y_{2, t, h}|F_{t-1}]  + 4H^3 \log\frac{4HT(T+1)}{\delta}. 
\end{align*}

{\bf Combining all the above}
We bound the probability of $\overline{G}$ as follows:
\begin{align*}
    \Pr(\overline{\G}) \le \Pr(\overline{\G_1}) + \Pr(\overline{E^{OP}}\cap\G_1) + \Pr(\overline{E^{\VAR}})  \le \frac{\delta}{2} + \frac{\delta}{4} +\frac{\delta}{4}= \delta.
\end{align*}
\end{proof}

\subsection{CBM-ULCBVI: Budget Constraint is Satisfied}
Similarly to all algorithms that follow the CBM paradigm, CBM-ULCBVI does not violate the budget constraint. That is, it queries reward feedback only if there is an available budget. The proof of this result is identical to the one of \Cref{lemma: cbm-ucbvi budget constraint is satisfied} (i.e., for the CBM-UCBVI algorithm), since the reward bonus and query rule for rewards of the two algorithms are the same, except to the log-factors. We restate the result here for convenience.
\begin{lemma}[CBM-ULCBVI: Budget Constraint is Satisfied]\label{lemma: cbm-ucbvi-UL budget constraint is satisfied}
For any $T\geq 1$ the budget constraint is not violated, that is $\Bq(T)\leq B(T)$ almost surely.
\end{lemma}

\clearpage

\subsection{Proof of Theorem~\ref{theorem appendix: CBM UCBVI UL}}
As in the proof of CBM-UCBVI, before establishing the proof of Theorem~\ref{theorem appendix: CBM UCBVI UL} we establish the following key lemma that bounds the on-policy errors at time step $h$ by the on-policy errors at time step $h+1$ and additional additive terms. Given this result, the analysis follows with relative ease.

\begin{lemma}[CBM-ULCBVI, Key Recursion Bound]\label{lemma: key recursion bound UL}
Conditioned on the good event $\G$, the following bound holds for all $h\in [H]$. 
\begin{align*}
    \sum_{t=1}^{T} \bar{V}_{t,h}(s_{t,h}) - \ubar{V}_{t,h}(s_{t,h})
    &\leq  27H^2\log\br*{\frac{4HT(T+1)}{\delta}}+ 2\sum_{t=1}^T b^r_{t,h}(s_{t,h},a_{t,h})  +\sum_{t=1}^T\frac{224H^2S L_{t,\delta}}{ n_{t-1,h}(s_{t,h},a_{t,h})\vee 1} \\
    &\quad+ 2\sqrt{2L_{t,\delta}}\frac{\sqrt{\VAR_{P_{h}(\cdot|s_{t,h},a_{t,h})}(V^{\pi_t}_{h+1})}}{\sqrt{n_{t,h}(s_{t,h},a_{t,h})\vee 1}} + \br*{1+ \frac{1}{2H}}^2\sum_{t=1}^T \bar{V}_{t,h+1}(s_{t,h+1}) - \ubar{V}_{t,h+1}(s_{t,h+1}).
\end{align*}
\end{lemma}
\begin{proof}
We bound each of the terms in the sum as follows.
\begin{align}
    &\bar{V}_{t,h}(s_{t,h}) - \ubar{V}_{t,h}(s_{t,h}) \nonumber \\
    &= 2b^r_{t,h}(s_{t,h},a_{t,h}) + 2b^{p}_{t-1}(s_{t,h},a_{t,h}) + \br*{1+\frac{1}{16H}}\E_{\bar{P}_{t-1,h}(\cdot| s_{t,h},a_{t,h})}[ \bar{V}_{t,h+1}(s_{h+1}) -  \ubar{V}_{t,h+1}(s_{h+1})]  \nonumber \\
    &= 2b^r_{t,h}(s_{t,h},a_{t,h}) + 2b^{p}_{t-1}(s_{t,h},a_{t,h}) + \br*{1+\frac{1}{16H}}\E[ \bar{V}_{t,h+1}(s_{h+1}) -  \ubar{V}_{t,h+1}(s_{h+1}) | F_{t,h-1}] \nonumber\\
    &\quad + \br*{1+\frac{1}{16H}}(\bar{P}_{t-1,h}-P_h)(\cdot |s,a)^T \br*{\bar{V}_{t,h+1} -  \ubar{V}_{t,h+1}}\nonumber\\
    &\leq 2b^r_{t,h}(s_{t,h},a_{t,h}) + 2b^{p}_{t-1}(s_{t,h},a_{t,h}) + \frac{6SH^2L_{t,\delta}}{n_{t-1,h}(s,a)\vee 1}+ \br*{1+\frac{1}{4H}}\E[ \bar{V}_{t,h+1}(s_{h+1}) -  \ubar{V}_{t,h+1}(s_{h+1}) |  F_{t,h-1}] \label{eq: central theorem UL RL relation 1},
\end{align}
where the last relation holds by Lemma~\ref{lemma: transition different to next state expectation} while setting $\alpha=8H,C_1=C_2=2L_{t,\delta}$ and bounding $HS( C_2+ \alpha S C_1/4)\le 6SH^2 L_{t,\delta}$ (the assumption of the lemma holds since the event $\cap_t E^p(t)$ holds). 
By Lemma~\ref{lemma: bound on bonus bp UL RL} it holds that
\begin{align*}
    \sum_{t=1}^T b^p_{t,h}(s_{t,h},a_{t,h})\leq &\sum_{t=1}^T\frac{109H^2S L_{t,\delta}}{ n_{t-1,h}(s_{t,h},a_{t,h})\vee 1} + \sqrt{2L_{t,\delta}}\frac{\sqrt{\VAR_{P_{h}(\cdot|s_{t,h},a_{t,h})}(V^{\pi_t}_{h+1})}}{\sqrt{n_{t,h}(s_{t,h},a_{t,h})\vee 1}} \\
    & + \frac{1}{8H}\sum_{t=1}^T \E[\bar{V}_{t,h+1}(s_{t,h+1}) - \ubar{V}_{t,h+1}(s_{t,h+1})| F_{t,h-1}].
\end{align*}
Plugging this into~\eqref{eq: central theorem UL RL relation 1} and rearranging the terms we get
\begin{align*}
     \sum_{t=1}^T\bar{V}_{t,h}(s_{t,h}) - \ubar{V}_{t,h}(s_{t,h})
     &\leq 2\sum_{t=1}^T b^r_{t,h}(s_{t,h},a_{t,h})  + \sum_{t=1}^T\frac{224H^2S L_{t,\delta}}{ n_{t-1,h}(s_{t,h},a_{t,h})\vee 1} + 2\sqrt{2L_{t,\delta}}\frac{\sqrt{\VAR_{P_{h}(\cdot|s_{t,h},a_{t,h})}(V^{\pi_t}_{h+1})}}{\sqrt{n_{t,h}(s_{t,h},a_{t,h})\vee 1}} \\
    &\quad+ \br*{1+ \frac{1}{2H}}\sum_{t=1}^T \E[\bar{V}_{t,h+1}(s_{t,h+1}) - \ubar{V}_{t,h+1}(s_{t,h+1})| F_{t,h-1}]\\
    &\leq  27H^2\log\br*{\frac{4HT(T+1)}{\delta}}+ 2\sum_{t=1}^T b^r_{t,h}(s_{t,h},a_{t,h})  +\sum_{t=1}^T\frac{224H^2S L_{t,\delta}}{ n_{t-1,h}(s_{t,h},a_{t,h})\vee 1} \\
    &\quad+ 2\sqrt{2L_{t,\delta}}\frac{\sqrt{\VAR_{P_{h}(\cdot|s_{t,h},a_{t,h})}(V^{\pi_t}_{h+1})}}{\sqrt{n_{t,h}(s_{t,h},a_{t,h})\vee 1}} + \br*{1+ \frac{1}{2H}}^2\sum_{t=1}^T \bar{V}_{t,h+1}(s_{t,h+1}) - \ubar{V}_{t,h+1}(s_{t,h+1}), \tag{event $E^{OP}$ holds (the second good event)}
\end{align*}
where in the last relation we also bounded $18\br*{1+\frac{1}{2H}}\le 27$.
\end{proof}

\clearpage
We are now ready to establish Theorem~\ref{theorem appendix: CBM UCBVI UL}.
\begin{proof}
Start by conditioning on the good event which holds with probability greater than $1-\delta$. Applying the optimism-pessimism of the upper and lower value function we get
\begin{align}
    \sum_{t=1}^{T} V_{1}^*(s_{t,1}) - V_{1}^{\pi_t}(s_{t,1})  \leq \sum_{t=1}^{T} \bar{V}_{t,1}(s_{t,1}) - \ubar{V}_{t,1}(s_{t,1}). \label{eq: central thm UL RL 1 relation}
\end{align}
Iteratively applying \Cref{lemma: key recursion bound UL} and bound the exponential growth by $\br*{1+\frac{1}{2H}}^{2h}\leq e\leq 3$ for any $h\in\brc*{0,\dots,H}$, The following upper bound on the cumulative regret is obtained.
\begin{align}
    &\eqref{eq: central thm UL RL 1 relation} \leq 81H^3 \log\br*{ \frac{HT(T+1)}{\delta}} + 6\sum_{t=1}^T \sum_{h=1}^H b^r_{t,h}(s_{t,h},a_{t,h}) + \sum_{t=1}^T \sum_{h=1}^H\frac{ 672 H^2SL_{t,\delta}}{n_{t-1,h}^{p}( s_{t,h},a_{t,h})\vee 1} \nonumber \\
    &\quad\quad +9\sum_{t=1}^T \sum_{h=1}^H \frac{\sqrt{L_{t,\delta}\VAR_{P_h(\cdot|s_{t,h},a_{t,h})}(V^{\pi_t}_{h+1}) }}{\sqrt{n_{t-1,h}(s_{t,h},a_{t,h})}}. \label{eq: rl final bound relation 2 UL}
\end{align}
The first sum in~\eqref{eq: rl final bound relation 2 UL} is bounded in \Cref{lemma: bound on cummulative reward bonus RL} (the reward bonus is exactly the same as for CBM-UCBVI which implies the same upper bound on the sum of bonuses hold) by
\begin{align*}
    &6\sum_{t=1}^T \sum_{h=1}^H b^r_{t,h}(s_{t,h},a_{t,h})\\
    &\leq 3L_{t,\delta}\br*{6  \sqrt{|\LR| HT  } + 10 SAH\log(HT) + 23SAH + \sum_{t=1}^T\br*{6\sqrt{\frac{|\LR|H^2}{B(t)}} + 4SAH^2\frac{\log(1+ B(t))+1}{B(t)}}}.
\end{align*}
The second sum is bounded via standard analysis as follows:
\begin{align*}
    \sum_{t=1}^T \sum_{h=1}^H\frac{ 672 H^2SL_{t,\delta}}{n_{t-1,h}( s_{t,h},a_{t,h})\vee 1} &\leq  672 H^2SL_{T,\delta} \sum_{t=1}^T \sum_{h=1}^H\frac{1}{n_{t-1,h}( s_{t,h},a_{t,h})\vee 1} \tag{$L_{t,\delta}$ increasing in $t$}\\
    &=  672 H^2S L_{T,\delta} \sum_{s,a,h} \sum_{i=0}^{n_{T,h}(s,a)} \frac{1}{i\vee 1} \tag{Reorganizing summation}\\
    &\leq  672 H^2S L_{T,\delta} \sum_{s,a,h} (2+\log(n_{T,h}(s,a)\vee1))\\
    &\leq 672 H^3S^2AL_{T,\delta}(2+\log(TH)) \tag{Jensen's inequality and $\sum_{s,a,h} n_{T,h}(s,a)=HT$}
    \\& \leq  2688 H^3S^2AL_{T,\delta}\log(TH+1)
\end{align*}
The third sum in~\eqref{eq: rl final bound relation 2 UL} is bounded in Lemma~\ref{lemma: bound on variance term} by
\begin{align*}
    &9\sum_{t=1}^T \sum_{h=1}^H \frac{\sqrt{L_{t,\delta}\VAR_{P_h(\cdot|s_{t,h},a_{t,h})}(V^{\pi_t}_{h+1}) }}{\sqrt{n_{t-1,h}(s_{t,h},a_{t,h})}} \leq 9 \sqrt{L_{T,\delta}}\sum_{t=1}^T \sum_{h=1}^H \frac{\sqrt{\VAR_{P_h(\cdot|s_{t,h},a_{t,h})}(V^{\pi_t}_{h+1}) }}{\sqrt{n_{t-1,h}(s_{t,h},a_{t,h})}} \tag{$L_{t,\delta}$ increasing in $t$}\\
    &\leq  27\sqrt{ SAH^3 T \log(TH+1)L_{T,\delta}} +  36H^2\sqrt{SA \log(TH+1) \log\br*{\frac{4HT(T+1)}{\delta}}L_{T,\delta}}
\end{align*}
Combining the above bounds with proper simplification yields a bound on~\eqref{eq: rl final bound relation 2 UL} and concludes the proof,
\begin{align*}
    \eqref{eq: rl final bound relation 2 UL}\leq  45 L_{T,\delta}^2 \br*{\sqrt{SAH^3 T} + \sum_{t=1}^T\br*{\sqrt{\frac{|\LR|H^2}{B(t)}} + SAH^2\frac{\log(1+ B(t))+1}{B(t)}}}+ 2904 H^3S^2A L^2_{T,\delta}.
\end{align*}
\end{proof}
\clearpage

\subsection{Results that Hold Conditioned on the Good Event}
\begin{lemma}[Bound on the Cumulative Transition Model Bonus]\label{lemma: bound on bonus bp UL RL}
Conditioning on the good event $\G_1$ the following bound holds for all $h\in [H]$.
\begin{align*}
    \sum_{t=1}^T b^{p}_t(s_{t,h},a_{t,h}) &\leq \sum_{t=1}^T\frac{109H^2SL_{t,\delta}}{ n_{t-1,h}(s_{t,h},a_{t,h})\vee 1} + \sqrt{2L_{t,\delta}}\frac{\sqrt{\VAR_{P_{h}(\cdot|s_{t,h},a_{t,h})}(V^{\pi_t}_{h+1})}}{\sqrt{n_{t-1,h}(s_{t,h},a_{t,h})\vee 1}} \\
    & + \frac{1}{8H}\sum_{t=1}^T \E[\bar{V}_{t,h+1}(s') - \ubar{V}_{t,h+1}(s')|F_{t,h-1}].
\end{align*}
\end{lemma}
\begin{proof}
First, observe that
\begin{align}
    &\E_{\bar{P}_{t-1,h}(\cdot|s,a)}\brs{\bar{V}_{t,h+1}(s') - \ubar{V}_{t,h+1}(s')} \nonumber \\
    & = \E_{P_{h}(\cdot|s,a)}\brs{\bar{V}_{t,h+1}(s') - \ubar{V}_{t,h+1}(s')} + (\bar{P}_{t-1,h}- P_h)(\cdot|s,a)^T\brs*{\bar{V}_{t,h+1}(s')- \ubar{V}_{t,h+1}(s')} \nonumber\\
    &\leq \frac{9}{8}\E_{P_{h}(\cdot|s,a)}\brs{\bar{V}_{t,h+1}(s') - \ubar{V}_{t,h+1}(s')} + \frac{ 6H^2S L_{t,\delta}}{n_{t-1,h}(s,a)\vee 1}, \label{eq: useful relation bound on bonus bp}
\end{align}
by applying Lemma~\ref{lemma: transition different to next state expectation} with $\alpha=8H,C_1=C_2=2L_{t,\delta}$ and $HS(C_2+ \alpha C_1/4)\le 6H^2S$ (applicable since $\cap_t E^p(t)$ holds).

The bonus $b^p_{t,h}(s,a)$ can be upper bounded as follows.
\begin{align}
    &b^p_{t,h}(s,a)\leq \sqrt{2}\sqrt{\frac{\VAR_{\bar{P}_{t-1,h}(\cdot| s,a)}(\bar{V}_{t,h+1}) L_{t,\delta}}{n_{t-1,h}(s,a)\vee 1}}  + \frac{1}{16H}\E_{\bar{P}_{t-1,h}(\cdot|s,a)}\brs{\bar{V}_{t,h+1}(s') - \ubar{V}_{t,h+1}(s')} + \frac{ 44H^2S L_{t,\delta}}{n_{t-1,h}(s,a)\vee 1} \nonumber \\
    &\leq  \sqrt{2}\sqrt{\frac{\VAR_{\bar{P}_{t-1,h}(\cdot| s,a)}(\bar{V}_{t,h+1}) L_{t,\delta}}{n_{t-1,h}(s,a)\vee 1}}  + \frac{9}{128H}\E_{P_{h}(\cdot|s,a)}\brs{\bar{V}_{t,h+1}(s') - \ubar{V}_{t,h+1}(s')} + \frac{50H^2S L_{t,\delta}}{n_{t-1,h}(s,a)\vee 1}\label{eq:  bound on bp UL RL relation 1},
\end{align}
by~\eqref{eq: useful relation bound on bonus bp}. We bound the first term to establish the lemma. 

{\bf Bound on the first term of~\eqref{eq:  bound on bp UL RL relation 1}} It holds that
\begin{align*}
    &\sqrt{2L_{t,\delta}}\sqrt{\frac{\VAR_{\bar{P}_{t-1,h}(\cdot| s,a)}(\bar{V}_{t,h+1}) }{n_{t-1,h}(s,a)\vee 1}}\\
    &=\underbrace{\sqrt{2L_{t,\delta}}\frac{\sqrt{\VAR_{\bar{P}_{t-1,h}(\cdot| s,a)}(\bar{V}_{t,h+1})} -  \sqrt{\VAR_{P_{h}(\cdot| s,a)}(V^*_{h+1})} }{\sqrt{n_{t-1,h}(s,a)\vee 1}} }_{(i)} +\underbrace{\sqrt{2L_{t,\delta}}\frac{ \sqrt{\VAR_{P_{h}(\cdot| s,a)}(V^*_{h+1})} - \sqrt{\VAR_{P_{h}(\cdot| s,a)}(V^{\pi_t}_{h+1})}}{\sqrt{n_{t-1,h}(s,a)\vee 1}}}_{(ii)} \\
    &\quad\quad +\frac{\sqrt{2L_{t,\delta}}\sqrt{\VAR_{P_{h}(\cdot| s,a)}(V^{\pi_t}_{h+1})}}{\sqrt{n_{t-1,h}(s,a)\vee 1}}.
\end{align*}
Term $(i)$ is bounded by Lemma~\ref{lemma: variance diff is upper bounded by value difference} (by setting $\alpha=32H$ and $(5H+ H\alpha/2)\leq 21H^2$),
\begin{align*}
   \sqrt{2L_{t,\delta}}\frac{\sqrt{\VAR_{\bar{P}_{t-1,h}(\cdot| s,a)}(\bar{V}_{t,h+1})} -  \sqrt{\VAR_{P_{h}(\cdot| s,a)}(V^*_{h+1})} }{\sqrt{n_{t-1,h}(s,a)\vee 1}} \leq   \frac{1}{32H}\E_{\bar{P}_{t-1,h}(\cdot|s,a)}\brs*{\bar{V}_{t,h+1}(s')- V^*_{h+1}(s')} + \frac{21H^2L_{t,\delta}}{ n_{t-1,h}(s,a)\vee 1}.
\end{align*}
Following the same steps as in~\eqref{eq: useful relation bound on bonus bp}, we get
\begin{align*}
    &\E_{\bar{P}_{t-1,h}(\cdot|s,a)}\brs*{\bar{V}_{t,h+1}(s')- V^*_{h+1}(s')} \leq \frac{9}{8}\E_{P_{h}(\cdot|s,a)}\brs*{\bar{V}_{t,h+1}(s')- V^*_{h+1}(s')} + \frac{6H^2S L_{t,\delta}}{n_{t-1,h}(s,a)\vee 1},
\end{align*}
and, thus, 
\begin{align*}
    (i) \leq \frac{9}{256H}\E_{P_{h}(\cdot|s,a)}\brs*{\bar{V}_{t-1,h+1}(s')- V^*_{h+1}(s')} + \frac{27 H^2S L_{t,\delta}}{ n_{t-1,h}(s,a)\vee 1}.
\end{align*}
Term $(ii)$ is bounded as follows.
\begin{align*}
    &(ii)\leq \sqrt{2L_{t,\delta}}\frac{ \sqrt{\VAR_{P_{h}(\cdot| s,a)}(V^*_{h+1} - V^{\pi_t}_{h+1})}}{\sqrt{n_{t-1,h}(s,a)\vee 1}} \tag{By Lemma~\ref{lemma: std difference}}\\
    &\leq \sqrt{2L_{t,\delta}}\frac{ \sqrt{\E_{P_{h}(\cdot|s,a)}[(V^*_{h+1}(s') - V^{\pi_t}_{h+1}(s'))^2]}}{\sqrt{n_{t-1,h}(s,a)\vee 1}}\\
    &\leq \sqrt{2L_{t,\delta}}\frac{ \sqrt{H\E_{P_{h}(\cdot|s,a)}[(V^*_{h+1}(s') - V^{\pi_t}_{h+1}(s'))]}}{\sqrt{n_{t-1,h}(s,a)\vee 1}} \tag{ $0 \leq V^*_{h+1}(s') - V^{\pi_t}_{h+1}(s')\leq H$ }\\
    &\leq \frac{1}{64H}\E_{P_{h}(\cdot|s,a)}[(V^*_{h+1}(s') - V^{\pi_t}_{h+1}(s'))] + \frac{32H^2L_{t,\delta}}{n_{t-1,h}(s,a)\vee 1}. \tag{$ab\leq \frac{1}{\alpha}a^2 + \frac{\alpha}{4}b^2$ for $\alpha=64H$}
\end{align*}
Thus, applying $\ubar{V}_{h+1}\leq V^{\pi_t}_{h+1}\leq V^*_{h+1}\leq \bar{V}_{h+1}$  (Lemma~\ref{lemma: optimism cbm-ucbvi-UL}) in the bounds of $(i)$ and $(ii)$ we get
\begin{align*}
    &b^p_{t,h}(s,a)\leq  \frac{1}{8H}\E_{P_{h}(\cdot|s,a)}[(\bar{V}_{t,h}(s') - \ubar{V}_{t,h}(s'))]
    +\frac{109H^2S L_{t,\delta}}{ n_{t-1,h}(s,a)\vee 1} + \frac{\sqrt{2L_{t,\delta}}\sqrt{\VAR_{P_{h}(\cdot| s,a)}(V^{\pi_t}_{h+1})}}{\sqrt{n_{t-1,h}(s,a)\vee 1}}.
\end{align*}
and summing over $t$ concludes the proof.
\end{proof}

\begin{lemma}[Bound on Variance Term]\label{lemma: bound on variance term}
Conditioning on the good event $E^{\VAR}$ it holds that
\begin{align*}
     &\sum_{t=1}^T\sum_{h=1}^H \frac{\sqrt{\VAR_{P_h(\cdot|s_{t,h},a_{t,h})}(V^{\pi_t}_{h+1}) }}{\sqrt{n_{t-1,h}(s_{t,h},a_{t,h})}}
     \leq  3\sqrt{ SAH^3 T \log(TH+1)} +  4H^2\sqrt{SA \log(TH+1) \log\br*{\frac{4HT(T+1)}{\delta}}}.
\end{align*}
\end{lemma}
\begin{proof}
Applying Cauchy-Schwartz inequality we get
\begin{align*}
    &\sum_{t=1}^T\sum_{h=1}^H \frac{\sqrt{\VAR_{P_h(\cdot|s_{t,h},a_{t,h})}(V^{\pi_t}_{h+1}) }}{\sqrt{n_{t-1,h}(s_{t,h},a_{t,h}) \vee 1}}\\
    &\leq\sqrt{\sum_{t=1}^T\sum_{h=1}^H  \VAR_{P_h(\cdot|s_{t,h},a_{t,h})}(V^{\pi_t}_{h+1})} \sqrt{\sum_{t=1}^T\sum_{h=1}^H \frac{1}{n_{t-1,h}(s_{t,h},a_{t,h})\vee 1}} \tag{Cauchy-Schwarz inequality}\\
    &\leq \sqrt{\sum_{t=1}^T\sum_{h=1}^H  \VAR_{P_h(\cdot|s_{t,h},a_{t,h})}(V^{\pi_t}_{h+1})} \sqrt{SAH (2+\log(TH))} \tag{By \cref{eq: standard RL analysis sum of 1/n}}\\
    &\leq   2\sqrt{2\sum_{t=1}^T \E\brs*{\sum_{h=1}^H  \VAR_{P_h(\cdot|s_{t,h},a_{t,h})}(V^{\pi_t}_{h+1})|  F_{t-1} } +4H^3\log\br*{\frac{4HT(T+1)}{\delta}}}\sqrt{SAH \log(TH+1)} \tag{$E^{\VAR}$ holds} \\
    &\leq 3\sqrt{\sum_{t=1}^T \E\brs*{\sum_{h=1}^H  \VAR_{P_h(\cdot|s_{t,h},a_{t,h})}(V^{\pi_t}_{h+1})| F_{t-1}}} \sqrt{SAH \log(TH+1)} +  4H^2\sqrt{SA \log(TH+1) \log\br*{\frac{4HT(T+1)}{\delta}}} \tag{$\sqrt{a+b}\leq \sqrt{a} + \sqrt{b}$}\\
    &=3\sqrt{\sum_{t=1}^T \E\brs*{\br*{ V_1^{\pi_t}(s_1) - \sum_{h=1}^H r_h(s_{t,h},a_{t,h})}^2 \bigg\vert F_{t-1}}} \sqrt{SAH \log(TH+1)} +  4H^2\sqrt{SA \log(TH+1)) \log\br*{\frac{4HT(T+1)}{\delta}}} \tag{Law of total variance~\cite{azar2017minimax}, see Lemma~\ref{lemma: law of total variance for RL} }\\
    &\leq  3\sqrt{ SAH^3 T \log(TH+1)} +  4H^2\sqrt{SA\log(TH+1) \log\br*{\frac{4HT(T+1)}{\delta}}} . \tag{$V_1^{\pi_t}(s)\in[0,H],\ r_h(s,a)\in[0,1]$}
\end{align*}
\end{proof}

\begin{lemma}[Variance Difference is Upper Bounded by Value Difference]\label{lemma: variance diff is upper bounded by value difference}
Assume that the value at time step $h+1$ is optimistic, $\bar{V}_{t,h+1}(s)\geq V^*_{h+1}(s)$ for all $s\in \Scal$. Conditioning on the event $\cap_t E^{pv2}(t)$ it holds for all $s,a\in \Scal\times \Acal,\in[H]$ that
\begin{align*}
  \sqrt{2L_{t,\delta}} \frac{ \abs*{\sqrt{\VAR_{\bar{P}_{t-1,h}(\cdot| s,a)}(\bar{V}_{t,h+1})} - \sqrt{\VAR_{P_h(\cdot| s,a)}(V^*_{h+1})}}}{\sqrt{n_{t-1,h}(s,a)}}\leq  \frac{1}{\alpha}\E_{\bar{P}_{t-1,h}(\cdot|s,a)}\brs*{\bar{V}_{t,h+1}(s')- V^*_{h+1}(s')} + \frac{(5H+ H\alpha/2)L_{t,\delta}}{n_{t-1,h}(s,a)\vee 1}.
\end{align*}
\end{lemma}
\begin{proof}
Conditioning on the first good event the following relations hold.
\begin{align*}
     & \abs*{\sqrt{\VAR_{\bar{P}_{t-1,h}(\cdot| s,a)}(\bar{V}_{t,h+1})} - \sqrt{\VAR_{P_h(\cdot| s,a)}(V^*_{h+1})}}\\
     &\leq \abs*{\sqrt{\VAR_{\bar{P}_{t-1,h}(\cdot| s,a)}(\bar{V}_{t,h+1})} - \sqrt{\VAR_{\bar{P}_{t-1,h}(\cdot| s,a)}(V^*_{h+1})}} + \sqrt{\frac{12H^2 L_{t,\delta}}{n_{t-1,h}(s,a) \vee 1}}  \tag{$\cap_t E^{pv2}(t)$ holds}  \nonumber\\
    &\leq \sqrt{\VAR_{\bar{P}_{t-1}(\cdot| s,a)}(V^*_{h+1} - \bar{V}_{t,h+1})}  + \sqrt{\frac{12H^2 L_{t,\delta}}{n_{t-1,h}(s,a) \vee 1}} \tag{Lemma~\ref{lemma: std difference}}\\
    &\leq \sqrt{\E_{\bar{P}_{t-1,h}}\brs*{(V^*_{h+1}(s') - \bar{V}_{t,h+1}(s'))^2}}  + \sqrt{\frac{12H^2 L_{t,\delta}}{n_{t-1,h}(s,a) \vee 1}} \\
    &\leq \sqrt{H\E_{\bar{P}_{t-1,h}}\brs*{\bar{V}_{t,h+1}(s') - V^*_{h+1}(s')}}  + \sqrt{\frac{12H^2 L_{t,\delta}}{n_{t-1,h}(s,a) \vee 1}} ,
\end{align*}
where the last relation holds since $V^*_{h+1}(s'),\bar{V}_{t,h+1}(s')\in [0,H]$ (the first, by model assumption, and the second, by the update rule) and since $V^*_{h+1}(s')\leq \bar{V}_{t,h+1}(s')$ by the assumption the value is optimistic. Thus,
\begin{align*}
    &\sqrt{2L_{t,\delta}}\frac{ \abs*{\sqrt{\VAR_{\bar{P}_{t-1,h}(\cdot| s,a)}(\bar{V}_{t,h+1})} - \sqrt{\VAR_{P_h(\cdot| s,a)}(V^*_{h+1})}}}{\sqrt{n_{t-1,h}(s,a)}}\\
    &\leq \sqrt{\E_{\bar{P}_{t-1,h}}\brs*{\bar{V}_{t,h+1}(s') - V^*_{h+1}(s')}} \sqrt{\frac{2HL_{t,\delta}}{n_{t-1,h}(s,a)\vee 1}}  + \frac{\sqrt{24}H L_{t,\delta}}{n_{t-1,h}(s,a) \vee 1}\\
    &\leq  \frac{1}{\alpha} \E_{\bar{P}_{t-1,h}}\brs*{\bar{V}_{t,h+1}(s') - V^*_{h+1}(s')}   + \frac{(5H+H\alpha/2)L_{t,\delta}}{n_{t-1,h}(s,a) \vee 1} \tag{Youngs inequality, $ab\leq \frac{1}{\alpha}a^2 + \frac{\alpha}{4}b^2$}.
\end{align*}
\end{proof}

\newpage
\section{Useful Results}
\begin{lemma}[Consequences of Freedman's Inequality for Bounded and Positive Sequence of Random Variables]\label{lemma: consequences of optimism and freedman's inequality}
Let $\brc{Y_t}_{t\geq 1}$ be a real valued sequence of random variables adapted to a filtration $\brc*{F_t}_{t\geq 0}$. Assume that for all $t\geq 1$ it holds that $0\leq Y_{t}\leq C$ a.s., and $T\in \mathbb{N}$. Then each of the following inequalities hold with probability greater than $1-\delta$.
\begin{align*}
   &\sum_{t=1}^T \E[Y_t|F_{t-1}]\leq \br*{1+\frac{1}{2C}} \sum_{t=1}^T Y_t + 2(2C+1)^2 \log\frac{1}{\delta},\\
   &\sum_{t=1}^T Y_t \leq 2\sum_{t=1}^T \E[Y_t|F_{t-1}] + 4C\log\frac{1}{\delta}.
\end{align*}
\end{lemma}
\begin{proof}

{\bf First statement.} Let
$
X_{t} \eqdef \E[Y_t|F_{t-1}]  - Y_t.
$
Observe that $X_{t}$ is a martingale difference sequence w.r.t. to $\brc*{F_{t}}_{t\geq 0}$, and that $|X_{t}|\leq C$ a.s.. Furthermore, observe that
\begin{align}
    &\E[X^2_{t}|F_{t-1}] \leq 2\E[(\E[Y_t|F_{t-1}])^2|F_{t-1}]  + 2\E[Y_t^2|F_{t-1}] \tag{$(a-b)^2\leq 2a^2 + 2b^2$} \nonumber\\
    &\leq 4\E[Y_t^2|F_{t-1}] \tag{Jensen's inequality} \nonumber\\
    &\leq 4C \E[Y_t|F_{t-1}], \label{eq: second moment for freedman}
\end{align}
where the last relation holds since $Y_t\geq 0$ and $Y_t\leq C$.  Applying Freedman's inequality we get that
\begin{align*}
    \sum_{t=1}^T \br*{\E[Y_t|F_{t-1}]  - Y_t }
    = \sum_{t=1}^T X_{t}  
    &\leq  \eta\sum_{t=1}^T \E[X_{t}^2|F_{t-1}] +   \frac{\log(1/\delta)}{\eta} \tag{Lemma~\ref{lemma: freedmans inequality}}\\
    &\leq \eta \sum_{t=1}^T 4C\E[Y_t|F_{t-1}] + \frac{\log(1/\delta)}{\eta}. \tag{Eq.~\eqref{eq: second moment for freedman}}
\end{align*}
Choosing $\eta =\frac{1}{4C(2C+1)}\in (0,\frac{1}{C})$ and rearranging we get
\begin{align*}
    \sum_{t=1}^T \br*{1 - \frac{1}{2C+1}}\E[Y_t|F_{t-1}] 
    &= \sum_{t=1}^T \br*{\frac{2C}{2C+1}}\E[Y_t|F_{t-1}]\\
    &\leq \sum_{t=1}^T Y_t +  4C(2C+1)\log(1/\delta).
\end{align*}
Thus,
\begin{align*}
    \sum_{t=1}^T \E[Y_t|F_{t-1}] \leq \sum_{t=1}^T \br*{1+\frac{1}{2C}}Y_t +  2(2C+1)^2\log(1/\delta).
\end{align*}

{\bf Second statement.} Let
$
X_{t} \eqdef Y_t - \E[Y_t|F_{t-1}].
$
Similarly to before, it holds that $X_{t}$ is a martingale difference sequence w.r.t. to the filtration $F_{t-1}$, and that $|X_{t,h}|\leq C$ a.s. . Similarly to~\eqref{eq: second moment for freedman} which is sign invariant we get $\E[X^2_{t}|F_{t-1}] \leq 4C \E[Y_t|F_{t-1}]$. Applying Freedman's inequality we get
\begin{align*}
    \sum_{t=1}^T Y_t- \E[Y_t|F_{t-1}] 
    = \sum_{t=1}^T X_{t}
    &\leq \eta\sum_{t=1}^T \E[X_{t}^2|F_{t-1}] +   \frac{\log(1/\delta)}{\eta} \tag{Lemma~\ref{lemma: freedmans inequality}}\\
    &\leq  \eta\sum_{t=1}^T 4C \E[Y_t|F_{t-1}] +   \frac{\log(1/\delta)}{\eta}.
\end{align*}
Setting $\eta = 1/4C$ and rearranging concludes the proof of the second statement.
\end{proof}

\begin{lemma}[Transition Difference to Next State Expectation]\label{lemma: transition different to next state expectation}
Let $Y\in \mathbb{R}^{S}$ be a  vector such that $0\leq Y(s) \leq H$ for all $s\in \Scal$. Let $P_1$ and $P_2$ be two transition models and $n\in \mathbb{R}^{SA}_+$. Let $\Delta P_{h} (\cdot| s,a)\in \mathbb{R}^{S}$ and $\Delta P_{h} (s'| s,a)\eqdef  P_{1,h} (s'| s,a) -  P_{2,h} (s'| s,a)$. Assume that  
\begin{align*}
        \brc*{\forall s,a,s'\in \Scal\times\Acal\times \Scal, h\in [H]:\ |\Delta P_{h} (s'| s,a)| \le \sqrt{\frac{ C_1 L_{t,\delta} P_{1,h}(s'|s,a) }{n(s,a) \vee 1}} + \frac{C_2 L_{t,\delta}}{n(s,a)\vee 1}},
\end{align*}
for some $C_1,C_2>0$, then, for any $\alpha>0$,
$$
\abs*{\Delta P_{h} (\cdot| s,a)Y}\leq \frac{1}{\alpha} \E_{P_1(\cdot|s,a)}\brs*{ Y(s')} + \frac{H S L_{t,\delta}(C_2+ \alpha C_1/4)}{n(s,a)\vee 1},
$$
\end{lemma}
\begin{proof}
The following relations hold.
\begin{align*}
    \abs*{\Delta P_{h} (\cdot| s,a)Y}
    & \le \sum_{s'}\abs*{\Delta P_{h} (s'| s,a)}\cdot\abs*{Y(s')}\\
    & \leq \sum_{s'}\br*{\sqrt{\frac{C_1 P_{1,h}(s'| s,a )}{n(s,a)\vee 1}}Y(s') + \frac{H C_2}{n(s,a)\vee 1}} \tag{By assumption of the lemma \& $Y(s)\in \brs*{0,H}$}\\
    & =  \sum_{s'}\sqrt{\frac{C_1 P_{1,h}(s'| s,a)Y^2(s')}{n(s,a)\vee 1}} + \frac{H S C_2 }{n(s,a)\vee 1}\\
    &\leq \sum_{s'}\sqrt{P_{1,h}(s'| s,a)Y(s')}\sqrt{\frac{HC_1}{n(s,a)\vee 1}} + \frac{H S C_2  }{n(s,a)\vee 1} \tag{$0 \leq Y(s')\leq H$}\\
    &\leq \sqrt{\sum_{s'} P_{1,h}(s'| s,a)Y(s')}\sqrt{\frac{S HC_1}{n(s,a)\vee 1}} + \frac{H S C_2 }{n(s,a)\vee 1} \tag{
Cauchy–Schwarz inequality}\\
    &\overset{(*)}{\leq} \frac{1}{2\alpha }\sum_{s'} P_{1,h}(s'| s,a)Y(s')+ \frac{ \alpha S HC_1}{2n(s,a)\vee 1}+ \frac{H S C_2}{n(s,a)\vee 1},
\end{align*}
where $(*)$ is by Young's inequality, namely, $ab\leq \frac{a^2}{2\alpha} + \frac{\alpha}{2}b^2$ for any $\alpha>0$.
Re-scaling $2 \alpha =\alpha'$ we conclude the proof.

\end{proof}

\newpage
\section{Useful Existing Results}
\begin{lemma}[Freedman's Inequality, \citealt{beygelzimer2011contextual}, Theorem 1]\label{lemma: freedmans inequality}
Let $\brc{X_t}_{t\geq 1}$ be a real valued martingale difference sequence adapted to a filtration $\brc*{F_t}_{t\geq 0}$. If $|X_t|\leq R$ a.s. then for any $\eta\in (0,1/R], T\in \mathbb{N}$ it holds with probability greater than $1-\delta$,
\begin{align*}
    \sum_{t=1}^T X_t \leq \eta \sum_{t=1}^T \E[X_t^2| F_{t-1}] +\frac{\log(1/\delta)}{\eta}.
\end{align*}
\end{lemma}

\begin{theorem}[\citealt{abbasi2011improved}, Theorem 1]\label{theorem: abassi confidence self normalized martingale}
Let $\brc*{F_t}_{t=0}^\infty$ be a filtration. Let $\brc*{\eta_t}_{t=0}^\infty$ be a real-valued stochastic process such that $\eta_t$ is $F_t$-measurable and $\eta_t$ is conditionally $\sigma$-sub-Gaussian for $\sigma\geq 0$. Let $\brc*{x_t}_{t=0}^\infty$ be an $\R^d$-valued stochastic process s.t. $X_t$ is $F_{t-1}$-measurable and $\norm{x_t}\leq L$. Assume that $V$ is a $d\times d$ positive-definite matrix. For any $t\geq 0$, define $V_t = V + \sum_{s=1}^t x_s x_s^T$, and $S_t = \sum_{s=1}^t \eta_s X_s$. Then, for any $\delta>0$, with a probability of at least $1-\delta$, for all $t\geq 0$, $$\norm{S_t}^2_{V_t^{-1}} \leq 2\sigma^2 \log\br*{\frac{\det{V_t}^{1/2} \det{V}^{-1/2}}{\delta}}.$$
\end{theorem}
The following results is an adaptation of \citealt{abbasi2011improved}, Theorem 2. It establishes a concentration guarantee for the least square estimator with a skipping process.
\begin{theorem}[Concentration of Least-Square with Skipping Process]\label{theorem: abassi confidence interval}
Let 
$  
\hat{\theta}_t = (X_t^TX_t+\lambda I_d)^{-1} X_t^T Y_t,
$
where $X_t$ is the matrix whose rows are $\indicator{q_1=1}x_1^T,..,\indicator{q_t=1}x_t^T$, $Y_t = (\indicator{q_1} y_1,..,\indicator{q_t}y_t)^T$, $y_t=\inner{x_t,\theta^*} + \eta_t$, and $\brc*{q_t}_{t\geq 1}$ is a sequence of binary $F_{t-1}$ measurable events. Also assume that $\norm{\theta^*}\le D$. Then, for any $\delta>0$ with a probability of at least $1-\delta$ for all, $t\geq 0$, $\theta^*$ lies in the set
\begin{align*}
    &C_t\eqdef \brc*{\theta\in \R^d: \norm{\hat{\theta}_t - \theta}_{V_t} \leq \sigma\sqrt{d\log\br*{\frac{1+ tL^2/\lambda}{\delta}}} + \lambda^{1/2}D}.
\end{align*}
\end{theorem}
\begin{proof}
Let $\eta = (\indicator{q_1}\eta_1,.,,\indicator{q_t}\eta_t)^T$ and denote $X=X_t,Y=Y_t$. Observe that
\begin{align*}
    \hat{\theta}_t &= (X^TX+\lambda I_d)^{-1} X^T(X\theta^* + \eta)\\
    &= (X^TX+\lambda I_d)^{-1} X^T\eta +\theta^* - \lambda(X^TX+\lambda I_d)^{-1}\theta^*.
\end{align*}
Thus, for any $x\in \R^d$
\begin{align*}
    \inner{x, \hat{\theta}_t-\theta^*} = \inner{x, X^T\eta}_{V_t^{-1}} - \lambda\inner{x, \theta^*}_{V_{t}^{-1}},
\end{align*}
where $V_t = X^TX+\lambda I$. Since $V_t$ is positive definite and symmetric the inner product is well define. By Cauchy-Schwartz inequality, we get
\begin{align}
    | \inner{x, \hat{\theta}_t-\theta^*}| &\leq \norm{x}_{V_t^{-1}} \br*{\norm{X^T\eta}_{V_t^{-1}} +\lambda\norm{\theta^*}_{V_t^{-1}}} \nonumber \\
   & \leq \norm{x}_{V_t^{-1}} \br*{\norm{X^T\eta}_{V_t^{-1}} +\lambda^{1/2}\norm{\theta^*}}. \label{eq: yassin analysis 1}
\end{align}

We now apply Theorem~\ref{theorem: abassi confidence self normalized martingale} to bound $\norm{X^T\eta}_{V_t^{-1}}$. Let
\begin{align*}
   X^T\eta = \sum_{s=1}^t x_t \indicator{q_t=1}^2 \eta_t = \sum_{s=1}^t x_t \indicator{q_t=1} \eta_t = \sum_{s=1}^t x'_t \eta_t \eqdef S_t,
\end{align*}
where we defined $x'_t =  x_t \indicator{q_t=1}$. See that $x'_t$ is $F_{t-1}$ measurable since both $x_{t}$ and $ \indicator{q_t=1}$ are $F_{t-1}$ measurable. Thus, we can apply Theorem~\ref{theorem: abassi confidence self normalized martingale} and get the following bound with probability greater than $1-\delta$ for all $t\geq 0$
\begin{align*}
    \norm{X^T\eta}_{V_t^{-1}} \leq  \sqrt{2\sigma^2 \log\br*{\frac{\det{V_t}^{1/2} \det{V}^{-1/2}}{\delta}}}.
\end{align*}
Conditioning on the event this bound holds and setting $x = V_t(\hat \theta_t -\theta^*)$ in~\eqref{eq: yassin analysis 1} and using $\norm{\theta^*}\leq D$ we get
\begin{align*}
    \norm{\hat{\theta}_t - \theta^*}_{V_t}^2 \leq \norm{V_t(\hat{\theta}_t - \theta^*)}_{V_t^{-1}} \br*{ \sqrt{2\sigma^2 \log\br*{\frac{\det{V_t}^{1/2} \det{V}^{-1/2}}{\delta}}} +\lambda^{1/2}D}.
\end{align*}
Using $\norm{V_t(\hat{\theta}_t - \theta^*)}_{V_t^{-1}} = \norm{\hat{\theta}_t - \theta^*}_{V_t}$ and rearranging leads to 
\begin{align}
    \norm{\hat{\theta}_t - \theta^*}_{V_t} \leq  \sigma\sqrt{2\log\br*{\frac{\det{V_t}^{1/2} \det{V}^{-1/2}}{\delta}}} +\lambda^{1/2}D. \label{eq: yassin analysis 2}
\end{align}

To obtain the final form, we us bound the term $\log\br*{\det{V_t}^{1/2} \det{V}^{-1/2}}.$ The trace of $V_t$ is bounded by $\trace{V} + tL^2$ since $\norm{x_t'}\leq \norm{x_t}\leq L$. Hence, $\det{V_t} = \prod_{i=1}^d \lambda_i\leq \br*{\frac{\trace{V} + t L^2}{d}}^d$ by the GM-AM inequality and $\sum_{i} \lambda_i = \trace{V_t}$. Therefore,
\begin{align*}
    \log\br*{\det{V_t} \det{V}^{-1}}=\log \det{V_t} - d\log \lambda \leq d \log\br*{\frac{\lambda+t L^2}{d}} -  d\log \lambda = d \log\br*{\frac{1+t L^2/\lambda}{d}}.
\end{align*}
Plugging this back into~\eqref{eq: yassin analysis 2} concludes the proof.
\end{proof}

\begin{lemma}[Elliptical Potential Lemma, \citep{abbasi2011improved}, Lemma 11] \label{lemma: eliptical potential lemma abbasi}
Let $\brc*{x_t}_{t=1}^\infty$ be a sequence in $\R^d$ and $V_t= V+\sum_{i=1}^t x_i x_i^T$. Assume $\norm{x_t}\leq L$ for all $t$. Then,
\begin{align*}
    \sum_{i=1}^t \min\br*{\norm{x_i}_{V_{i-1}^{-1}}^2,1} \leq 2\log\br*{\frac{\det{V_t}}{\det{V}}}  \leq  2d\log\br*{\frac{\trace{V} + tL^2}{d}} -2\log{\det{V}}.
\end{align*}

Furthermore, if $\lambda_{min}(V)\geq \max(1,L^2)$ then
\begin{align*}
    \sum_{i=1}^t \norm{x_i}_{V_{i-1}^{-1}}^2 \leq 2\log{\frac{\det{V_t}}{\det{V}}}  \leq  2d\log{\frac{\trace{V} + tL^2}{d}}.
\end{align*}
\end{lemma}

The next result appears in~\citep{zanette2019tighter}, proposition 2, lines 48-51, and is a variation of a result utilized by~\citep{azar2017minimax}.
\begin{lemma}[Standard Deviation Difference \citep{zanette2019tighter}]\label{lemma: std difference}
Let $V_1,V_2: \mathcal{S}\rightarrow \mathbb{R}$ be fixed mappings. Let $P(s)$ be a probability measure over the state space. Then, $\sqrt{\VAR(V_1)} - \sqrt{\VAR(V_2)}\leq \sqrt{\VAR(V_1-V_2)}$.
\end{lemma}

\begin{lemma}[Law of Total Variance, e.g.,~\citep{azar2017minimax,zanette2019tighter}]\label{lemma: law of total variance for RL}
For any $\pi$ the following holds.
\begin{align*}
    \E\brs*{\sum_{h=1}^H\VAR_{P_h(\cdot |s_{h},a_h)}(V^\pi_{h+1} ) |\pi} = \E\brs*{\br*{\sum_{h=1}^H r(s_h,a_h) - V_1^\pi(s_1) }^2|\pi}.
\end{align*}

\end{lemma}

\newpage

\section{Useful Identities}

\begin{lemma}\label{lemma: solution of inner production uncertaintly}
Let $C= \brc{\theta: \norm{\bar{\theta}-\theta}_{A}\leq \alpha}$ for some $\bar\theta\in \mathbb{R}^d,\alpha\in \mathbb{R}$ and $A\in \mathbb{R}^{d\times d}$ a PD matrix. Then, for any $x\in \mathbb{R}^d$
\begin{align*}
    \max_{\theta\in C}\inner{x,\theta - \bar{\theta}} =  \max_{\theta\in C}\inner{x,\bar{\theta} - \theta } = \alpha \norm{x}_{A^{-1}}.
\end{align*}
\end{lemma}
\begin{proof}
The following relations hold.
\begin{align*}
    \max_{\theta\in C}\inner{x,\theta - \bar{\theta}}  &= \max_{\theta\in C}\inner{A^{-1/2}x,A^{1/2}(\theta - \bar{\theta})}\\
    &\leq \norm{x}_{A^{-1}} \max_{\theta\in C}\norm{\theta - \bar{\theta}}_{A} \tag{Cauchy-Schwartz Inequality}\\
    &\leq \alpha \norm{x}_{A^{-1}} \tag{Definition of the set $C$}.
\end{align*}
Setting $\theta - \bar{\theta} = \frac{\alpha}{\norm{x}^2}A^{-1/2}x$ the above inequalities hold with equality (see that $\norm{\theta - \bar{\theta}}_{A} =\alpha$ which implies that $\theta - \bar{\theta}\in C$). Thus, 
\begin{align*}
    \max_{\theta\in C}\inner{x,\theta - \bar{\theta}} = \alpha \norm{x}_{A^{-1}}.
\end{align*}
Lastly, since $\norm{x}_{A^{-1}}= \norm{-x}_{A^{-1}}$ we conclude that
\begin{align*}
    \max_{\theta\in C}\inner{x,\bar{\theta} - \theta } = \max_{\theta\in C}\inner{-x,\theta -\bar{\theta} }= \alpha \norm{-x}_{A^{-1}} = \alpha \norm{x}_{A^{-1}}.
\end{align*}
\end{proof}

\end{document}